\documentclass{article}

\pdfoutput=1

\PassOptionsToPackage{numbers}{natbib}

\usepackage[final]{neurips_2021}

\usepackage[utf8]{inputenc} 
\usepackage[T1]{fontenc}    
\usepackage{hyperref}       
\usepackage{url}            
\usepackage{booktabs}       
\usepackage{amsfonts}       
\usepackage{nicefrac}       
\usepackage{microtype}      

\usepackage{float}
\usepackage{wrapfig}
\usepackage{tcolorbox}
\usepackage{xcolor}
\definecolor{pearThree}{HTML}{E74C3C}
\definecolor{pearcomp}{HTML}{B97E29}
\definecolor{pearDark}{HTML}{2980B9}
\definecolor{pearDarker}{HTML}{1D2DEC}
\hypersetup{
	colorlinks,
	citecolor=pearDark,
	linkcolor=pearThree,
	breaklinks=true,
	urlcolor=black}

\definecolor{gold}{rgb}{0.99, 0.76, 0.0}
\definecolor{lust}{rgb}{0.9, 0.13, 0.13}

\usepackage{verbatim}
\usepackage{amsmath,amsthm,amsfonts,amssymb}
\usepackage{cleveref}

\usepackage{mathrsfs}
\usepackage{enumitem}
\usepackage{bm}
\usepackage{multirow}
\usepackage{tikz}
\usetikzlibrary{positioning,angles,quotes}
\usepackage{booktabs}
\usepackage{makecell}
\usepackage{graphicx}
\usepackage{caption}
\usepackage{subcaption}
\usepackage{thmtools}
\usepackage{thm-restate}
\usepackage{dsfont}
\usepackage{xspace}
\usepackage{empheq}

\usepackage{stfloats}

\usepackage{wrapfig}

\include{notations}
\usepackage[title]{appendix}
\allowdisplaybreaks[4]

\newtheorem{theorem}{Theorem}
\newtheorem{lemma}{Lemma}
\newtheorem{corollary}{Corollary}
\newtheorem{claim}{Claim}

\newtheorem{definition}{Definition}

\newtheorem{assumption}{Assumption}

\theoremstyle{definition}

\usepackage{standalone}

\usepackage{mathtools}
\usepackage{amssymb}
\usepackage{bbm}
\usepackage{bm}
\usepackage{enumitem}
\usepackage{dirtytalk}
\usepackage{pifont}
\usepackage{tikz}
\usetikzlibrary{arrows.meta}
\usepackage{float}

\usepackage{soul}
\usepackage{dsfont}
\usepackage{amssymb}

\usepackage{fontawesome}

\usepackage{sidecap}
\usepackage{upgreek}

\usepackage{xifthen}

\usepackage{algorithm}
\usepackage{algorithmic}

\newcommand{\oneline}[1]{%
  \newdimen{\namewidth}%
  \setlength{\namewidth}{\widthof{#1}}%
  \ifthenelse{\lengthtest{\namewidth < \linewidth}}%
  {#1}
  {\resizebox{\linewidth}{!}{#1}}
}

\newcommand{\SO}{$\mathcal{SO}$\xspace}
\newcommand{\OSP}{{\small$\textsc{GOSPRL}$}\xspace}
\newcommand{\OSPtitle}{{\normalsize$\textsc{GOSPRL}$}\xspace}

\newcommand{\gammaVI}{\mu_{\textsc{{\tiny VI}}}\xspace}

\newcommand{\OSPmath}{{\small\textsc{GOSPRL}}\xspace}

\newcommand{\OSPmathL}{{\small\textsc{GOSPRL-L}}\xspace}

\newcommand{\EVI}{{\small\textsc{EVI}}}
\newcommand{\UCRL}{{\small\textsc{UCRL}}\xspace}
\newcommand{\UCRLtwo}{{\small\textsc{UCRL2}}\xspace}
\newcommand{\UCRLtwoB}{{\small\textsc{UCRL2B}}\xspace}
\newcommand{\ZEROUCRLtwo}{{\small\textsc{0-UCRL}}\xspace}
\newcommand{\ZEROONEUCRLtwo}{{\small\textsc{0/1-UCRL}}\xspace}

\newcommand{\OSPsmall}{{\footnotesize\textsc{GOSPRL}}\xspace}
\newcommand{\ZEROONEUCRLtwosmall}{{\footnotesize\textsc{0/1-UCRL}}\xspace}
\newcommand{\MODESTsmall}{{\footnotesize\textsc{ModEst}}\xspace}

\newcommand{\subExp}{{\small\textsc{subExp}}}

\newcommand{\RMAX}{{\small\textsc{RMax}}\xspace}
\newcommand{\ZERORMAX}{{\small\textsc{ZeroRMax}}\xspace}
\newcommand{\RFRLExplore}{{\small\textsc{RF-RL-Explore}}\xspace}

\newcommand{\Ethree}{{\small\textsc{E3}}\xspace}

\newcommand{\VI}{{\small\textsc{VI}}\xspace}

\newcommand{\TREASURE}{{\small\textsc{Treasure}}\xspace}
\newcommand{\TREASUREtitle}{\textsc{Treasure}\xspace}

\newcommand{\TREASURETEN}{{\small\textsc{Treasure-10}}\xspace}

\newcommand{\OSPfootnote}{{\footnotesize$\textsc{GOSPRL}$}\xspace}
\newcommand{\TREASURETENfootnote}{{\footnotesize\textsc{Treasure-10}}\xspace}
\newcommand{\TREASUREONEfootnote}{{\footnotesize\textsc{Treasure-1}}\xspace}

\newcommand{\MaxEnt}{{\small\textsc{MaxEnt}}\xspace}
\newcommand{\WeightedMaxEnt}{{\small\textsc{WeightedMaxEnt}}\xspace}

\newcommand{\MODEST}{{\small\textsc{ModEst}}\xspace}
\newcommand{\MODESTtitle}{\textsc{ModEst}\xspace}

\newcommand{\RMODEST}{{\small\textsc{RModEst}}\xspace}

\newcommand\myeqa{\mathrel{\stackrel{\makebox[0pt]{\mbox{\normalfont\tiny (a)}}}{=}}}
\newcommand\myeqb{\mathrel{\stackrel{\makebox[0pt]{\mbox{\normalfont\tiny (b)}}}{=}}}
\newcommand\myineeqa{\mathrel{\stackrel{\makebox[0pt]{\mbox{\normalfont\tiny (a)}}}{\leq}}}
\newcommand\myineeqb{\mathrel{\stackrel{\makebox[0pt]{\mbox{\normalfont\tiny (b)}}}{\leq}}}
\newcommand\myineeqc{\mathrel{\stackrel{\makebox[0pt]{\mbox{\normalfont\tiny (c)}}}{\leq}}}
\newcommand\myineeqd{\mathrel{\stackrel{\makebox[0pt]{\mbox{\normalfont\tiny (d)}}}{\leq}}}

\newcommand\mygineeqc{\mathrel{\stackrel{\makebox[0pt]{\mbox{\normalfont\tiny (c)}}}{\geq}}}

\DeclareMathOperator*{\argmax}{arg\,max}
\DeclareMathOperator*{\argmin}{arg\,min}

\newcommand{\SD}{\Pi^\text{SD}}

\newcommand{\Gk}{\mathcal{G}_k}

\newcommand{\Gj}{\mathcal{G}_j}

\newcommand{\Gm}{\mathcal{G}_m}

\newcommand{\cG}{\mathcal{G}}
\newcommand{\cS}{\mathcal{S}}
\newcommand{\cA}{\mathcal{A}}
\newcommand{\SA}{\mathcal{S} \times \mathcal{A}}

\newcommand{\cSc}{\mathcal{S}^{\textsc{C}}}
\newcommand{\cSt}{\mathcal{S}^{\textsc{T}}}

\makeatletter
\newcommand\footnoteref[1]{\protected@xdef\@thefnmark{\ref{#1}}\@footnotemark}
\makeatother

\usepackage[colorinlistoftodos]{todonotes}

\newcommand{\wt}[1]{\widetilde{#1}}
\newcommand{\wb}[1]{\overline{#1}}

\newcommand{\wh}[1]{\widehat{#1}}

\DeclarePairedDelimiter\abs{\lvert}{\rvert}%
\DeclarePairedDelimiter\norm{\lVert}{\rVert}%

\let\originalleft\left
\let\originalright\right
\renewcommand{\left}{\mathopen{}\mathclose\bgroup\originalleft}
\renewcommand{\right}{\aftergroup\egroup\originalright}

\newcount\Comments  
\definecolor{darkgreen}{rgb}{0,0.5,0}
\definecolor{darkred}{rgb}{0.7,0,0}
\definecolor{teal}{rgb}{0.3,0.8,0.8}

\usetikzlibrary{matrix, positioning, arrows.meta}
\newlength{\myheight}
\tikzset{labels/.style={font=\sffamily\scriptsize},
    circuit/.style={draw,minimum width=2cm,minimum height=\myheight,very thick,inner sep=1mm,outer sep=0pt,cap=round,font=\sffamily\bfseries},
    triangle 45/.tip={Triangle[angle=45:8pt]}
}

\usepackage{minitoc}
\renewcommand\epsilon{\varepsilon}

\title{A Provably Efficient Sample Collection Strategy \\ for Reinforcement Learning}

\author{%
  Jean Tarbouriech \hspace{-4em}\\
  Facebook AI Research Paris \& Inria Lille\hspace{-5em}\\
  \hspace{-4em}\texttt{jean.tarbouriech@gmail.com\hspace{-5em}}
  \And
  \hspace{-2em}Matteo Pirotta \\
  \hspace{-2em}Facebook AI Research Paris\\
  \hspace{-2em}\texttt{pirotta@fb.com}
  \AND
  \hspace{3em}Michal Valko \\
  \hspace{3em}DeepMind Paris\\
  \hspace{3em}\texttt{valkom@deepmind.com} \\
  \And
  \hspace{2em}Alessandro Lazaric \\
  \hspace{2em}Facebook AI Research Paris\\
  \hspace{2em}\texttt{lazaric@fb.com}
}

\usepackage{minitoc}
\begin{document}

\maketitle

\doparttoc 
\faketableofcontents 

\begin{abstract}
One of the challenges in \textit{online} reinforcement learning (RL) is that the agent needs to trade off the exploration of the environment and the exploitation of the samples to optimize its behavior. Whether we optimize for regret, sample complexity, state-space coverage or model estimation, we need to strike a different exploration-exploitation trade-off. In this paper, we propose to tackle the exploration-exploitation problem following a decoupled approach composed of: \textbf{1)}~An ``objective-specific'' algorithm that (adaptively) prescribes \textit{how many} samples to collect \textit{at which} states, as if it has access to a generative model (i.e., a simulator of the environment); \textbf{2)}~An ``objective-agnostic'' sample collection exploration strategy responsible for generating the prescribed samples as fast as possible. Building on recent methods for exploration in the stochastic shortest path problem, we first provide an algorithm that, given as input the number of samples $b(s,a)$ needed in each state-action pair, requires $\wt{O}\left( B D + D^{3/2} S^2 A \right)$ time steps to collect the $B=\sum_{s,a} b(s,a)$ desired samples, in any unknown communicating MDP with $S$~states, $A$~actions and diameter~$D$. Then we show how this general-purpose exploration algorithm can be paired with ``objective-specific'' strategies that prescribe the sample requirements to tackle a variety of settings --- e.g., model estimation, sparse reward discovery, goal-free cost-free exploration in communicating MDPs --- for which we obtain improved or novel sample complexity guarantees.
\end{abstract}

\captionsetup[figure]{labelfont=small,font=small}


\section{Introduction}

One of the challenges in \textit{online} reinforcement learning (RL) is that the agent needs
to trade off the exploration of the environment and the exploitation of the samples to optimize its behavior. 
Whenever the agent needs to gather information about a specific region of the Markov decision process (MDP), it must plan for a policy to reach the desired states, despite not having exact knowledge of the environment dynamics. This makes solving the exploration-exploitation problem in RL highly non-trivial and it requires designing a specific strategy depending on the learning objective, such as PAC-MDP learning \mbox{\citep[e.g.,][]{brafman2002r, strehl2009reinforcement, wang2019q}}, regret minimization~\citep[e.g.,][]{jaksch2010near, azar2017minimax, jin2018is-q-learning, zhang2020almost} or pure exploration~\citep[e.g.,][]{jin2020reward, kaufmann2021adaptive, menard2020fast, zhang2020taskagnostic, zhang2020nearly}.

A simpler scenario  considered in the literature is to assume access to a \textit{generative model} or \textit{sampling oracle} (\SO) \citep{kearns2002sparse}. Given any  
state-action pair $(s,a)$, the \SO returns a next state $s'$ drawn from the transition probability $p(\cdot \vert s,a)$ and a reward $r(s,a)$. In this case,
it is possible to focus exclusively on where and how many samples to collect, while disregarding the problem of finding a suitable policy to obtain them.
For instance, an \SO can be used to obtain samples from the environment, 
which are combined with dynamic programming techniques to compute an $\epsilon$-optimal policy. \SO-based algorithms can be as simple as prescribing the same amount of samples from each state-action pair ~\citep[e.g.][]{kearns2000approximate, kearns2002sparse, azar2013minimax, chen2016stochastic, sidford2018near, agarwal2019optimality, li2020breaking} or they may adaptively change the sample requirements on different state-action pairs~\citep[e.g.][]{chen2018scalable, wang2017primal, zanette2019almost}. An \SO is also used in Monte-Carlo planning~\cite{szorenyi2014optimistic,grill2016blazing,bartlett2019scale-free} which focuses on computing the optimal action at the current state by optimizing over rollout trajectories sampled from the \SO. Finally, in multi-armed bandit~\citep{lattimore2019bandit}, 
there are cases where each arm 
corresponds to a state (or state-action), and ``pulling'' an arm translates into a call to an \SO (see e.g., the pure exploration setting of~\citep{tarbouriech2019active}). Unfortunately, while an \SO may be available in domains such as simulated robotics and computer games, this is not the case in the more general \textit{online RL} setting.


In this paper we tackle the exploration-exploitation problem in online RL by drawing inspiration from the \SO assumption. Specifically, we define an approach that is decoupled in two parts: \textbf{1)}~an ``objective-specific'' algorithm that assumes access to an \SO that (adaptively) prescribes the samples needed to achieve the learning objective of interest, and \textbf{2)}~an ``objective-agnostic'' algorithm that takes on the exploration challenge of collecting the samples requested by the \SO-based algorithm as quickly as possible.
\footnote{Alternatively, we can view it as a general approach to take any \SO-based algorithm and convert it into an online RL algorithm.} Our main contributions can be summarized as follows:
\begin{itemize}[leftmargin=.1in,topsep=-2.5pt,itemsep=1pt,partopsep=0pt, parsep=0pt]
	\item We define the sample complexity of the objective-agnostic algorithm as the number of (online) steps needed to satisfy the prescribed sampling requirements. Leveraging recent techniques on exploration in the stochastic shortest path (SSP) problem~\cite{cohen2020near,tarbouriech2019no}, we propose \OSP (Goal-based Optimistic Sampling Procedure for RL), a conceptually simple and flexible exploration algorithm that learns how to ``generate'' the samples requested by any \SO-based algorithm and we derive bounds on its sample complexity. 
	\newcommand{\wtO}{{\scalebox{0.92}{$\wt O$}}}
	\newcommand{\Dexp}{{\scalebox{0.92}{$D^{3/2}$}}}

	\item 
	Leveraging the generality of our approach, we combine \OSP with problem-specific \SO-based algorithms and readily obtain online RL algorithms in difficult exploration problems. 
	While in general our decoupled approach may be suboptimal compared to exploration strategies designed to solve one specific problem, we obtain sample complexity guarantees that are on par or better than state-of-the-art algorithms in a range of problems. \textbf{1)}~\OSP solves the problem of sparse reward discovery in $\wtO\left( \Dexp S^2 A \right)$ time steps, which improves the dependency on the diameter $D$ w.r.t.\,a reward-free variant of \UCRLtwoB \citep{jaksch2010near, improved_analysis_UCRL2B}, as well as on $S$ and $A$ w.r.t.\,a \MaxEnt-type approach \citep{hazan2019provably, cheung2019exploration}. \textbf{2)}~\OSP improves over the method of \citep{tarbouriech2020active} for model estimation, by removing their ergodicity assumption as well as achieving better sample complexity. \textbf{3)}~\OSP provably tackles the problem of goal-free cost-free exploration, for which no specific strategy is available.
	
	\item We report numerical simulations supporting our theoretical findings and showing that pairing \OSP with \SO-based algorithms outperforms both heuristic and theoretically grounded baselines in various problems.
\end{itemize}

\vspace{0.06in}
\textbf{Related work.}
While to the best of our knowledge no other work directly addresses the problem of simulating an \SO, a number of approaches are related to it. The problem solved by \OSP can be seen as a \textit{reward-free} exploration problem, since it is not driven by any external reward but by the objective of covering the state space to quickly meet the sampling requirements. Standard exploration-exploitation algorithms, such as \UCRLtwo \citep{jaksch2010near} in the undiscounted setting or \RMAX \citep{brafman2002r} in the discounted one, implicitly encourage exploration to specific areas of the state-action space that are not estimated accurately enough. The objective of covering the state space is also studied in~\citep{hazan2019provably, cheung2019exploration} with a Frank-Wolfe approach
that optimizes a smooth aggregate function of the state visitations. 


Recent works on reward-free exploration (RFE) in the finite-horizon setting~\citep[e.g.,][]{jin2020reward,kaufmann2021adaptive,menard2020fast,zhang2020nearly} provide sufficient exploration so that an $\epsilon$-optimal policy for \textit{any} reward function can be computed. Our proposed solution shares high-level algorithmic principles with RFE approaches which incentivize the agent to visit insufficiently visited states via intrinsic reward. Nonetheless, our contribution significantly differs from existing RFE literature in two dimensions: \textbf{1)} While we study the performance of \OSP in one goal-conditioned RFE problem (Sect.~\ref{subsection_costfree_goalfree}), our framework is much broader and it allows us to tackle a wider and diverse set of problems (Sect.\,\ref{ex} and App.\,\ref{app_other_applications}); \textbf{2)} Our setting is \textit{horizon-agnostic} and \textit{reset-free}, which prevents from directly using any method or technical analysis in RFE designed for problems with an imposed planning horizon (e.g., finite-horizon or discounted).


Finally, \OSP draws inspiration from the SSP formalism and solutions of~\citep{tarbouriech2019no,cohen2020near}, but our approach critically differs from these works in three main ways: \textbf{1)} we are interested in sample complexity guarantees rather than a regret analysis; \textbf{2)} we consider requirements (i.e., goals to sample) that vary throughout the learning process, instead of an SSP problem with fixed goal state and cost function; \textbf{3)} we show how \OSP can serve as a sample collection component to tackle various learning problems other than regret minimization.


\section{Problem Definition}
\label{sect_preliminaries}

We consider a finite and \textit{reset-free} 
MDP \citep{puterman2014markov} $M := \langle\mathcal{S}, \mathcal{A}, p, r, s_{0} \rangle$, with $S := \abs{\mathcal{S}}$ states, $A := \abs{\mathcal{A}}$ actions and an arbitrary starting state \mbox{$s_0 \in \mathcal{S}$}. Calling an \SO in any state-action pair $(s,a)$ 
leads to two outcomes: a next state sampled from the transition probability distribution \mbox{$p(\cdot \vert s,a) \in \Delta(\cS)$}, and a reward $r(s,a) \in \mathbb{R}$. 
A stationary deterministic policy is a mapping \mbox{$\pi:\cS \rightarrow \cA$} from states to actions and we denote by $\Pi^{\text{SD}}$ the set of all such policies. For any policy~$\pi$ and pair of states $(s, s')$, let $\tau_{\pi}(s \rightarrow s')$ be the (possibly infinite) hitting time from~$s$ to~$s'$ when executing~$\pi$, i.e.,\,$\tau_{\pi}(s \rightarrow s') :=$ \mbox{$\inf \{ t \geq 0: s_{t+1} = s' \vert\, s_1 = s, \pi \}$}, where $s_t$ is the state visited at time step $t$. We introduce
\begin{align*}
        D_{s s'} := \min_{\pi \in \SD} \mathbb{E}\left[ \tau_{\pi}(s \rightarrow s')\right], \quad \quad
        D_{s'} := \max_{s \in \cS \setminus \{s'\}} D_{s s'}, \quad \quad D := \max_{s' \in \cS} D_{s'},
\end{align*}

where $D_{s s'}$ is the shortest-path distance between $s$ and $s'$, $D_{s'}$ is the SSP-diameter of $s'$ \citep{tarbouriech2019no} and $D$ is the MDP diameter \citep{jaksch2010near}.

We now formalize the problem of simulating an \SO (i.e., to generate the samples prescribed by an \SO-based algorithm). At each time step $t \geq 1$ the agent receives a function $b_t: \SA \rightarrow \mathbb{N}$, where  $b_t(s,a)$ defines the total number of samples that need to be collected at $(s,a)$ by time step $t$. We consider that $(b_t)_{t \geq 1}$ is an arbitrary sequence with each $b_t$ measurable w.r.t.\,the filtration up to time $t$ (i.e., it may depend on the samples observed so far).\footnote{Allowing adaptive sampling requirements enables to pair \OSP with \SO-based algorithms that adjust their requirements \textit{online} as samples are being generated (see e.g., Sect.\,\ref{subsection_model_estimation}).} 
We focus on the objective of designing an \textit{online algorithm} that minimizes the time required to collect the prescribed samples. Since the environment is initially unknown, we need to trade off between exploring states and actions to improve estimates of the dynamics and exploiting current estimates to collect the required samples as quickly as possible. We formally define the performance metric as follows.

\begin{definition}\label{def:sample.comeplxity}
	For any state-action pair, we denote by $N_t(s,a) := \sum_{i=1}^t \mathds{1}_{\{(s_i,a_i) = (s,a)\}}$ the number of visits to state $s$ and action $a$ up to (and including) time step $t$. Given a sampling requirement sequence $b := (b_t)_{t \geq 1}$ with $b_t : \SA \rightarrow \mathbb{N}$ and a confidence level $\delta \in (0,1)$, we define the \textit{sample complexity} of a learning algorithm $\mathfrak{A}$ as
   \begin{align*}
                  \mathcal{C}\big(\mathfrak{A}, b, \delta \big) := \min \big\{ t > 0: \mathbb{P}\big( \forall (s,a) \in \SA, ~N_t(s,a) \geq b_t(s,a) \big) \geq 1 - \delta \big\}.
     \end{align*}
\end{definition}%

With no additional condition, it is trivial to define problems such that $\mathcal{C}(\mathfrak{A}, b, \delta) = + \infty$ for any algorithm. To avoid this case, we introduce the following assumptions.
\begin{assumption}\label{asm_communicating}
    The MDP $M$ is communicating with a finite and unknown diameter $D < + \infty$.
\end{assumption}
\begin{assumption}\label{asm_bounded_requirements}
    There exist an unknown and bounded function $\overline{b}: \SA \rightarrow \mathbb{N}$ such that the sequence~$(b_t)_{t \geq 1}$ verifies: $\forall t \geq 1, ~ \forall (s,a) \in \SA, ~ b_t(s,a) \leq \overline{b}(s,a)$.
\end{assumption}
%

Asm.\,\ref{asm_communicating} guarantees that whatever state needs to be sampled, there exists at least one policy that can reach it in finite time almost-surely (notice that it is considerably weaker than the ergodicity assumption (App.\,\ref{app_ergodicity}) often used in online RL, see e.g.,~\citep{wei2019model,ortner2020regret,garcelon2020conservative}). Asm.\,\ref{asm_bounded_requirements} ensures that the sequence of sampling requirements does not diverge and can thus be fulfilled in finite time. These assumptions guarantee that the problem in Def.\,\ref{def:sample.comeplxity} is well-posed and the sample complexity is bounded.



A variety of problems can be cast under our decoupled approach, in the sense that they can be tackled by solving the problem of Def.\,\ref{def:sample.comeplxity} under a specific instantiation of the sampling requirement sequence $(b_t)_{t \geq 1}$. For instance, consider the problem of covering the state-action space (e.g., to discover a hidden sparse reward), then the requirement is immediately defined as $b_t(s,a)=1$. In Sect.\,\ref{ex} and App.\,\ref{app_other_applications}, we review problems where defining $b_t$ can be as simple as computing the sufficient number of samples needed to reach a certain level of accuracy in estimating a quantity of interest (e.g., model estimation) or can be directly extracted from existing literature (e.g., $\epsilon$-optimal policy learning).


We now provide a simple worst-case lower bound on the sample complexity (details in App.\,\ref{app:lower_bound}).

\begin{lemma}
\label{lemma_LB_informal}
For any $S \geq 1$, there exists an MDP with $S$ states 
satisfying Asm.~\ref{asm_communicating} such that for any sampling requirement $b: \mathcal{S} \rightarrow \mathbb{N}$ satisfying Asm.~\ref{asm_bounded_requirements}, 
\begin{align*}
\min_{\mathfrak{A}} \mathcal{C}\big(\mathfrak{A}, b, \tfrac{1}{2}\big) = \Omega\Big(\sum_{s \in \mathcal{S}} D_s b(s)\Big).
\end{align*}
\end{lemma}
Lem.\,\ref{lemma_LB_informal} shows that the (possibly non-stationary) policy minimizing the time to collect all samples requires $\Omega\big(\sum_{s} D_s b(s)\big)$ time steps in a worst-case MDP. We also notice that when the total sampling requirement $B$ is concentrated on the state $\overline{s}$ for which $D_{\overline{s}} = D$ (i.e., $b(s')=0$, $\forall s'\neq \overline{s}$), the previous bound reduces to $\Omega(BD)$.


\section{Online Learning for \SO Simulation}
\label{so}

We now introduce our algorithm for the problem in Def.\,\ref{def:sample.comeplxity}, bound its sample complexity and discuss several extensions.

\subsection{The \texorpdfstring{\OSP}{OSP} Algorithm}
\label{section_osp_algorithm}

\begin{algorithm}[tb]
\begin{small}
\begin{algorithmic}
    \STATE 
    {\bfseries Input:} sampling requirement sequence $(b_t)_{t \geq 1}$ with $b_{t} : \SA \rightarrow \mathbb{N}$ revealed at time $t$ (or anytime before).
	\STATE {\bfseries Initialize:} Set $\mathcal{G}_1 := \{ s \in \mathcal{S}: \exists a \in \cA, b_1(s,a) > 0 \}$, time step $t:=1$, counters $N_1(s,a) := 0$, attempt index $k:=1$ and attempt counters $U_1(s,a) := 0$, $\nu_1(s,a) := 0$. \\ 
	\WHILE{$\mathcal{G}_{k}~\textup{is not empty}$} 
	\STATE Define the SSP problem $M_k$ with goal states $\Gk$, and compute its optimistic shortest-path policy $\widetilde{\pi}_k$.
	 \STATE Set \texttt{flag} = \texttt{True} and counter $\nu_k(s,a) := 0$. 
		\WHILE{\textup{\texttt{flag}}}
			\STATE Execute action $a_t := \widetilde{\pi}_k(s_t)$ and observe next state \mbox{$s_{t+1} \sim p(\cdot \vert s_t, a_t)$}. 
		    \STATE	Increment counters $\nu_{k}(s_t,a_t)$ and $N_{t}(s_t,a_t)$. 
			\IF{$s_{t+1} \in \Gk$ \textup{or} $\nu_k(s_t,a_t) > \{U_k(s_t,a_t) \lor 1\}$}
			\STATE Set \texttt{flag} = \texttt{False}.
			\ENDIF
			\STATE Set $t \mathrel{+}= 1$.
		\ENDWHILE
		\IF{ $s_t \in \Gk$}
		\STATE Execute an action $a$ s.t.\,$N_t(s_t,a) < b_t(s_t,a)$, observe next state $s_{t+1} \sim p(\cdot \vert s_t, a)$ and set $t \mathrel{+}= 1$.
		\ENDIF 
	\STATE Set $U_{k+1}(s,a) := U_k(s,a) + \nu_k(s,a)$, $k \mathrel{+}= 1$. 
	\STATE Update the set of goal states $\mathcal{G}_{k} := \big\{ s \in \mathcal{S}: \exists a \in \cA, N_{t-1}(s,a) < b_{t-1}(s,a) \big\}$.
	\ENDWHILE
\caption{\OSP Algorithm}\label{algo:SO}
\vspace{-0.03in}
\end{algorithmic}
\end{small}
\end{algorithm}


In Alg.\,\ref{algo:SO} we outline \OSP (\textit{Goal-based Optimistic Sampling Procedure for Reinforcement Learning}). At each time step $t$, \OSP receives a sampling requirement $b_t : \SA \rightarrow \mathbb{N}$. The algorithm relies on the principle of optimism in the face of uncertainty and proceeds through \textit{attempts} to collect relevant samples. We index the attempts by $k = 1, 2, \ldots$ and denote by $t_k$ the time step at the start of attempt $k$ and by $U_k := N_{t_k - 1}$ the number of samples available at the start of attempt $k$. At each attempt, \OSP goes through the following steps: \textbf{1)}~Cast the under-sampled states as goal states and define an associated unit-cost multi-goal SSP instance (with unknown transitions); \textbf{2)}~Compute an optimistic shortest-path policy; \textbf{3)}~Execute the policy until either a goal state is reached or a stopping condition is satisfied; \textbf{4)}~If a sought-after goal state denoted by $g$ has been reached, execute an under-sampled action (i.e., an action $a$ such that $N_t(g,a) < b_t(g,a)$). The algorithm ends when the sampling requirements are met, i.e., at the first time $t \geq 1$ where $N_t(s,a) \geq b_t(s,a)$ for all $(s,a)$.

\textbf{Step 1.}
At any attempt $k$ we begin by defining the set of all under-sampled states
\begin{align*}
\Gk := \big\{ s \in \mathcal{S}: \exists a \in \cA, N_{t_k-1}(s,a) < b_{t_k-1}(s,a) \big\}.
\end{align*}
We then cast the sample collection problem as a goal-reaching objective, by constructing a multi-goal SSP problem \citep{bertsekas1995dynamic} denoted by \mbox{$M_k:= \langle \mathcal{S}_k, \mathcal{A}, p_{k}, c_{k}, \Gk \rangle$},
with:\footnote{If the current state $s_{t_k}$ is under-sampled (i.e., $s_{t_k} \in \Gk$), we duplicate the state and consider it to be both a goal state in $\Gk$ and a non-goal state from which the attempt $k$ starts (and whose outgoing dynamics are the same as those of $s_{t_k}$), which ensures that the state at the start of each attempt cannot be a goal state.}
\begin{itemize}[leftmargin=.1in,topsep=-2.5pt,itemsep=1pt,partopsep=0pt, parsep=0pt]
	\item $\Gk$ denotes the set of goal states, $\mathcal{S}_{k}: = \mathcal{S} \setminus \Gk$ the set of non-goal states and $\cA$ the set of actions.
	\item The transition model $p_{k}$ is the same as the original~$p$ except for the transitions exiting the goal states which are redirected as a self-loop, i.e., \mbox{$p_{k}(s'|s,a) := p(s'|s,a)$} and \mbox{$p_{k}(g|g,a) := 1$} for any \mbox{$(s,s',a,g) \in \mathcal{S}_{k} \times \cS \times \cA \times \Gk$}.
	\item The cost function~$c_{k}$ is defined as follows: for any $a \in \cA$, any goal state $g \in \Gk$ is zero-cost ($c_{k}(g,a) := 0$), while the non-goal costs are unitary ($c_{k}(s,a) := 1$ for $s \in \cS_{k}$).
\end{itemize}
From~\citep{bertsekas1991analysis}, Asm.\,\ref{asm_communicating} and the positive non-goal costs $c_k$ entail that solving $M_k$ is a well-posed SSP problem and that there exists an optimal policy that is \textit{proper} (i.e., that eventually reaches one of the goal states with probability $1$ when starting from any $s \in \cS_k$). Crucially, the objective of collecting a sample from the under-sampled states $\Gk$ coincides with the SSP objective of minimizing the expected cumulative cost to reach a goal state in $M_k$.

\textbf{Step 2.} Since $p_k$ is unknown, we cannot directly compute the shortest-path policy for $M_k$. Instead, leveraging the samples collected so far, we apply an extended value iteration scheme for SSP which implicitly skews the empirical transitions $\wh p_k$ towards reaching the goal states. This procedure can be done efficiently as shown in~\citep{tarbouriech2019no} (see App.\,\ref{app_value_iteration_SSP}), and it outputs an \textit{optimistic} shortest-path policy~$\widetilde{\pi}_k$.

\textbf{Step 3.} $\widetilde{\pi}_k$ is then executed with the aim of quickly reaching an under-sampled state. Along its trajectory, the counter $N_t$ is updated for each visited state-action. Because of the error in estimating the model, $\widetilde{\pi}_k$ may never reach one of the goal states (i.e., it may not be proper in $p_k$). Thus $\wt{\pi}_k$ is executed until either one of the goals in $\Gk$ is reached, or the number of visits is doubled in a state-action pair in $\mathcal{S}_k \times \cA$, a standard termination condition first introduced in~\citep{jaksch2010near}. If a sought-after goal state is reached, the agent executes an under-sampled action according to the current sampling requirements at that state. At the end of each attempt, the statistics 
(e.g., model estimate) are updated. 

The algorithmic design of \OSP is conceptually simple and can flexibly incorporate various~modifications driven by slightly different objectives or prior knowledge, without altering Thm.\,\ref{theorem_upper_bound_general} (cf.\,App.\,\ref{sub_sect_subroutines}).

\subsection{Sample Complexity Guarantee of \texorpdfstring{\OSP}{OSP}}
\label{subsection_sample_complexity_sampling_procedure}

Thm.\,\ref{theorem_upper_bound_general} establishes the sample complexity guarantee of \OSP (Alg.\,\ref{algo:SO}).

\begin{theorem}\label{theorem_upper_bound_general}
    Under Asm.\,\ref{asm_communicating} and~\ref{asm_bounded_requirements}, for any sampling requirement sequence $b = (b_t)_{t \geq 1}$ and any confidence level $\delta \in (0,1)$, the sample complexity of \OSP is \mbox{bounded as}
    \begin{align}
    &\mkern-10mu \mathcal{C}\big(\OSPmath, b, \delta\big) = \wt{O}\Big( \overline{B} D + D^{3/2} S^2 A \Big), \label{sample_complexity_1} \\
      &\mkern-10mu \mathcal{C}\big(\OSPmath, b, \delta\big) = \wt{O}\Big( \sum_{s \in \mathcal{S}} \big( D_s \overline{b}(s) + D_s^{3/2} S^2 A \big) \Big),\label{sample_complexity_2}
    \end{align}%
where the $\wt{O}$ notation hides logarithmic dependencies on $S$, $A,$ $D$, $1/\delta$ and $\overline{b}(s) := \sum_{a \in \cA} \overline{b}(s,a)$ and $\overline{B} := \sum_{s \in \cS} \overline{b}(s)$. Recall that \mbox{$D_s \leq D$} is the \textit{SSP-diameter} of state $s$ and captures the difficulty of collecting a sample at state $s$ starting at any other state in the MDP.
\end{theorem}

We notice that in practice \OSP stops at the first \textit{random} step $\tau$ at which the sampling requirement $b_\tau(s,a)$ is achieved for all $(s,a)$. Thm.\,\ref{theorem_upper_bound_general} provides a worst-case upper bound on the stopping time of \OSP using the possibly loose bound $b_\tau(s,a) \leq \overline{b}(s,a)$. On the other hand, in the special case of $b: \cS \rightarrow \mathbb{N}$ when the requirements are both time-independent (i.e., given as initial input to the algorithm) and action-independent, the actual sampling requirement $b(s)$ (resp.\,$B := \sum_{s \in \mathcal{S}} b(s)$) replaces $\overline{b}(s)$ (resp.\,$\overline{B}$) in the bound. In the following, we consider this case for the ease of exposition.

\textbf{Proof idea.} The key step (see App.\,\ref{app_proof_upper_bound} for the full derivation) is to link the sample complexity of \OSP to the regret accumulated over the sequence of multi-goal SSP problems $M_k$ generated across multiple attempts. Indeed we can define the regret at attempt $k$ as the gap between the performance of the SSP-optimal policy $\pi^\star_k$ solving $M_k$ (i.e., the minimum expected number of steps to reach any of the states in $\Gk$ starting from $s_{t_k}$) and the actual number of steps executed by \OSP before terminating the attempt. While the SSP regret minimization analysis of~\citep{cohen2020near} assumes that the goal is fixed, we show that it is possible to bound the regret accumulated across different attempts for any arbitrary sequence of goals. The proof is concluded by bounding the cumulative performance of the SSP-optimal policies and it leads to the bound $\wt{O}\big( BD + D^{3/2} S^2 A \big)$ where $B := \sum_{s \in \mathcal{S}} b(s)$. On the other hand, the refined bound in Eq.\,\ref{sample_complexity_2} requires a more careful analysis, where we no longer directly translate regret bounds into sample complexity and we rather focus on relating the performance to state-dependent quantities $D_s$ and $b(s)$. Finally, we show that the extension to the general case of time-dependent action-dependent sampling requirements is straightforward and obtain~Thm.\,\ref{theorem_upper_bound_general}.

\newcommand{\wtO}{{\scalebox{0.9}{$\wt O$}}}

\textbf{Interpretation of Thm.\,\ref{theorem_upper_bound_general}.} We can decompose Eq.\,\ref{sample_complexity_1} as a linear term in $B$ and a constant term. In the regime of large sample requirements (i.e., large $B$), the sample complexity thus reduces to $\wtO(BD)$, which adds at most an extra ``cost'' factor of $D$ w.r.t.\,an \SO. As this may be loose in many cases, the more refined analysis of Eq.\,\ref{sample_complexity_2} stipulates a cost of~$D_s$ 
to collect a sample at state~$s$, which better captures the connectivity of the MDP. 
In fact the lower bound in Lem.\,\ref{lemma_LB_informal} shows that this cost of $D_s$ is \textit{unavoidable in the worst case}, and that \OSP is only constant and logarithmic terms off w.r.t.\,to the best sample complexity that can be achieved in the worst case. While an extra attempt of refinement would be to avoid being worst-case w.r.t.\,the starting state in the definition of $D_s$,\footnote{For instance, consider a simple deterministic chain with a requirement of one sample per state. If the~agent starts on the leftmost state, then a policy that keeps moving right has sample complexity $S$ without extra factor~$D$.} this seems particularly challenging as the randomness of the environment makes it hard to control and analyze the sequence of states traversed by the agent. Also note that existing bounds in SSP \citep{tarbouriech2019no, cohen2020near} are only worst-case and it remains an open question to derive finer (e.g., problem-dependent) bounds in SSP and how they could be leveraged in our case.

\textbf{Optimal solution.} 
\OSP 
targets a \textit{greedy-optimal} strategy, which seeks to sequentially minimize each 
time to reach an under-sampled state. 
Alternatively, one may wonder if 
it is possible to design a learning algorithm that approaches the performance of the \textit{exact-optimal} solution, i.e.,~a~(non-stationary) policy explicitly minimizing the number of steps required to fulfill the sampling requirements.\footnote{Notice that as illustrated in the lower bound of Lem.\,\ref{lemma_LB_informal}, the exact-optimal and greedy-optimal have the same performance in the worst case.} Such strategy can be characterized as the optimal policy of an SSP problem for an MDP with state space augmented by the current sampling requirements and goal state corresponding to the case 
when all desired samples are collected. Even under known dynamics, the computational complexity of computing the optimal policy in this MDP (e.g., via value iteration) is exponential (scaling in~$B^S$). When the dynamics is unknown, it appears highly challenging to obtain any learning algorithm whose performance is comparable to the exact-optimal strategy for any finite sample requirement~$B$.%

\textbf{Beyond Communicating MDPs.} In App.\,\ref{app_beyond_communicating} we design an extension of \OSP to poorly or weakly communicating environments. In this setting, it is expected to assess online the \say{feasibility} of certain sampling requirements and discard them whenever associated to states that are \textit{too difficult} to reach~or unreachable. Given as input a \say{reachability} threshold $L$, we derive sample complexity guarantees for our variant of \OSP where the (possibly large or infinite) diameter $D$ is fittingly replaced by~$L$.


\section{Applications of \OSPtitle}
\label{ex}

\makeatletter
\newcommand*\mysize{%
   \@setfontsize\mysize{8.2}{9.0}%
}
\makeatother

An appealing feature of \OSP is that it can be integrated with techniques that compute the (fixed or adaptive) sampling requirements to readily obtain an online RL algorithm with theoretical guarantees. In this section we focus on three specific problems where in our decoupled approach the \SO-based algorithm is either trivial or can be directly extracted from existing literature, and its combination with the sample collection strategy of \OSP yields improved or novel guarantees. Other applications (e.g., PAC-policy learning, diameter estimation, bridging bandits and MDPs) are illustrated in App.\,\ref{app_other_applications}.


\subsection{Sparse Reward Discovery (\TREASUREtitle)}
\label{subsection_treasure}


A number of recent methods focus on the state-space coverage problem, where each state in the MDP needs to be reached as quickly as possible. This problem is often motivated by environments where a one-hot reward signal, called the \textit{treasure}, is hidden and can only be discovered by reaching a specific state and taking a specific action. Not only the environment but also the treasure state-action pair is unknown, and the agent does not receive any side information to guide its search (e.g., a measure of closeness to the treasure). Thus the agent must perform exhaustive exploration to find the treasure. 
\begin{definition}
	Given a confidence $\delta \in (0,1)$, the \textit{\TREASURE sample complexity} of a learning algorithm~$\mathfrak{A}$ is defined as $\mathcal{C}_{\TREASURE}(\mathfrak{A} , \delta)  := \min \big\{t > 0: \mathbb{P}\big( \forall (s,a) \in \SA , ~N_t(s,a) \geq 1 \big) \geq 1 - \delta \big\}$.
\end{definition}%
%
%
In this case, a \SO-based algorithm would immediately solve the problem by collecting one sample from each state-action pair. As a result, we can directly apply \OSP for \TREASURE by simply setting $b(s,a) = 1$ for each $(s,a)$ and from Thm.\,\ref{theorem_upper_bound} with $B = SA$ we obtain the following guarantee.

\begin{lemma}\label{prop_ohrd}
\OSP with $b(s,a)=1$ verifies \, $\mathcal{C}_{\TREASURE}(\OSPmath, \delta) = \widetilde{O}\left( D^{3/2} S^2 A \right)$.
\end{lemma}%
%
We now compare this result to 
alternative approaches to the problem, showing that \OSP has state-of-the-art guarantee for \TREASURE (see App.\,\ref{app_alternative_approaches} for details).
\begin{itemize}[leftmargin=.1in,topsep=-2.5pt,itemsep=1pt,partopsep=0pt, parsep=0pt]

\item 
First, reward-free methods such as 
\citep{jin2020reward, zhang2020nearly, kaufmann2021adaptive, menard2020fast} 
are designed for finite-horizon problems so their guarantees cannot be directly translated to sample complexity for the \TREASURE problem. Nonetheless, we draw inspiration from their algorithmic principles and analyze a \textit{reward-free} variant of \UCRLtwo \citep{jaksch2010near, improved_analysis_UCRL2B}. 
Specifically we consider \ZEROONEUCRLtwo, which runs \UCRL by setting a reward of~$1$ to under-sampled states and~$0$ otherwise. 
However, we obtain a \TREASURE sample complexity for \ZEROONEUCRLtwo of $\wt{O}\left(\sum_{s \in \cS} D_s^3 S^2 A \right)$, which is always worse than the bound in Lem.\,\ref{prop_ohrd}.

\item Second, we can adapt the \MaxEnt approach \citep{hazan2019provably} to state-action coverage so that it targets a policy whose stationary state-action distribution $\lambda$ maximizes $H(\lambda) := - \sum_{s,a} \lambda(s,a) \log \lambda(s,a)$. While optimizing this entropy does not provably solve \TREASURE, it encourages us to take a \say{worst-case} approach w.r.t.\,the state-action visitations, and rather maximize $F(\lambda) := \min_{(s,a) \in \SA} \lambda(s,a)$. We show that the learning algorithm of \citep{cheung2019exploration} instantiated to maximize $F$ yields a \TREASURE sample complexity of at least $\Omega\left( \min\left\{ D^2 S^2 A / (\omega^{\star})^2, D^3 / (\omega^{\star})^3  \right\} \right)$ with $\omega^{\star} := \min_{\lambda} F(\lambda) \leq (SA)^{-1}$, which is significantly poorer than Lem.\,\ref{prop_ohrd}. In fact, in contrast to \MaxEnt-inspired methods that optimize for a single \textit{stationary} policy, \OSP realizes a non-stationary strategy that gradually collects the required samples by tackling successive learning~problems.

\end{itemize}

\subsection{Model Estimation (\MODESTtitle)}
\label{subsection_model_estimation}


We now study the problem of accurately estimating the unknown transition dynamics in a reward-free communicating environment. The objective was recently introduced in \citep{tarbouriech2020active} and we refer to it as the \textit{model-estimation problem}, or \MODEST for short. 

\begin{definition}\label{def_robust_modest}
    Given an accuracy level $\eta > 0$ and a confidence level $\delta \in (0,1)$, the \MODEST sample complexity of an online learning algorithm $\mathfrak{A}$ is defined as 
    \begin{align*}
        \mathcal{C}_{\MODEST}(\mathfrak{A}, \eta, \delta) := \min \big\{t > 0: \mathbb{P}\big( \forall (s,a) \in \mathcal{S} \times \mathcal{A}, \norm{\widehat{p}_{\mathfrak{A},t}(\cdot \vert s,a) - p(\cdot \vert s,a)}_{1} \leq \eta \big) \geq 1 - \delta \big\},
    \end{align*}
where $\widehat{p}_{\mathfrak{A},t}$ is the estimate (i.e., empirical average) of the transition dynamics $p$ after $t$ time steps.
\end{definition}%

Unlike in \TREASURE, here the sampling requirements are not immediately prescribed by the problem. To define the \SO-based algorithm we first upper-bound the estimation error using an empirical Bernstein inequality and then invert it to derive the amount of samples~$b_t(s,a)$ needed to achieve the desired level of accuracy~$\eta$ (see App.\,\ref{app_modest}). Specifically, letting $\widehat{\sigma}^2_{t}(s'\vert s,a) := \wh p_t(s'\vert s,a) (1-\wh p_t(s'\vert s,a))$ be the estimated variance of the transition from $(s,a)$ to $s'$ after $t$ steps, we set

\makeatletter
\newcommand*\mysizeb{%
   \@setfontsize\mysizeb{9.5}{9.5}%
}
\makeatother
\vspace{-0.18in}
\begin{mysizeb}
\begin{align}\label{budget_modest}
    b_t(s,a) &:= \Big\lceil \frac{57 (\sum_{s'} \widehat{\sigma}_{t}(s'\vert s,a))^2 }{\eta^2}  \log^2\left(\frac{8 e (\sum_{s'} \widehat{\sigma}_{t}(s'\vert s,a))^2 \sqrt{2 S A} }{\sqrt{\delta}\eta}   \right) + \frac{24 S}{\eta} \log\left( \frac{24 S^2 A}{\delta \eta} \right) \Big\rceil.
\end{align}%
\end{mysizeb}%

Since the estimated variance changes depending on the samples observed so far, the sampling requirements are \textit{adapted} over time. Given that $\widehat{\sigma}^2_{t}(s'\vert s,a) \leq 1/4$, $b_t(s,a)$ is always bounded so Thm.\,\ref{theorem_upper_bound} provides the following guarantee. 

\begin{lemma}\label{prop_arme}
Let $\Gamma := \max_{s,a} \norm{p(\cdot \vert s,a)}_0 \leq S$ be the maximal support of $p(\cdot \vert s,a)$ over the state-action pairs $(s,a)$. Running \OSP with the sampling requirements in Eq.\,\ref{budget_modest} yields
\begin{align*}
    &\mathcal{C}_{\MODEST}(\OSPmath, \eta, \delta) = \widetilde{O}\Big( \displaystyle\frac{D \Gamma S A}{\eta^2} + \frac{D S^2 A}{\eta} +  D^{3/2} S^2 A \Big).
    \end{align*}%
\end{lemma}%
Lem.\,\ref{prop_arme} improves over the result of~\citep{tarbouriech2020active} in two important aspects. First, the 
latter suffers from an inverse dependency on the stationary state-action distribution that optimizes a proxy objective function used in the derivation of their algorithm.
Second, while \citep{tarbouriech2020active} requires an ergodicity assumption, 
Lem.\,\ref{prop_arme} 
is the first sample complexity result for \MODEST in the more general communicating setting. 


\subsection{Goal-Free \& Cost-Free Exploration in Communicating MDPs} 
\label{subsection_costfree_goalfree}


We finally delve into the paradigm of \textit{reward-free exploration} introduced by~\citep{jin2020reward}: the objective of the agent is to collect enough information during the reward-free exploration phase, so that it can readily compute a near-optimal policy once \textit{any} reward function is provided. The problem has been analyzed in the \textit{finite-horizon} setting \citep[e.g.,][]{jin2020reward, menard2020fast, zhang2020nearly}. Here we study the more general and challenging setting of \textit{goal-conditioned} RL.\footnote{While an approach was proposed in \citep{tarbouriech2020improved}, it is restricted to considering only the \textit{incrementally} attainable goal states from a resettable reference state $s_0$.} We define the \textit{goal-free cost-free} objective as follows: after the exploration phase, the agent is expected to compute a near-optimal goal-conditioned policy for \textit{any} goal state and \textit{any} cost function (w.l.o.g.\,we consider a maximum possible cost $c_{\max}=1$). Recall that given a goal state $g$ and costs $c$, the (possibly unbounded) value function of a policy $\pi$ is
\begin{align*}
    V^{\pi}(s \rightarrow g) := \mathbb{E}\bigg[ \sum_{t=1}^{\tau_{\pi}(s \rightarrow g)} c(s_t, \pi(s_t)) ~\big\vert~ s_1 = s \bigg].
\end{align*}

Given a slack parameter $\theta \in [1, + \infty]$, we say that a policy $\wh \pi$ is $(\epsilon,\theta)$-optimal if\,\footnote{This reduces to standard $\epsilon$-optimality for $\theta = + \infty$. We only consider $\theta < + \infty$ in the case of minimum possible cost $c_{\min} = 0$ and it ensures that the algorithm targets proper policies (see App.\,\ref{app_extension_goalfree_costfree}).}
\begin{align*}
    V^{\wh{\pi}}(s \rightarrow g) \leq \min_{\pi : \,\mathbb{E}\left[ \tau_{\pi}(s \rightarrow g) \right] \leq \theta D_{s, g}} V^{\pi}(s \rightarrow g) + \epsilon.
\end{align*}
%
In this setting, constructing an efficient \SO-based algorithm is considerably more complex than \TREASURE and \MODEST. Relying on a sample complexity analysis for the fixed-goal SSP problem with a \textit{generative model} \citep{tarbou}, we define the (adaptive) number of samples needed in each state-action pair for our online objective. Although the number depends on the unknown diameter, we estimate $D$ using \OSP. The resulting sequence of sampling requirements is then fed online to \OSP. Combining the result of \citep{tarbou} and the properties of \OSP yields the following bound (see App.\,\ref{app_extension_goalfree_costfree}).

\newcommand{\whpi}{{\scalebox{0.92}{$\wh \pi$}}}

\begin{lemma}\label{theorem_goal_free_cost_free_main}
Consider any MDP satisfying Asm.\,\ref{asm_communicating} and the goal-free cost-free exploration problem with accuracy level $0 < \epsilon \leq 1$, confidence level $\delta \in (0,1)$, minimum cost $c_{\min} \in [0, 1]$, slack parameter $\theta \in [1, + \infty]$. We can instantiate \OSP so that its exploration phase (i.e., number of time steps) is bounded with probability at least $1-\delta$ by
\begin{align*}
        \wt{O}\left( \frac{D^4 \Gamma S A}{\omega \epsilon^2} + \frac{D^3 S^2 A}{\omega \epsilon} + \frac{D^3 \Gamma S A}{\omega^2} \right),
\end{align*}%
where $\omega := \max\big\{ c_{\min}, \epsilon / (\theta D) \big\} > 0$ (thus, either \mbox{$c_{\min} = 0$} or $\theta = + \infty$, but not both simultaneously). Following the exploration phase, the algorithm can compute in the planning phase, for any goal state $g \in \mathcal{S}$ and any cost function $c$ in $[c_{\min}, 1]$, a policy $\whpi_{g,c}$ that is $(\epsilon,\theta)$-optimal.
\end{lemma}

Lem.\,\ref{theorem_goal_free_cost_free_main} establishes the first sample complexity guarantee for general goal-free, cost-free exploration. While the objective is demanding and the upper bound on the length of the exploration phase can be large, the main purpose of this result is to showcase how \OSP can be readily instantiated to tackle a challenging exploration problem for which no existing solution can be easily leveraged. Comparing our analysis to the finite-horizon objective of \citep{jin2020reward} reveals two interesting properties:
\begin{itemize}[leftmargin=.1in,topsep=-3pt,itemsep=1pt,partopsep=0pt, parsep=0pt]
\item \textbf{\textit{The goal-free aspect:}} moving from finite-horizon to goal-conditioned renders \textit{unavoidable} both the communicating requirement (Asm.\,\ref{asm_communicating}) and the bound's dependency on the unknown diameter $D$ (which partly captures the role of the known horizon $H$ in the bound of \citep{jin2020reward}).
\item \textbf{\textit{The cost-free aspect:}} in contrast to finite-horizon, the value of $c_{\min}$ has an important impact on the type of performance guarantees we can obtain; in particular our analysis distinguishes between positive and non-negative costs (as also done in existing SSP analysis \citep{bertsekas2013stochastic, tarbouriech2019no, cohen2020near}).
\end{itemize}

\begin{figure}[t]
	\centering
	\begin{minipage}{0.13\linewidth}
		\includegraphics[width=0.99\linewidth]{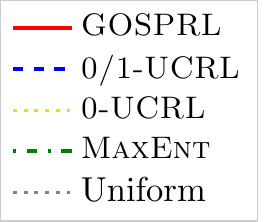}
	\end{minipage}\hfill
	\begin{minipage}{0.29\linewidth}
		\centering
		\hspace*{0.1cm}\includegraphics[width=1.03\linewidth,trim={0 0 0.1cm 0},clip]{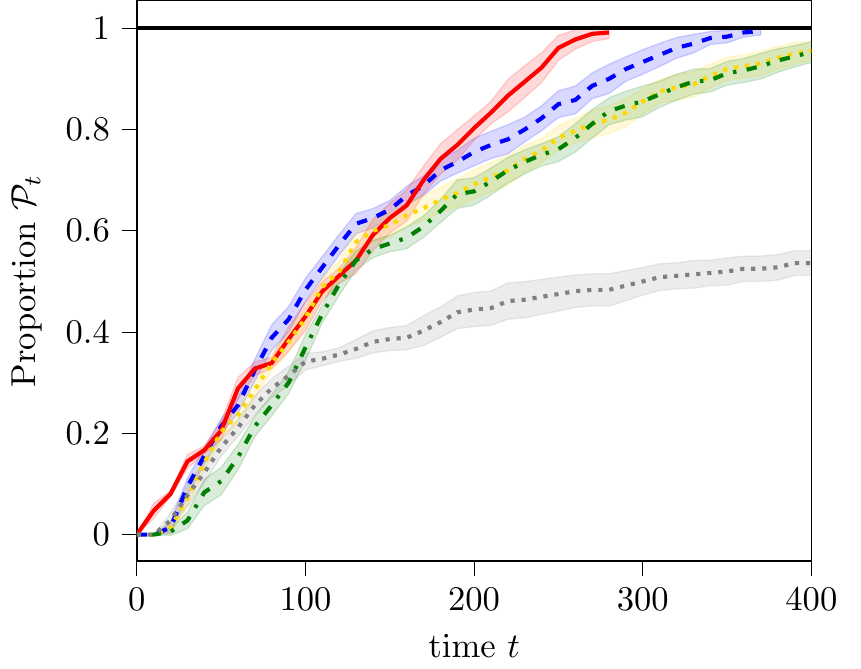}
	\end{minipage}\hfill
	\begin{minipage}{0.29\linewidth}
		\centering
		\hspace*{0.3cm}\includegraphics[width=0.96\linewidth,trim={0.68cm 0 0.1cm 0},clip]{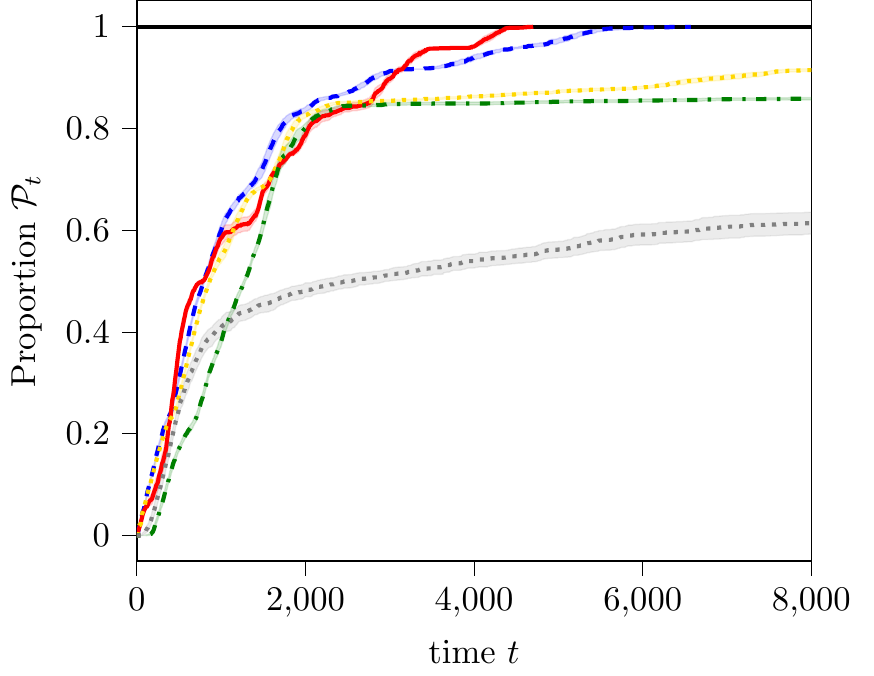}
	\end{minipage}\hfill
	\begin{minipage}{0.29\linewidth}
		\centering
		\hspace*{0.1cm}\includegraphics[width=0.96\linewidth,trim={0.68cm 0 0.1cm 0},clip]{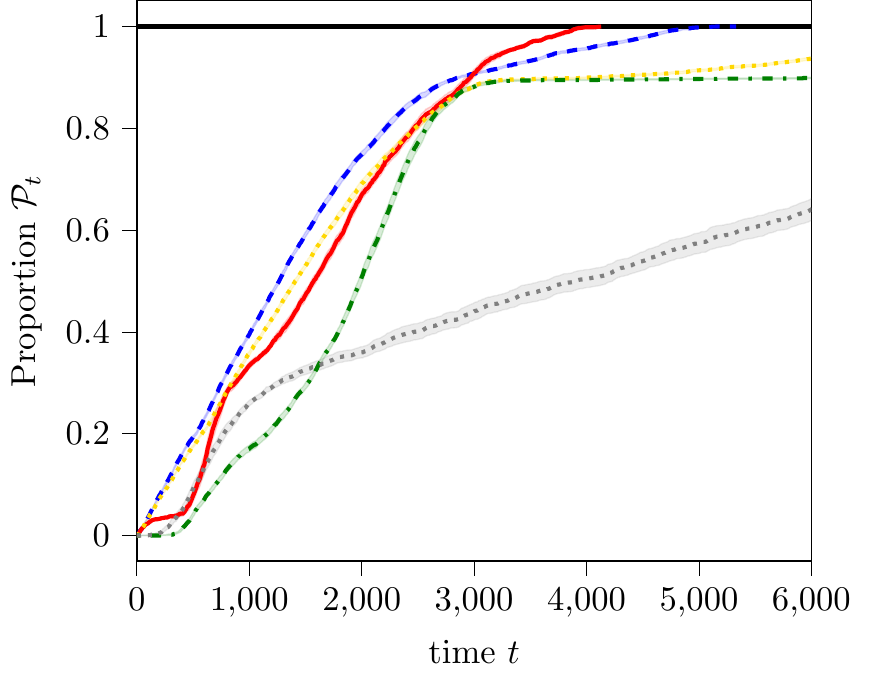}
	\end{minipage}
	\vspace{-0.05in}
	\caption{\small\TREASURETEN problem (i.e., with $b(s,a)=10$): Proportion $\mathcal{P}_t$ of states meeting the requirements at time $t$, averaged over $30$~runs. By definition of the sample complexity, the metric of interest is \textit{not} the rate of increase of $\mathcal{P}_t$ over time but only the time needed to reach the line of success $\mathcal{P}_t=1$. \textit{Left:} 6-state RiverSwim, \textit{Center:} 24-state corridor gridworld, \textit{Right:} 43-state $4$-room gridworld (see App.\,\ref{app_exp} for details on the domains).} 
	\label{fig:proportion_treasure}
\end{figure}

\begin{figure*}[t!]
	\centering
    \vspace{0.06in}
	\begin{minipage}{0.7\linewidth}
	\begin{minipage}{0.49\linewidth}
		\center
		\vspace*{-0.25cm}\includegraphics[width=0.99\linewidth]{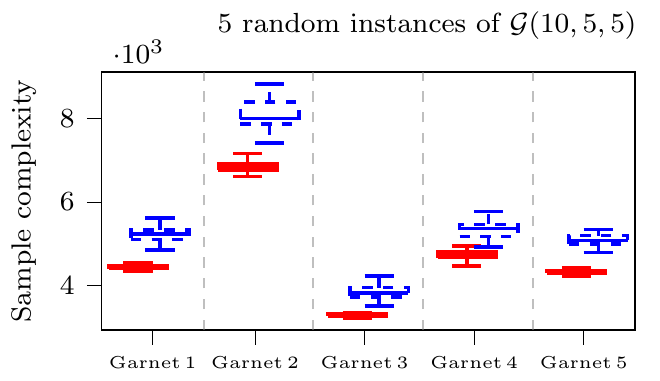}
	\end{minipage}\hfill
	\begin{minipage}{0.49\linewidth}
		\center
		\vspace*{-0.25cm}\includegraphics[width=0.97\linewidth,trim={0.5cm 0 0 0},clip]{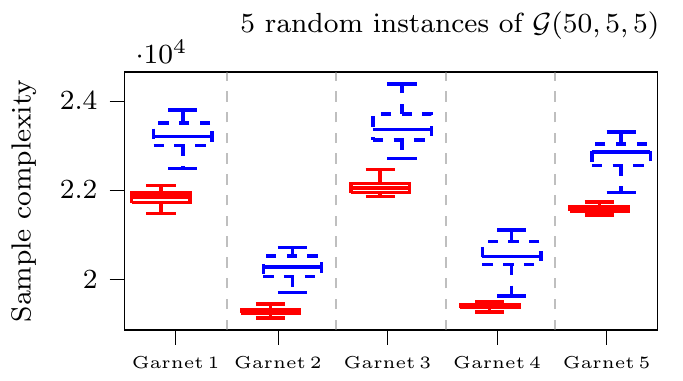}
	\end{minipage}
		\vspace{-0.03in}
	\captionof{figure}{\small Sample complexity boxplots of \OSPsmall (in \textcolor{lust}{red}) and \mbox{\ZEROONEUCRLtwosmall} (in \textcolor{blue}{blue}). Each column represents $30$ runs on a randomly generated Garnet $\mathcal{G}(S, A=5, \beta=5)$ with randomly generated state-action sampling requirements $b: \cS \times \cA \rightarrow \mathcal{U}(0, 100)$. \textit{Left:} $S=10$, \textit{Right:} $S=50$.}
	    \label{fig:boxplots}
	\vspace{0.1in}
	\begin{minipage}{0.33\linewidth}
		\flushleft
		\hspace*{0.1cm}\includegraphics[width=1.04\linewidth,trim={0 0 0.1cm 0},clip]{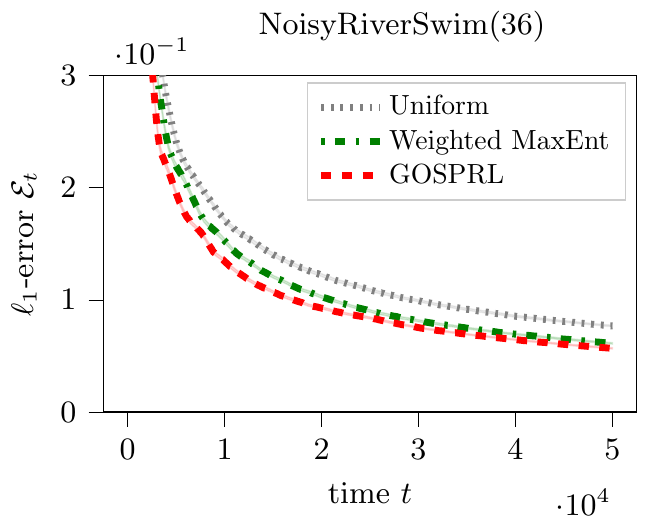}
	\end{minipage}\hfill
	\begin{minipage}{0.33\linewidth}
		\centering
		\hspace*{0.4cm}\includegraphics[width=0.97\linewidth,trim={0.5cm 0 0.1cm 0},clip]{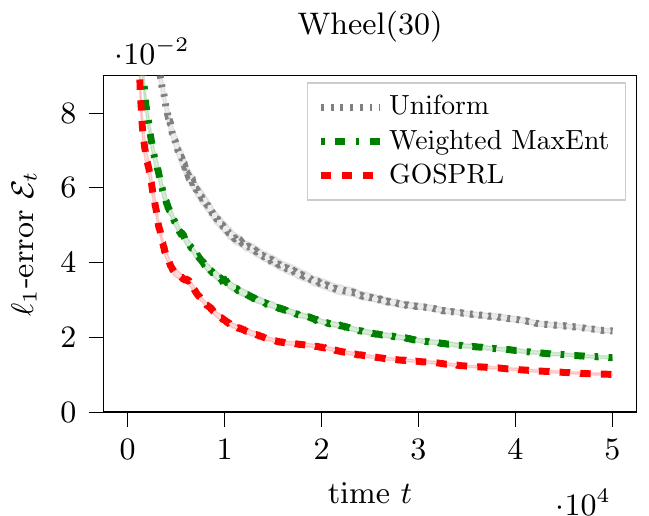}
	\end{minipage}\hfill
	\begin{minipage}{0.33\linewidth}
		\flushright
		\hspace*{0.4cm}\includegraphics[width=0.97\linewidth,trim={0.5cm 0 0.1cm 0},clip]{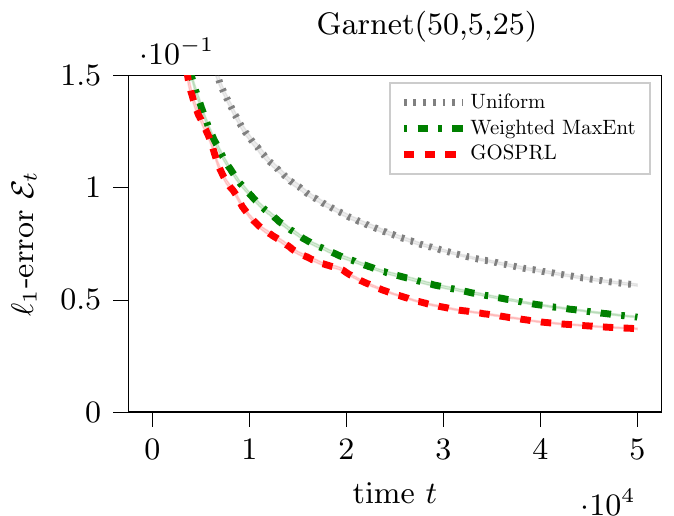}
	\end{minipage}
	\vspace{-0.03in}
	\captionof{figure}{\small \mbox{\MODESTsmall} problem: $\ell_1$-error $\mathcal{E}_t := (SA)^{-1} \cdot \sum_{s,a} \norm{ \wh{p}_t(\cdot \vert s,a) - p(\cdot \vert s,a)}_1$, averaged over 30 runs. \textit{Left:} NoisyRiverSwim(36), \textit{Center:} Wheel(30), \textit{Right:} Randomly generated Garnet $\mathcal{G}(50,5,25)$.}
	    \label{fig:modest}
	\end{minipage}
	\hfill\vline\hfill
		\begin{minipage}{0.25\linewidth}
		\flushright
		
    \begin{tikzpicture}
    \flushright
    \hspace*{-0.2cm}
                        \begin{scope}[scale=0.62]
	\tikzset{VertexStyle/.style = {draw, 
									shape          = circle,
	                                text           = black,
	                                inner sep      = 2pt,
	                                outer sep      = 0pt,
	                                minimum size   = 
	                                12 pt}}
	\tikzset{VertexStyle2/.style = {shape          = circle,
	                                text           = black,
	                                inner sep      = 2pt,
	                                outer sep      = 0pt,
	                                minimum size   = 12 pt}}
	\tikzset{Action/.style = {draw, 
                					shape          = circle,
	                                text           = black,
	                                fill           = black,
	                                inner sep      = 2pt,
	                                outer sep      = 0pt}}

	\node[VertexStyle](s0) at (0,0) {$ s_{0} $};
	\node[Action](a0s0) at (.7,.7){};
	\node[VertexStyle](s1) at (2.5,0){$s_1$};
	\node[Action](a0s1) at (1.5,0.){};
	\node[VertexStyle](s2) at (5,0){$s_2$};
	\node[Action](a0s2) at (4.3,-0.7){};
	\node[Action](a1s2) at (5,1){};
    
	\draw[->, >=latex, double, color=red](s0) to node[midway, right]{{\small $a_0$}} (a0s0);
    \draw[->, >=latex](a0s0) to [out=30,looseness=0.8] node[midway, xshift=1em]{$\nu$} (s1);   
	\draw[->, >=latex](a0s0) to [out=30,in=120,looseness=0.8] node[above]{$1-\nu$} (s2);

	\draw[->, >=latex, double, color=red](s1) to node[above,xshift=0.2em]{{\small $a_0$}} (a0s1);
	\draw[->, >=latex](a0s1) to (s0);
    
	\draw[->, >=latex, double, color=red](s2) to node[midway, left]{{\small $a_0$}} (a0s2);
    \draw[->, >=latex](a0s2) to [out=210, in=330, looseness=0.8] node[midway,xshift=-1em]{$\nu$} (s1);   
	\draw[->, >=latex](a0s2) to [out=210, in=300, looseness=0.8] node[above]{$1-\nu$} (s0);
	\draw[->, >=latex, double, color=red](s2) to node[midway, right]{{\small $a_1$}} (a1s2);
	\draw[->, >=latex](a1s2) to [out=30, in=30, looseness=1.2] (s2);

    \end{scope}
    
\end{tikzpicture}
\includegraphics[width=0.99\linewidth]{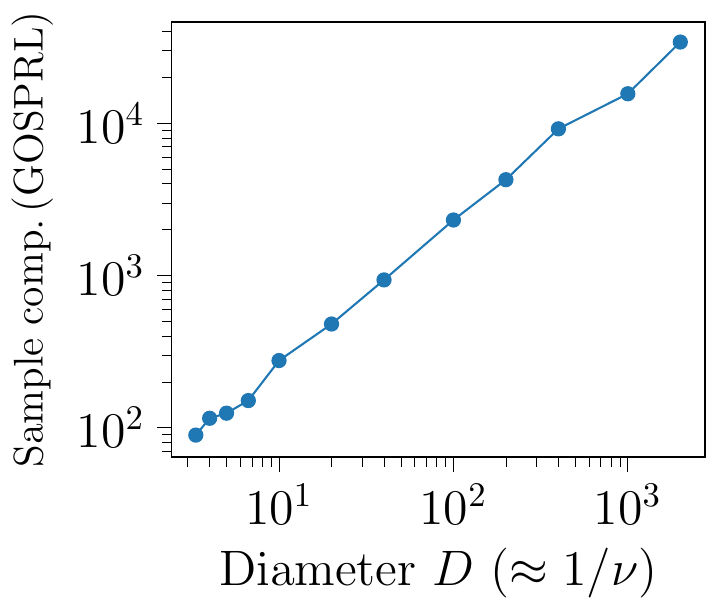}
\captionof{figure}{\small Simple three-state reward-free domain \citep{fruit2018efficient} and \TREASURETEN sample complexity of \OSP (averaged over 30 runs) as a function of the diameter $D \approx 1/\nu$.}
\label{toy_diameter}
		\end{minipage}
\vspace{-0.11in}
\end{figure*}


\section{Experiments}
\label{section_experiments}


In this section we report a preliminary numerical validation of our theoretical findings. While \OSP can be integrated in many different contexts, here we focus on the problems where our theory suggests that \OSP performs better than state-of-the-art online learning methods.

\paragraph{\TREASURE-type problem.} We consider a \TREASURE-type problem (Sect.\,\ref{subsection_treasure}), where for all $(s,a)$ we set $b(s,a) = 10$ instead of $1$ (we call it the \TREASURETEN problem).\footnote{Since \OSPfootnote and our baselines are all based on upper confidence bounds, they tend to display similar behaviors in the initial phases of learning, since the estimates when $N(s,a) = 0$ are similar. As the number of samples required in each state-action increases, the difference between the algorithms' design starts making a real difference in the behavior and eventually their performance. This is why we study here \TREASURETENfootnote instead of the \TREASUREONEfootnote problem for which empirical performance is comparable between learning algorithms.} We begin by showing in Fig.\,\ref{toy_diameter} that it is easy to construct a worst-case problem where the sample complexity scales linearly with the diameter, which is consistent with the theoretical discussion in Sect.\,\ref{sect_preliminaries} and \ref{so}. 


We compare to two heuristics based on \UCRLtwoB~\cite{jaksch2010near,improved_analysis_UCRL2B}: \mbox{\ZEROUCRLtwo}, where the reward used in computing the optimistic policy is set proportional to $([N(s,a) - b(s,a)]^+)^{-1/2}$, and \mbox{\ZEROONEUCRLtwo} with reward $1$ for undersampled state-action pairs and $0$ otherwise. We also compare with the \MaxEnt algorithm~\cite{cheung2019exploration} that maximizes entropy over the state-action space, and with a uniformly random baseline policy. We test on the RiverSwim domain~\cite{strehl2008analysis} and various gridworlds (see App.\,\ref{app_exp} for details and more results). Fig.\,\ref{fig:proportion_treasure} reports the proportion $\mathcal{P}_t$ of states that satisfy the sampling requirements at time $t$. Our metric of interest is the time needed to collect all required samples, and we see that \OSP reaches the $\mathcal{P}_t=1$ line of success consistently, and faster than \ZEROONEUCRLtwo, while the other heuristics struggle. The steady increase of $\mathcal{P}_t$ illustrates \OSP's design to progressively meet the sampling requirements, and not exhaust them state after state.

\paragraph{Random MDPs and sampling requirements.} To study the generality of \OSP to collect arbitrary sought-after samples, we further compare \OSP with \ZEROONEUCRLtwo which is the best heuristic from the previous experiment. We test on a variety of randomly generated \textit{configurations}, that we define as follows: each configuration corresponds to \textbf{i)} a randomly generated Garnet environment $\mathcal{G}(S,A,\beta)$ (with $S$ states, $A$ actions and branching factor~$\beta$, see \citep{bhatnagar2009natural}), and \textbf{ii)} randomly generated requirements $b(s,a) \in \mathcal{U}(0, \overline{U})$, where the maximum budget is set to $\overline{U}=100$ to have a wide range of possible requirements across each environment. The boxplots in Fig.\,\ref{fig:boxplots} provide aggregated statistics on the sample complexity for different configurations. We observe that \OSP consistently meets the sampling requirements faster than \ZEROONEUCRLtwo, as well as suffers from lower variance across runs.


\paragraph{\MODEST problem.} Finally, we empirically evaluate \OSP for the \MODEST problem (Sect.\,\ref{subsection_model_estimation}). We compare to the fully online \WeightedMaxEnt heuristic, which weighs the state-action entropy components with an optimistic estimate of the next-state transition variance and was shown in~\citep{tarbouriech2020active} to perform empirically better than algorithms with theoretical guarantees. We test on the two environments (NoisyRiverSwim and Wheel) proposed in~\citep{tarbouriech2020active} for their high level of stochasticity, as well as on a randomly generated Garnet. To facilitate the comparison, we consider a \OSP-for-\MODEST algorithm where the sampling requirements are computed using a decreasing error $\eta$ (see App.\,\ref{app_exp} for details). We observe in Fig.\,\ref{fig:modest} that \OSP outperforms the \WeightedMaxEnt heuristic.

\section{Conclusion} \label{sec_conclusion}

In this paper, we introduced the online learning problem of simulating a sampling oracle (Sect.\,\ref{sect_preliminaries}) and derived the algorithm \OSP with its sample complexity guarantee (Sect.\,\ref{so}). We then illustrated how it can be used to tackle in a unifying fashion a variety of applications without having to design a specific online algorithm for each, while at the same time obtaining improved or novel sample complexity guarantees (Sect.\,\ref{ex}). Going forward, we believe that \OSP can be used as a competitive off-the-shelf baseline when a new application is introduced.

An exciting direction of future investigation is to 
extend the general sample collection problem and its various applications beyond the tabular setting. Handling a continuous state space or linear function approximation requires redefining the notion of reaching a specific state (e.g., via adequate discretization or by considering requirements based on the covariance matrix). Studying the SSP problem beyond tabular may provide insights, as recently initiated in \citep{vial2021regret} in linear function approximation under the assumption that all policies are proper. 
On the more algorithmic side, \OSP hinges on knowing the sampling requirement function $b_t$ and deriving a shortest-path policy $\widetilde{\pi}$. Interestingly, we can identify algorithmic counterparts to both modules in deep RL. The computation of $\widetilde{\pi}$ can be entrusted to a goal-conditioned network (using e.g.,~\citep{andrychowicz2017hindsight}), while the specification of $b_t$ can be related to goal-sampling selection mechanisms that elect hard-to-reach~\citep{florensa2018automatic} or rare~\citep{pong2020skew} states as goals.

\vspace{0.3in}


\bibliographystyle{abbrv}
\bibliography{bibliography.bib}

\begin{thebibliography}{10}

\bibitem{agarwal2019optimality}
A.~Agarwal, S.~Kakade, and L.~F. Yang.
\newblock Model-based reinforcement learning with a generative model is minimax
  optimal.
\newblock In {\em Proceedings of Thirty Third Conference on Learning Theory},
  volume 125 of {\em Proceedings of Machine Learning Research}. PMLR, 2020.

\bibitem{andrychowicz2017hindsight}
M.~Andrychowicz, F.~Wolski, A.~Ray, J.~Schneider, R.~Fong, P.~Welinder,
  B.~McGrew, J.~Tobin, P.~Abbeel, and W.~Zaremba.
\newblock Hindsight experience replay.
\newblock In {\em Proceedings of the 31st International Conference on Neural
  Information Processing Systems}, pages 5055--5065, 2017.

\bibitem{audibert2010best}
J.-Y. Audibert and S.~Bubeck.
\newblock Best arm identification in multi-armed bandits.
\newblock In {\em COLT - 23th Conference on Learning Theory}, 2010.

\bibitem{audibert2009exploration}
J.-Y. Audibert, R.~Munos, and C.~Szepesv{\'a}ri.
\newblock Exploration--exploitation tradeoff using variance estimates in
  multi-armed bandits.
\newblock {\em Theoretical Computer Science}, 410(19):1876--1902, 2009.

\bibitem{azar2013minimax}
M.~G. Azar, R.~Munos, and H.~J. Kappen.
\newblock Minimax pac bounds on the sample complexity of reinforcement learning
  with a generative model.
\newblock {\em Machine learning}, 91(3):325--349, 2013.

\bibitem{azar2017minimax}
M.~G. Azar, I.~Osband, and R.~Munos.
\newblock Minimax regret bounds for reinforcement learning.
\newblock In {\em Proceedings of the 34th International Conference on Machine
  Learning-Volume 70}, pages 263--272. JMLR, 2017.

\bibitem{bartlett2019scale-free}
P.~Bartlett, V.~Gabillon, J.~Healey, and M.~Valko.
\newblock Scale-free adaptive planning for deterministic dynamics \& discounted
  rewards.
\newblock In {\em International Conference on Machine Learning}, pages
  495--504. PMLR, 2019.

\bibitem{bartlett2009regal}
P.~L. Bartlett and A.~Tewari.
\newblock Regal: a regularization based algorithm for reinforcement learning in
  weakly communicating mdps.
\newblock In {\em Proceedings of the Twenty-Fifth Conference on Uncertainty in
  Artificial Intelligence}, 2009.

\bibitem{bertsekas1995dynamic}
D.~P. Bertsekas.
\newblock {\em Dynamic programming and optimal control}, volume~1.
\newblock Athena scientific Belmont, MA, 1995.

\bibitem{bertsekas1991analysis}
D.~P. Bertsekas and J.~N. Tsitsiklis.
\newblock An analysis of stochastic shortest path problems.
\newblock {\em Mathematics of Operations Research}, 16(3):580--595, 1991.

\bibitem{bertsekas2013stochastic}
D.~P. Bertsekas and H.~Yu.
\newblock Stochastic shortest path problems under weak conditions.
\newblock {\em Lab. for Information and Decision Systems Report LIDS-P-2909,
  MIT}, 2013.

\bibitem{bhatnagar2009natural}
S.~Bhatnagar, R.~S. Sutton, M.~Ghavamzadeh, and M.~Lee.
\newblock Natural actor--critic algorithms.
\newblock {\em Automatica}, 45(11):2471--2482, 2009.

\bibitem{brafman2002r}
R.~I. Brafman and M.~Tennenholtz.
\newblock R-max-a general polynomial time algorithm for near-optimal
  reinforcement learning.
\newblock {\em Journal of Machine Learning Research}, 3(Oct):213--231, 2002.

\bibitem{carpentier2011upper}
A.~Carpentier, A.~Lazaric, M.~Ghavamzadeh, R.~Munos, and P.~Auer.
\newblock Upper-confidence-bound algorithms for active learning in multi-armed
  bandits.
\newblock In {\em International Conference on Algorithmic Learning Theory},
  2011.

\bibitem{chen2018scalable}
Y.~Chen, L.~Li, and M.~Wang.
\newblock Scalable bilinear $\pi$ learning using state and action features.
\newblock {\em arXiv preprint arXiv:1804.10328}, 2018.

\bibitem{chen2016stochastic}
Y.~Chen and M.~Wang.
\newblock Stochastic primal-dual methods and sample complexity of reinforcement
  learning.
\newblock {\em arXiv preprint arXiv:1612.02516}, 2016.

\bibitem{cheung2019exploration}
W.~C. Cheung.
\newblock Exploration-exploitation trade-off in reinforcement learning on
  online markov decision processes with global concave rewards.
\newblock {\em arXiv preprint arXiv:1905.06466}, 2019.

\bibitem{cheung2019regret}
W.~C. Cheung.
\newblock Regret minimization for reinforcement learning with vectorial
  feedback and complex objectives.
\newblock In {\em Advances in Neural Information Processing Systems}, pages
  724--734, 2019.

\bibitem{dekel2014bandits}
O.~Dekel, J.~Ding, T.~Koren, and Y.~Peres.
\newblock Bandits with switching costs: ${T}^{2/3}$ regret.
\newblock In {\em Proceedings of the forty-sixth annual ACM symposium on Theory
  of computing}, pages 459--467, 2014.

\bibitem{d1963probabilistic}
F.~d'Epenoux.
\newblock A probabilistic production and inventory problem.
\newblock {\em Management Science}, 10(1):98--108, 1963.

\bibitem{florensa2018automatic}
C.~Florensa, D.~Held, X.~Geng, and P.~Abbeel.
\newblock Automatic goal generation for reinforcement learning agents.
\newblock In {\em International conference on machine learning}, pages
  1515--1528. PMLR, 2018.

\bibitem{improved_analysis_UCRL2B}
R.~Fruit, M.~Pirotta, and A.~Lazaric.
\newblock Improved analysis of ucrl2 with empirical bernstein inequality.
\newblock {\em arXiv preprint arXiv:2007.05456}, 2020.

\bibitem{fruit2018efficient}
R.~Fruit, M.~Pirotta, A.~Lazaric, and R.~Ortner.
\newblock Efficient bias-span-constrained exploration-exploitation in
  reinforcement learning.
\newblock In {\em ICML 2018-The 35th International Conference on Machine
  Learning}, volume~80, pages 1578--1586, 2018.

\bibitem{garcelon2020conservative}
E.~Garcelon, M.~Ghavamzadeh, A.~Lazaric, and M.~Pirotta.
\newblock Conservative exploration in reinforcement learning.
\newblock In {\em International Conference on Artificial Intelligence and
  Statistics}, pages 1431--1441. PMLR, 2020.

\bibitem{grill2016blazing}
J.-B. Grill, M.~Valko, and R.~Munos.
\newblock Blazing the trails before beating the path: Sample-efficient
  monte-carlo planning.
\newblock In {\em Advances in Neural Information Processing Systems}, pages
  4680--4688, 2016.

\bibitem{hazan2019provably}
E.~Hazan, S.~Kakade, K.~Singh, and A.~Van~Soest.
\newblock Provably efficient maximum entropy exploration.
\newblock In {\em International Conference on Machine Learning}, pages
  2681--2691, 2019.

\bibitem{hsu2015mixing}
D.~J. Hsu, A.~Kontorovich, and C.~Szepesv{\'a}ri.
\newblock Mixing time estimation in reversible markov chains from a single
  sample path.
\newblock In {\em Advances in neural information processing systems}, pages
  1459--1467, 2015.

\bibitem{jaksch2010near}
T.~Jaksch, R.~Ortner, and P.~Auer.
\newblock Near-optimal regret bounds for reinforcement learning.
\newblock {\em Journal of Machine Learning Research}, 11(Apr):1563--1600, 2010.

\bibitem{jin2018is-q-learning}
C.~Jin, Z.~Allen-Zhu, S.~Bubeck, and M.~I. Jordan.
\newblock Is q-learning provably efficient?
\newblock In {\em Advances in Neural Information Processing Systems}, pages
  4863--4873, 2018.

\bibitem{jin2020reward}
C.~Jin, A.~Krishnamurthy, M.~Simchowitz, and T.~Yu.
\newblock Reward-free exploration for reinforcement learning.
\newblock In {\em International Conference on Machine Learning}, pages
  4870--4879. PMLR, 2020.

\bibitem{kaufmann2021adaptive}
E.~Kaufmann, P.~M{\'e}nard, O.~D. Domingues, A.~Jonsson, E.~Leurent, and
  M.~Valko.
\newblock Adaptive reward-free exploration.
\newblock In {\em Algorithmic Learning Theory}, pages 865--891. PMLR, 2021.

\bibitem{kazerouni2017conservative}
A.~Kazerouni, M.~Ghavamzadeh, Y.~Abbasi, and B.~Van~Roy.
\newblock Conservative contextual linear bandits.
\newblock In {\em Advances in Neural Information Processing Systems}, pages
  3910--3919, 2017.

\bibitem{kearns2002sparse}
M.~Kearns, Y.~Mansour, and A.~Y. Ng.
\newblock A sparse sampling algorithm for near-optimal planning in large markov
  decision processes.
\newblock {\em Machine learning}, 49(2-3):193--208, 2002.

\bibitem{kearns2002near}
M.~Kearns and S.~Singh.
\newblock Near-optimal reinforcement learning in polynomial time.
\newblock {\em Machine learning}, 49(2-3):209--232, 2002.

\bibitem{kearns2000approximate}
M.~J. Kearns, Y.~Mansour, and A.~Y. Ng.
\newblock Approximate planning in large pomdps via reusable trajectories.
\newblock In {\em Advances in Neural Information Processing Systems}, pages
  1001--1007, 2000.

\bibitem{koren2017bandits}
T.~Koren, R.~Livni, and Y.~Mansour.
\newblock Bandits with movement costs and adaptive pricing.
\newblock In {\em Conference on Learning Theory}, pages 1242--1268. PMLR, 2017.

\bibitem{lattimore2019bandit}
T.~Lattimore and C.~Szepesv{\'a}ri.
\newblock {\em Bandit algorithms}.
\newblock Cambridge University Press, 2020.

\bibitem{li2020breaking}
G.~Li, Y.~Wei, Y.~Chi, Y.~Gu, and Y.~Chen.
\newblock Breaking the sample size barrier in model-based reinforcement
  learning with a generative model.
\newblock {\em Advances in Neural Information Processing Systems}, 2020.

\bibitem{menard2020fast}
P.~M{\'e}nard, O.~D. Domingues, A.~Jonsson, E.~Kaufmann, E.~Leurent, and
  M.~Valko.
\newblock Fast active learning for pure exploration in reinforcement learning.
\newblock In {\em International Conference on Machine Learning}, pages
  7599--7608. PMLR, 2021.

\bibitem{ortner2020regret}
R.~Ortner.
\newblock Regret bounds for reinforcement learning via markov chain
  concentration.
\newblock {\em Journal of Artificial Intelligence Research}, 67:115--128, 2020.

\bibitem{paulin2015concentration}
D.~Paulin.
\newblock Concentration inequalities for markov chains by marton couplings and
  spectral methods.
\newblock {\em Electronic Journal of Probability}, 20, 2015.

\bibitem{pong2020skew}
V.~Pong, M.~Dalal, S.~Lin, A.~Nair, S.~Bahl, and S.~Levine.
\newblock Skew-fit: State-covering self-supervised reinforcement learning.
\newblock In {\em International Conference on Machine Learning}, pages
  7783--7792. PMLR, 2020.

\bibitem{puterman2014markov}
M.~L. Puterman.
\newblock {\em Markov Decision Processes.: Discrete Stochastic Dynamic
  Programming}.
\newblock John Wiley \& Sons, 2014.

\bibitem{jian2019exploration}
J.~Qian, R.~Fruit, M.~Pirotta, and A.~Lazaric.
\newblock Exploration bonus for regret minimization in discrete and continuous
  average reward mdps.
\newblock In {\em Advances in Neural Information Processing Systems}, pages
  4891--4900, 2019.

\bibitem{cohen2020near}
A.~Rosenberg, A.~Cohen, Y.~Mansour, and H.~Kaplan.
\newblock Near-optimal regret bounds for stochastic shortest path.
\newblock In {\em International Conference on Machine Learning}, pages
  8210--8219. PMLR, 2020.

\bibitem{sidford2018near}
A.~Sidford, M.~Wang, X.~Wu, L.~Yang, and Y.~Ye.
\newblock Near-optimal time and sample complexities for solving markov decision
  processes with a generative model.
\newblock In {\em Advances in Neural Information Processing Systems}, pages
  5186--5196, 2018.

\bibitem{strehl2009reinforcement}
A.~L. Strehl, L.~Li, and M.~L. Littman.
\newblock Reinforcement learning in finite mdps: Pac analysis.
\newblock {\em Journal of Machine Learning Research}, 10(Nov):2413--2444, 2009.

\bibitem{strehl2008analysis}
A.~L. Strehl and M.~L. Littman.
\newblock An analysis of model-based interval estimation for markov decision
  processes.
\newblock {\em Journal of Computer and System Sciences}, 74(8):1309--1331,
  2008.

\bibitem{szorenyi2014optimistic}
B.~Sz{\"o}r{\'e}nyi, G.~Kedenburg, and R.~Munos.
\newblock Optimistic planning in markov decision processes using a generative
  model.
\newblock {\em Advances in Neural Information Processing Systems},
  27:1035--1043, 2014.

\bibitem{tarbouriech2019no}
J.~Tarbouriech, E.~Garcelon, M.~Valko, M.~Pirotta, and A.~Lazaric.
\newblock No-regret exploration in goal-oriented reinforcement learning.
\newblock In {\em International Conference on Machine Learning}, pages
  9428--9437. PMLR, 2020.

\bibitem{tarbouriech2019active}
J.~Tarbouriech and A.~Lazaric.
\newblock Active exploration in markov decision processes.
\newblock In {\em The 22nd International Conference on Artificial Intelligence
  and Statistics}, pages 974--982, 2019.

\bibitem{tarbouriech2020improved}
J.~Tarbouriech, M.~Pirotta, M.~Valko, and A.~Lazaric.
\newblock Improved sample complexity for incremental autonomous exploration in
  mdps.
\newblock In {\em Advances in Neural Information Processing Systems},
  volume~33, pages 11273--11284, 2020.

\bibitem{tarbou}
J.~Tarbouriech, M.~Pirotta, M.~Valko, and A.~Lazaric.
\newblock Sample complexity bounds for stochastic shortest path with a
  generative model.
\newblock In {\em Algorithmic Learning Theory}, pages 1157--1178. PMLR, 2021.

\bibitem{tarbouriech2020active}
J.~Tarbouriech, S.~Shekhar, M.~Pirotta, M.~Ghavamzadeh, and A.~Lazaric.
\newblock Active model estimation in markov decision processes.
\newblock In {\em Conference on Uncertainty in Artificial Intelligence}, pages
  1019--1028. PMLR, 2020.

\bibitem{trevizan2016heuristic}
F.~Trevizan, S.~Thi{\'e}baux, P.~Santana, and B.~Williams.
\newblock Heuristic search in dual space for constrained stochastic shortest
  path problems.
\newblock In {\em Twenty-Sixth International Conference on Automated Planning
  and Scheduling}, 2016.

\bibitem{vial2021regret}
D.~Vial, A.~Parulekar, S.~Shakkottai, and R.~Srikant.
\newblock Regret bounds for stochastic shortest path problems with linear
  function approximation.
\newblock {\em arXiv preprint arXiv:2105.01593}, 2021.

\bibitem{wainwright}
M.~Wainwright.
\newblock Course on mathematical statistics, chapter 2: Basic tail and
  concentration bounds. {U}niversity of {C}alifornia at {B}erkeley,
  {D}epartment of {S}tatistics, 2015.

\bibitem{wang2017primal}
M.~Wang.
\newblock Primal-dual $\pi$ learning: Sample complexity and sublinear run time
  for ergodic markov decision problems.
\newblock {\em arXiv preprint arXiv:1710.06100}, 2017.

\bibitem{wang2019q}
Y.~Wang, K.~Dong, X.~Chen, and L.~Wang.
\newblock Q-learning with ucb exploration is sample efficient for
  infinite-horizon mdp.
\newblock In {\em International Conference on Learning Representations}, 2019.

\bibitem{wei2019model}
C.-Y. Wei, M.~J. Jahromi, H.~Luo, H.~Sharma, and R.~Jain.
\newblock Model-free reinforcement learning in infinite-horizon average-reward
  markov decision processes.
\newblock In {\em International Conference on Machine Learning}, pages
  10170--10180. PMLR, 2020.

\bibitem{zanette2019tighter}
A.~Zanette and E.~Brunskill.
\newblock Tighter problem-dependent regret bounds in reinforcement learning
  without domain knowledge using value function bounds.
\newblock In {\em International Conference on Machine Learning}, pages
  7304--7312, 2019.

\bibitem{zanette2019almost}
A.~Zanette, M.~J. Kochenderfer, and E.~Brunskill.
\newblock Almost horizon-free structure-aware best policy identification with a
  generative model.
\newblock In {\em Advances in Neural Information Processing Systems}, pages
  5626--5635, 2019.

\bibitem{zhang2020taskagnostic}
X.~Zhang, Y.~Ma, and A.~Singla.
\newblock Task-agnostic exploration in reinforcement learning.
\newblock In {\em 34th Conference on Neural Information Processing Systems},
  pages 11734--11743, 2020.

\bibitem{zhang2020nearly}
Z.~Zhang, S.~S. Du, and X.~Ji.
\newblock Nearly minimax optimal reward-free reinforcement learning.
\newblock {\em arXiv preprint arXiv:2010.05901}, 2020.

\bibitem{zhang2019regret}
Z.~Zhang and X.~Ji.
\newblock Regret minimization for reinforcement learning by evaluating the
  optimal bias function.
\newblock In {\em Advances in Neural Information Processing Systems}, pages
  2823--2832, 2019.

\bibitem{zhang2020almost}
Z.~Zhang, Y.~Zhou, and X.~Ji.
\newblock Almost optimal model-free reinforcement learning via
  reference-advantage decomposition.
\newblock {\em Advances in Neural Information Processing Systems}, 33, 2020.

\end{thebibliography}




\newpage

\appendix

\captionsetup[figure]{labelfont=normalsize,font=normalsize}

\part{Appendix}

\parttoc


\section{Efficient Computation of Optimistic SSP Policy}
\label{app_value_iteration_SSP}

In this section, we recall how to compute an optimistic stochastic shortest path (SSP) policy using an extended value iteration (EVI) scheme tailored to SSP, as explained in~\citep{tarbouriech2019no}. Here we leverage a Bernstein-based construction of confidence intervals, as done by e.g.,~\citep{improved_analysis_UCRL2B, cohen2020near}. For details on the SSP formulation, we refer to e.g.,~\citep[][Sect.\,3]{bertsekas1995dynamic}.

Consider as input an SSP-MDP instance $M^{\dagger} := \langle \mathcal{S}^{\dagger}, \mathcal{A}, c, p, s^{\dagger} \rangle$, with goal $s^{\dagger}$, non-goal states $\cS^{\dagger} = \cS \setminus \{ s^{\dagger} \}$, actions~$\cA$, unknown dynamics $p$, and known cost function with costs in $[c_{\min}, 1]$ where $c_{\min} > 0$. We assume that there exists at least one proper policy (i.e., that reaches the goal $s^{\dagger}$ with probability one when starting from any state in $\cS^{\dagger}$). Note that in particular such condition is verified under Asm.\,\ref{asm_communicating}. We denote by $N(s,a)$ the current number of samples available at the state-action pair $(s,a)$ and set $N^{+}(s,a) := \max \{1,N(s,a) \} $. We also denote by $\wh{p}$ the current empirical average of transitions: $\wh{p}(s'|s,a) = N(s,a,s') / N(s,a)$.

The algorithm first computes a set of plausible SSP-MDPs defined as $$\mathcal{M}^{\dagger} := \{ \langle \mathcal{S}^{\dagger}, \mathcal{A}, c, \widetilde{p}, s^{\dagger} \rangle ~\vert ~ \widetilde{p}(s^{\dagger}|s^{\dagger},a) =~1,~ \widetilde{p}(s'|s,a) \in \mathcal{B}(s,a,s'),~ \sum_{s'} \wt{p}(s' \vert s,a) = 1\},$$ where for any $(s,a) \in \cS^{\dagger}\times \cA$, $\mathcal{B}(s,a,s')$ is a high-probability confidence set on the dynamics of the true SSP-MDP $M^{\dagger}$. Specifically, we define the compact sets $\mathcal{B}(s,a,s') := [ \wh{p}(s' \vert s,a) - \beta(s,a,s'), \wh{p}(s' \vert s,a) + \beta(s,a,s')] \cap [0, 1]$, where
\begin{align*}
    \beta(s,a,s') := 2 \sqrt{\frac{\wh{\sigma}^2(s' \vert s,a)}{N^+(s,a)} \log\left(\frac{2SAN^+(s,a)}{\delta}\right)} + \frac{6 \log\left(\frac{2SAN^+(s,a)}{\delta}\right)}{N^+(s,a)},
\end{align*}
where $\wh{\sigma}^2(s' \vert s,a) := \wh{p}(s' \vert s,a) (1 - \wh{p}(s' \vert s,a))$ is the variance of the empirical transition $\wh{p}(s' \vert s,a)$. Importantly, the choice of $\beta(s,a,s')$ guarantees that $M^{\dagger} \in \mathcal{M}^{\dagger}$ with high probability. Indeed, let us now spell out the high-probability event. Denote by $\mathcal{E}$ the event under which for any time step $t \geq 1$ and for any state-action pair $(s,a) \in \cS \times \cA$ and next state $s' \in \cS$, it holds that
\begin{align}\label{eq_bernstein}
    \abs{\widehat{p}_{t}(s' \vert s,a) - p(s' \vert s,a)} \leq \beta_t(s,a,s').
\end{align}
Given the way the confidence intervals are constructed using the empirical Bernstein inequality \citep[see e.g.,][]{improved_analysis_UCRL2B, cohen2020near}, we have $\mathbb{P}(\mathcal{E}) \geq 1 - \delta$. Throughout the remainder of the analysis, we will assume that the event $\mathcal{E}$ holds.

Once $\mathcal{M}^{\dagger}$ has been computed, the algorithm applies an extended value iteration (\EVI) scheme to compute a policy with lowest optimistic value. Formally, it defines the extended optimal Bellman operator $\widetilde{\mathcal{L}}$ such that for any vector $\wt{v} \in \mathbb{R}^{S^{\dagger}}$ and non-goal state $s \in \cS^{\dagger}$,
\begin{align*}
    \wt{\mathcal{L}}\wt{v}(s) := \min_{a \in \cA} \Big\{ c(s,a) + \min_{\widetilde{p}\in \mathcal{B}(s,a)} \sum_{s' \in  \cS^{\dagger}} \wt{p}(s' \vert s,a) \wt{v}(s') \Big\}.
\end{align*}
We consider an initial vector $\wt{v}_0 := 0$ and set iteratively $\wt{v}_{i+1} := \wt{\mathcal{L}} \wt{v}_{i}$. For a predefined \VI precision $\gammaVI > 0$, the stopping condition is reached for the first iteration $j$ such that $\norm{\wt{v}_{j+1} - \wt{v}_{j} }_{\infty} \leq \gammaVI$. The policy $\wt{\pi}$ is then selected to be the optimistic greedy policy w.r.t.\,the vector $\wt{v}_j$. While $\wt{v}_j$ is \textit{not} the value function of $\wt{\pi}$ in the optimistic model $\wt{p}$, which we denote by $\wt{V}^{\wt{\pi}}$, both quantities can be related according to the following lemma, which is a simple adaptation of~\citep[][Lem.~4 \& App.\,E]{tarbouriech2019no}. We denote by ${V}^\star$ (resp.\,$\wt{V}^\star$) the optimal value function in the true (resp.\,optimistic) SSP instance.

\begin{lemma}\label{lemma_app_value_iteration_SSP}
Under the event $\mathcal{E}$, the following component-wise inequalities hold: \textbf{1)} $\wt{v}_j \leq V^{\star}$, ~ \textbf{2)} $\wt{v}_j \leq \wt{V}^{\star} \leq \wt{V}^{\wt{\pi}}$, ~ \textbf{3)} If the \VI precision level verifies $\gammaVI \leq \frac{c_{\min}}{2}$, then $\wt{V}^{\wt{\pi}} \leq \left(1 + \frac{2 \gammaVI}{c_{\min}}\right) \wt{v}_j$.
\end{lemma}

Note that for the purposes of \OSP (Alg.\,\ref{algo:SO}), the \VI precision  $\gammaVI$ can for example be selected as in \citep{tarbouriech2019no} equal to $1/(2t_k)$ with $t_k$ the current time step, which only translates in a negligible, lower-order error in the sample complexity result of Thm.\,\ref{theorem_upper_bound}.



\section{Algorithmic Variants of \OSPtitle}
\label{sub_sect_subroutines}

The algorithmic design of \OSP (Alg.\,\ref{algo:SO}) is conceptually simple and it can flexibly incorporate a number of modifications driven by the agent's desiderata or possible prior knowledge.
\begin{itemize}[leftmargin=.2in,topsep=-3pt,itemsep=1pt,partopsep=0pt, parsep=0pt]
\item Any non-unit SSP costs can be designed as long as they are positive and bounded: detering costs may e.g., be assigned to \say{trap} states with large negative environmental reward that the agent may seek to avoid.
\item Penalizing the visitation of sufficiently visited states (with carefully selected larger-than-one costs) may give the agent incentive to \say{even out} its sample collection and thus avoid over-sampling some areas of the state-action space. 
\item It is possible to change the construction of the SSP problem and focus on specific goal states instead of considering all under-sampled states as goals. In practice, using such a \textit{meta-goal} makes the optimal SSP policy more robust to noise. While the SSP solution to $M_{k}$ indeed seeks to reach the closest under-sampled state, random transitions may move the agent closer to any other state in $\Gk$ and this would naturally trigger the policy to focus on such closer state. On the other hand, providing the SSP policy with a single goal state may lead to much longer and wasteful attempts.
\item Finally we remark that if the entire state space is initially under-sampled, any action would produce a ``useful'' sample and different heuristics can be implemented in prioritizing actions accordingly.
\end{itemize}

In the following, we delve into such goal-selection (App.\,\ref{subsection_selecting_target_states}) and cost-shaping (App.\,\ref{subsection_selecting_costs}) variants of \OSP, which do not affect the sample complexity bound of Thm.\,\ref{theorem_upper_bound}.

\subsection{Selecting the goal state}
\label{subsection_selecting_target_states}
In Alg.\,\ref{algo:SO}, each attempt $k$ casts as goals the states that are under-sampled w.r.t.\,the sampling requirements so far. As mentioned above, having such multiple goals is algorithmically appealing as it reduces the number of attempts that fail to collect a desired sample. Although specifically eliciting a single valid (i.e., under-sampled) goal state at each attempt may yield poorer performance, the resulting sample complexity guarantee would be the same as in Thm.\,\ref{theorem_upper_bound}. We can then distinguish between two strategies for goal state selection.

The first strategy prioritizes the states that appear harder to be successfully sampled. This is sensible when the aim is to have the most even possible sample collection over time so as to shy away from purely local exploration. For instance, the learning agent can select as goal state the least-sampled state so far, i.e., $\overline{s}_k \in \argmin_{ s \in \Gk} N_k(s)$.

On the other hand, the second strategy prioritizes the states that appear easier to be successfully sampled. This is sensible when the objective is to solely meet the total sampling requirements as fast as possible. For instance, the agent can select the state with the best current ratio \say{successful sampling} / \say{attempted sampling}, in order to encourage the algorithm to exploit areas of the state space that it supposedly masters well, i.e.,%

\vspace{-0.2in}
\begin{small}
\begin{align*}
    \overline{s}_k \in \argmin_{ s \in \Gk} \frac{ \#\{ i \in [k-1]: \overline{s}_i = s ~ \textrm{and} ~ N_{i+1}(s) = N_i(s) + 1 \} }{ \# \{ i \in [k-1]: \overline{s}_i = s \} }.
\end{align*}
\end{small}%
We point out that the possibility of not considering all undersampled states as goal states may be particularly relevant during the \textit{initial phase} of \OSP, which corresponds to the time steps when \textit{all} the states are under-sampled and thus are goal states~$\Gk$. This initial phase may furthermore be quite long when the sampling requirements verify $b(s) \gg 1$ for all $s \in \cS$ (e.g., in the \MODEST problem of Sect.\,\ref{subsection_model_estimation}). Naturally, the execution of any policy in the initial phase will collect \say{relevant} samples, until we get $\Gk \subsetneq \cS$. As such, the sample complexity guarantee of Thm.\,\ref{theorem_upper_bound} is the same whatever the strategy employed during the initial phase. In our experiments (Sect.\,\ref{section_experiments} and App.\,\ref{app_exp}), we consider an initial phase where the goal states $s$ are selected as those minimizing the \say{remaining budget} $b(s) - N(s)$ in the case of state-only requirements, or $\sum_{a} \max\{ b(s,a) - N(s,a), 0 \}$ for state-action requirements. This has the effect of shortening the length of the initial phase.

\subsection{Cost-shaping the trajectories}
\label{subsection_selecting_costs}

Instead of considering unit costs, it is possible to introduce varying costs for the SSP instance considered at each attempt. Indeed, 
if we seek to penalize the state-action pair $(s,a)$ at an attempt $k$, we can simply set the cost $c_k(s,a)$ to a quantity larger than~$1$. Imposing the costs to belong to the interval $[1, \overline{c}]$, where $\overline{c} \geq 1$ is a constant that upper bounds all possible costs, the resulting sample complexity bound in Thm.\,\ref{theorem_upper_bound} stays the same as it only inherits a constant multiplicative factor of $\overline{c}$.

First, this cost-sensitive procedure implies that if the agent has a prior knowledge or requirement that some (resp.\,actions) should be avoided, the agent can straightforwardly set the maximal cost $\overline{c}$ to such states (resp.\,actions) in order to discourage their visitation (resp.\,their execution). We show this behavior in a simple experiment in App.\,\ref{app_exp}.

Second, while \OSP is attentive in avoiding under-sampling (i.e., to achieve a desired threshold of state visitations),  it is not mindful in avoiding over-sampling certain state-action pairs. Some recent approaches (e.g., \citep{hazan2019provably, cheung2019exploration, cheung2019regret}) perform a sort of \say{distribution tracking} (via the Frank-Wolfe algorithm), achieving a more \say{stable} and \say{smooth} behavior which attempts to limit \textit{both} over-sampling and under-sampling. Unfortunately, their direct application struggles to provably enforce a minimum amount of sampling, as we explain in Sect.\,\ref{subsection_treasure} and App.\,\ref{app_alternative_approaches}. Yet we can draw inspiration from these techniques to give the agent incentive to \say{even out} the sample collection w.r.t.\,the requirements $b(s)$. 
A way to mitigate this effect is to encourage the agent to visit each state $s$ with empirical frequency close to the target frequency $b(s)/ B$. 
To do so, we can propose to penalize the visitation of sufficiently visited states by considering cost-sensitive SSP instances that verify the following informal claim.
\begin{claim}
    In order to even out the sample collection w.r.t.\,the final requirements, at each attempt $k$,  each cost $c_k(s)$ should scale as $\phi(N_k(s))$, where $N_k(s)$ is the number of samples collected so far at state $s$, and $\phi$ is a non-decreasing function which is either clipped or re-scaled in the interval $[1, \overline{c}]$.
\label{claim_cost_sensitive}
\end{claim}

This idea is fairly intuitive and, although it seems complicated to quantify the extent to which the sample collection would be effectively evened out, we now provide a theoretically grounded justification behind Claim \ref{claim_cost_sensitive} which draws a parallel between reinforcement learning and convex optimization (namely, the Frank-Wolfe algorithm).

On the one hand, for a given starting state $s_0$, goal state $\overline{s}$ and costs $c$, the SSP problem can be solved with linear programming over the dual space, where the optimization variables $\lambda(s,a)$, known as occupation measures, represent the expected number of times action $a$ is executed in state $s$. The program can be written as follows (see e.g., \citep{d1963probabilistic, trevizan2016heuristic})

\vspace{-0.18in}
\begin{small}
\begin{subequations}
\begin{alignat*}{2}
&\!\min_{\lambda}        &\qquad& \sum_{s,a} c(s,a) \lambda(s,a) \\
&\text{subject to} &      & \text{(i)\quad} \lambda(s,a) \geq 0 \quad \forall (s,a) \in \SA, \quad \quad \quad \quad \quad \text{(iv)\quad} \mu_{\textrm{out}}(s_0) - \mu_{\textrm{in}}(s_0) = 1,\\
&                  &      & \text{(ii)\quad} \mu_{\textrm{in}}(s) = \sum_{s',a} \lambda(s',a) p(s \vert s',a) \quad \forall s \in \mathcal{S}, \quad  \text{(v)\quad}  \mu_{\textrm{out}}(s) = \sum_{a} \lambda(s,a) \quad \forall s \in \mathcal{S} \setminus \{ \overline{s} \},\\
&                  &      & \text{(iii)\quad} \mu_{\textrm{out}}(s) - \mu_{\textrm{in}}(s) = 0 \quad \forall s \in \mathcal{S} \setminus \{ s_0, \overline{s} \} , \quad  \text{(vi)\quad}  \mu_{\textrm{in}}(\overline{s}) = 1.
\end{alignat*}
\end{subequations}
\end{small}%
This dual formulation can be interpreted as a flow problem, where the constraints (ii) and (v) respectively define the expected flow entering and leaving state $s$; (iii) is the flow conservation principle; (iv) and (vi) define respectively the starting state and the goal state. The objective function captures the minimization of the total expected cost to reach the goal state from the starting state. Once the optimal solution $\lambda^{\star}$ is computed, the optimal policy is $\pi^{\star}(a \vert s) = \lambda^{\star}(s,a) / \mu_{\textrm{out}}^{\star}(s)$ and is guaranteed to be deterministic, i.e., for all $s$ such that $\mu_{\textrm{out}}^{\star}(s) > 0$, we have $\lambda^{\star}(s,a) > 0$ for exactly one action $a$.

On the other hand, we seek to \say{even out} the sample collection w.r.t.\,the requirements $b(s)$. 
A natural way to do so can be to encourage the agent to visit each state $s$ with empirical frequency close to the target frequency $\frac{b(s)}{B}$. For instance, two objective functions achieving this are the following
\begin{small}
\begin{align*}
    \min_{\lambda} \mathcal{L}_1(\lambda) := \frac{1}{2} \sum_{s \in \mathcal{S}} \left( \frac{b(s)}{B} - \sum_{a \in \mathcal{A}} \lambda(s,a) \right)^2; \quad \quad \min_{\lambda} \mathcal{L}_2(\lambda) := \sum_{s \in \mathcal{S}} \sum_{a \in \mathcal{A}} \lambda(s,a) \log\left( \frac{\sum_{a \in \mathcal{A}} \lambda(s,a)}{b(s)/B} \right).
\end{align*}
\end{small}%
The first objective function is studied in \citep{cheung2019regret} as the \say{Space Exploration} problem, while the second KL-divergence objective is tackled in \citep{hazan2019provably, cheung2019exploration}. Both methods leverage a Frank-Wolfe algorithmic design, since $\mathcal{L}_1$ and $\mathcal{L}_2$ are both convex and Lipschitz-continuous in $\lambda$.
Following \citep{hazan2019provably, cheung2019exploration, cheung2019regret}, for a given attempt $k$ with empirical state-action frequencies $\wt{\lambda}_k$, the occupation measure to be targeted should minimize the following inner product: $\min_{\lambda} \langle \nabla \mathcal{L}(\wt{\lambda}_k), \lambda \rangle$. The gradients of the two objective functions above are

\vspace{-0.18in}
\begin{small}
\begin{align*}
    \nabla \mathcal{L}_1(\lambda) &= \phi_1\Big(\sum_{a \in \cA} \lambda(\cdot,a)\Big), ~ \textrm{with} ~ \phi_1(x) = x - \frac{b(s)}{B}; \\ \nabla \mathcal{L}_2(\lambda) &= \phi_2\Big(\sum_{a \in \cA} \lambda(\cdot,a)\Big), ~ \textrm{with} ~ \phi_2(x) = \log(x) + 1 - \log\left( \frac{b(s)}{B} \right).
\end{align*}
\end{small}%
Note that both $\phi_1$ and $\phi_2$ are non-decreasing functions. As a result, the $s$-th component of $\nabla \mathcal{L}(\wt{\lambda}_k)$ scales with $\sum_{a} \wt{\lambda}_k(s,a)$, i.e., with $N_k(s)$. Furthermore, in light of the contrained linear program above, the costs $c_k(s)$ should scale with the $s$-th component of $\nabla \mathcal{L}(\wt{\lambda}_k)$. This gives informal grounds to Claim~\ref{claim_cost_sensitive} of having the SSP costs at each state grow with a non-decreasing function in the state visitations. It remains an open question whether it is possible to show if this approach yields a provable improvement in the sample complexity of \OSP, or quantify the extent to which it succeeds in evening out the sample collection w.r.t.\,the final requirements (i.e., by obtaining small values for the objective functions $\mathcal{L}_1$ or $\mathcal{L}_2$).


\section{Proof of Theorem~\ref{theorem_upper_bound_general}}
\label{app_proof_upper_bound}

We first prove the special case where the sampling requirements are time- and action-independent, i.e., $b: \cS \rightarrow \mathbb{N}$.

\begin{corollary}\label{theorem_upper_bound}
    Under Asm.\,\ref{asm_communicating}, for any input sampling requirements $b: \cS \rightarrow \mathbb{N}$ with $B := \sum_{s \in \mathcal{S}} b(s)$ and for any confidence level $\delta \in (0,1)$,
    \begin{align}
    &\mkern-10mu \mathcal{C}\big(\OSPmath, b, \delta\big) = \wt{O}\Big( BD + D^{3/2} S^2 A \Big), \label{sample_complexity_1_specialcase} \\
      &\mkern-10mu \mathcal{C}\big(\OSPmath, b, \delta\big) = \wt{O}\Big( \sum_{s \in \mathcal{S}} \big( D_s b(s) + D_s^{3/2} S^2 A \big) \Big).\label{sample_complexity_2_specialcase}
    \end{align}%
\end{corollary}

\subsection{Proof of Corollary~\ref{theorem_upper_bound}}
\label{app_C1}

We denote by $\mathcal{E}$ the event under which the Bernstein inequalities stated in Eq.\,\ref{eq_bernstein} hold simultaneously for each time step $t$ and each state-action-next-state triplet $(s,a,s') \in \cS \times \cA \times \cS$, i.e.,
\makeatletter
\newcommand*\mysizebis{%
   \@setfontsize\mysizebis{8.6}{9}%
}
\makeatother
\begin{mysizebis}
\begin{align*}
\abs{\widehat{p}_{t}(s' \vert s,a) - p(s' \vert s,a)} \leq \beta_t(s,a,s') := \small{2 \sqrt{\frac{\wh{\sigma}_t^2(s' \vert s,a)}{N_t^+(s,a)} \log\left(\frac{2SAN_t^+(s,a)}{\delta}\right)} + \frac{6 \log\left(\frac{2SAN_t^+(s,a)}{\delta}\right)}{N_t^+(s,a)}.}
\end{align*}
\end{mysizebis}
Note that we have $\mathbb{P}(\mathcal{E}) \geq 1 - \delta$ from \citep{improved_analysis_UCRL2B, cohen2020near}.

We recall that at the beginning of each episode $j$, the under-sampled states $\Gj$ are cast as goal states.\footnote{In the case of state-only requirements, a state $s$ is considered under-sampled if $\sum_{a \in \mathcal{A}} N_{t-1}(s,a) < b(s)$. In the case of state-action requirements, a state $s$ is considered under-sampled if $\exists a \in \mathcal{A}, N_{t-1}(s,a) < b(s,a)$.} \OSP then constructs an SSP-MDP instance \mbox{$M_j:= \langle \mathcal{S}_j, \mathcal{A}, p_{j}, c_{j}, \Gj \rangle$}, where $\Gj$ encapsulates the goal states and $\mathcal{S}_{j}: = \mathcal{S} \setminus \Gj$ the non-goal states. The transition model $p_{j}$ is the same as the original~$p$ except for the transitions exiting the goal states which are redirected as a self-loop, i.e., \mbox{$p_{j}(s'|s,a) := p(s'|s,a)$} and \mbox{$p_{j}(g|g,a) := 1$} for any \mbox{$(s,s',a,g) \in \mathcal{S}_{j} \times \cS \times \cA \times \Gj$}. As for the cost function $c_{j}$, for any action $a \in \cA$, any goal state $g \in \Gj$ is zero-cost (i.e., $c_{j}(g,a) := 0$), while the non-goal costs are unitary (i.e., $c_{j}(s,a) := 1$ for all $s \in \cS_{j}$).

We now make more explicit the way the SSP optimistic policy is constructed at the beginning of any episode $j$. Denote by $\wh p_{j}$ the empirical transitions of the induced SSP-MDP $M_{j}$. We consider the following confidence intervals $\beta'_j$ in the optimistic SSP policy computation from App.\,\ref{app_value_iteration_SSP}
\begin{align*}
\forall(a,s') \in \cA \times \cS, \quad \forall s \notin \cG_j, ~\beta'_{j}(s,a,s') :=\beta_t(s,a,s'), \quad \forall s \in \cG_j, ~\beta'_{j}(s, a, s') = 0.
\end{align*}
We denote by $\wt{p}_j$ the optimistic model computed by the \EVI~scheme with such confidence intervals.
Now, denoting by $\mathcal{P}(\cS)$ the power set of the state space $\cS$, we have the following event inclusion
\begin{align*}
\mathcal{E} \subseteq \mathcal{E}' := \Big\{ &\forall j \geq 1, ~ \forall  \cG_j \in \mathcal{P}(\cS), ~ \forall(a,s') \in \cA \times \cS, \\ & \forall s \notin \cG_j, ~ \abs{\wt{p}_{j}(s' \vert s,a) - \wh{p}_j(s' \vert s,a)} \leq \beta'_{j}(s,a,s'), \\ &\forall s \in \cG_j, ~ \wt{p}_{j}(s \vert s,a) = 1 \Big\}.
\end{align*}
Indeed, the only transitions that are redirected from $p$ to $p_j$ are those that exit from states in $\cG_j$ and they are set to deterministically self-loop, which implies that they do not contain any uncertainty. Note that $\mathcal{E}'$ is the event that we require to hold so that the SSP analysis goes through for \textit{any} considered SSP-MDP $M_{j}$. From the inclusion above, we have that the event $\mathcal{E}'$ holds with probability at least $1-\delta$, and we assume from now on that it holds.

We denote by $H_j$ the length of each episode $j$, specifically $H_j = \min_{h \geq 1} \{ s_{j,h} \in \mathcal{G}_j \}$, where we denote by $s_{j,h}$ the $h$-th state visited during episode $j$. We denote by $\overline{s}_j := s_{j,H_j}$ the goal state in $\mathcal{G}_j$ that is reached at the end of episode $j$. Correspondingly, the starting state of each episode $j$, denoted by $\underline{s}_j$, also varies: if $j=1$ it is the initial state $s_{0}$ of the learning interaction, otherwise it is equal to $\overline{s}_{j-1}$ which is the reached goal state at the end of the previous episode $j-1$.\footnote{This choice of initial state for episodes is when we have state-only sampling requirements. If we instead have state-action requirements, the action taken at each reached goal state matters. In that case, when episode $j-1$ reaches a goal state $\overline{s}_{j-1}$, the agent takes a relevant action $\overline{a}_{j-1}$ and we then consider that the starting state $\underline{s}_j$ at the next episode $j$ is distributed according to $p(\cdot \vert \overline{s}_{j-1},\overline{a}_{j-1})$. The action $\overline{a}_{j-1}$ is naturally specified by the algorithm depending on the current and desired requirements $N(\overline{s}_{j-1}, \cdot)$ and $b(\overline{s}_{j-1}, \cdot)$, i.e., we should select $\overline{a}_{j-1} \in \{a \in \cA : N(\overline{s}_{j-1}, a) < b(\overline{s}_{j-1}, a) \}$ with $N$ the state-action counter at the end of episode $j-1$. We explain in App.\,\ref{app_exp} the way we select this action in our experiments.\label{footnote_extension_action_requirements}}
The important property is that both the starting state $\underline{s}_j$ and the goal states $\Gj$ are measurable (i.e., known and fixed) at the beginning of each episode $j$.

We define $R_J$ the regret after~$J$ episodes as follows
\begin{align}
        R_J := \sum_{j=1}^J \sum_{h=1}^{H_j} c_{j}(s_{j,h}, a_{j,h}) - \sum_{j=1}^J  \min_{\pi} V_{j}^\pi(\underline{s}_j),
\label{eq_def_ssp_regret}
\end{align}
where we denote by $V_{j}^\pi(s)$ the value function of a policy $\pi$ starting from state $s$ in the SSP-MDP instance $M_j$. We also denote by $\mathcal{C}(\OSPmath, b)$ the random variable of the total time accumulated by \OSP until the sampling requirements $b$ are met. 

On the one hand, the regret $R_J$ can be lower bounded almost surely as follows
\begin{align}
        R_J &\myeqa \sum_{j=1}^J H_j - \sum_{j=1}^J \min_{\pi} \mathbb{E}\left[ \tau_{\pi}(\underline{s}_j \rightarrow \Gj) \right] \nonumber\\
        &\myeqb \mathcal{C}(\OSPmath, b) - \sum_{j=1}^J \min_{\pi} \mathbb{E}\left[ \tau_{\pi}(\underline{s}_j \rightarrow \Gj) \right] \nonumber\\
        &\mygineeqc \mathcal{C}(\OSPmath, b) - D B, \label{eq_lower_bound_regret}
\end{align}
where (a) stems from the fact that all the non-goal costs are unitary, (b) comes from the definition of the index $J$ (i.e., the episode at which all the sampling requirements are met) and (c) combines that $J \leq B$ almost surely and that $\mathbb{E}\left[ \tau_{\pi}(\underline{s}_j \rightarrow \Gj) \right] \leq D$ by definition of the diameter $D$.

On the other hand, retracing the analysis of \citep{cohen2020near}, the derivation of the regret bound can be easily extended to varying initial states and varying (possibly multiple\footnote{Note that the SSP formulation can easily handle multiple goal states. To justify this statement, we make explicit an SSP instance with single goal state that is strictly equivalent to the SSP instance $M_j$ at hand with multiple goals $\Gj$. To do so, we introduce an artificial terminal state $\uplambda$ and define the SSP-MDP $Q_{j}$ with $\cS \cup \{\uplambda\}$ states (the non-goals are $\cS$ while the unique goal is $\uplambda$). Its transition dynamics $q_{j}$ is defined as follows: $q_{j}(\uplambda \vert \uplambda,a) = 1$, $\forall s \notin \Gj, q_{j}(s' \vert s,a) = p_j(s' \vert s,a)$, and $\forall s \in \Gj, q_{j}(\uplambda \vert s,a) = 1$. Its cost function is set to the original costs $c_j$ for states not in $\Gj$, and to $0$ (or equivalently any constant) for states in $\Gj$, and finally to $0$ for the terminal state $\uplambda$. This construction mirrors the one proposed by Bertsekas in the lecture \url{https://web.mit.edu/dimitrib/www/DP_Slides_2015.pdf} (page~25). Note that the SSP instance $M_j$ with multiple goal states $\mathcal{G}_j$ is equivalent to the single-goal SSP instance $Q_j$. The artificial terminal state $\uplambda$ is not formally necessary; it justifies why having multiple goal states is well-defined from an analysis point of view.}) goal states across episodes, as long as they are all \textit{known} to the learner at the beginning of each episode (which is our case here). 
In particular, the high-probability event is $\mathcal{E}' \supseteq \mathcal{E}$ defined above, which holds with probability at least $1-\delta$. Under this event, we have from \citep[][Thm.\,2.4]{cohen2020near} that \OSP satisfies
\begin{align}
    R_J = \wt{O}\left( D S \sqrt{A J} + D^{3/2} S^2 A \right).
\label{eq_bound_regret}
\end{align}
Combining Eq.\,\ref{eq_lower_bound_regret} and \ref{eq_bound_regret} yields that with probability at least $1-\delta$, we have
\begin{align*}
    \mathcal{C}(\OSPmath, b) \leq \wt{O}\left( B D + D S \sqrt{A J} + D^{3/2} S^2 A \right).
\end{align*}
Given that $J \leq B$ almost surely, we get
\begin{align}
    \mathcal{C}(\OSPmath, b, \delta) \leq \wt{O}\left( B D + D S \sqrt{A B} + D^{3/2} S^2 A \right).
\label{__eq__}
\end{align}
We now proceed with a separation of cases. If $B \geq S^2 A$, we have $D S \sqrt{A B} \leq B D$. Otherwise, if $B \leq S^2 A$, we have $D S \sqrt{A B} \leq D^{3/2} S^2 A$. This implies that the second summand in the $\wt{O}$ sum in Eq.\,\ref{__eq__} can be removed, which yields the first sought-after bound of Eq.\,\ref{sample_complexity_1_specialcase}.

In order to obtain the second more state-dependent bound of Eq.\,\ref{sample_complexity_2_specialcase}, the bound of Eq.\,\ref{eq_bound_regret} is too loose, hence we need to extend the analysis of \citep{cohen2020near} to bring out dependencies on $b(s)$ and $D_s$. In particular, we consider a similar decomposition in \textit{epochs} and \textit{intervals} that we carefully adapt for our purposes of varying goal states. The first epoch starts at the first time step and each epoch ends once the number of visits to some state-action pair is doubled. We denote by $\Gm$ the goal states that are considered during interval $m$ and by $D_{\Gm}$ the SSP-diameter of the goal states $\Gm$. The first interval starts at the initial time step and each interval $m$ (with goal states $\Gm$) ends once one of the four following conditions holds: (i) the length of the interval reaches $D_{\Gm}$; (ii) an unknown state-action pair is reached (where a state-action pair $(s,a)$ becomes known if its total number of visits exceeds $\alpha D_{\Gm} S \log(D_{\Gm} S A / \delta)$ for some constant $\alpha > 0$); (iii) the current episode ends, i.e., the a goal state in $\Gm$ is reached; (iv) the current epoch ends, i.e., the number of visits to some state-action pair is doubled. Finally, we denote by $H_m$ the length of each interval $m$, by $M$ the total number of intervals and by $T_M := \sum_{m=1}^M H_m$ the total time steps. As such, $T_M$ amounts to the sample complexity that we seek to bound.
Note that the goal states $\Gm$ are measurable at the beginning of the attempt $m$. Hence we can extend the reasoning of \citep[][App.\,B.2.7 \& B.2.8]{cohen2020near} to varying goal states using the decomposition described above. Assuming throughout that the high-probability events hold, we get\footnote{The intuition behind Eq.\,\ref{eq_bound_time_intervals} comes from the Cauchy-Schwarz inequality. For instance, let us consider the objective of bounding the quantity $Y := \sum_{m} x_m \sqrt{y_m}$, where the $(x_m)$ correspond to the SSP-diameters considered at each interval $m$ and the $(y_m)$ are the summands whose sums are bounded by \citep[][Lem.\,B.16]{cohen2020near}. In the latter work, denoting by $\overline{x}$ the common upper bound on the $(x_m)$, the analysis yields $Y \leq \overline{x} \sum_{m} \sqrt{y_m} \leq \overline{x} \sqrt{M} \sqrt{\sum_{m} y_m}$. In contrast, our setting requires to perform the tighter inequality $Y \leq \sqrt{\sum_{m} x_m^2} \sqrt{\sum_{m} y_m}$.}
\begin{align}
    T_M &= \wt{O}\left( \sum_{m \in \mathcal{M}^{(iii)}} D_{\Gm} + S \sqrt{A} \sqrt{ \sum_{m=1}^M D_{\Gm}^2 } + D S^2 A \right),
\label{eq_bound_time_intervals}
\end{align}
where $\mathcal{M}^{(iii)}$ is defined as the set of intervals that end according to condition (iii). We now proceed with the following decomposition, which is analogous to \citep[][Observation 4.1]{cohen2020near}
\begin{align*}
    \sum_{m=1}^M D_{\Gm}^2 \leq \sum_{m: H_m \geq D_{\Gm}}  D_{\Gm}^2 + \sum_{m: H_m < D_{\Gm}}  D_{\Gm}^2.
\end{align*}
Using that $D_{\Gm} \leq D$, the first term can be bounded as
\begin{align*}
    \sum_{m: H_m \geq D_{\Gm}}  D_{\Gm}^2 ~\leq~ D \sum_{m: H_m \geq D_{\Gm}}  D_{\Gm} ~\leq~ D \sum_{m: H_m \geq D_{\Gm}}  H_m ~\leq~ D \sum_{m=1}^M H_m = D T_M.
\end{align*}
As for the second term, we observe that it removes intervals ending under the condition (i) and thus only accounts for intervals ending under the conditions (ii), (iii) or (iv). We now perform the following key partition of intervals: each interval is categorized depending on the first goal state that ends up being reached at the end or after the considered interval. We call this goal state the \textit{retrospective goal state of the interval}. This retrospective categorization of intervals can be performed since it does not appear at an algorithmic level, but only appears at an analysis-level after Eq.\,\ref{eq_bound_time_intervals} is obtained, in order to simplify it. For any interval $m$, we denote by $s_m$ its retrospective goal. Likewise, let us denote by $M_s$ (resp.\,$\mathcal{M}_s$) the number (resp.\,the set) of intervals with retrospective goal state $s$. Finally, for any $j \in \{ii, iii, iv \}$, we denote we denote by $M^{(j)}$ (resp.\,$\mathcal{M}^{(j)}$) the number (resp.\,the set) of intervals that end according to condition $(j)$, and by $M_s^{(j)}$ (resp.\,$\mathcal{M}^{(j)}_s$) the number (resp.\,the set) of intervals with retrospective goal state $s$ that end according to condition $(j)$. We can now write
\begin{align*}
    \sum_{m: H_m < D_{\Gm}} D_{\Gm}^2 = \sum_{m \in \mathcal{M}^{(ii)}} D_{\Gm}^2 + \sum_{m \in \mathcal{M}^{(iii)}} D_{\Gm}^2 + \sum_{m \in \mathcal{M}^{(iv)} } D_{\Gm}^2 = \sum_{j \in \{ii, iii, iv \}} \sum_{m \in \mathcal{M}^{(j)}} D_{\Gm}^2.
\end{align*}
Now, for any $j \in \{ii, iii, iv \}$,
\begin{align*}
    \sum_{m \in \mathcal{M}^{(j)}} D_{\Gm}^2 = \sum_{m \in \mathcal{M}^{(j)}} \Big(\sum_{s \in \cS} \mathds{1}_{\{s_m = s\}} \Big) D_{\Gm}^2 &= \sum_{s \in \cS} \sum_{m \in \mathcal{M}_s^{(j)}} D_{\Gm}^2 \\ &\myineeqa \sum_{s \in \cS} \sum_{m \in \mathcal{M}_s^{(j)}} D_s^2 \\ &= \sum_{s \in \cS} M_s^{(j)} D_s^2,
\end{align*}
where inequality (a) comes from Lem.\,\ref{lemma_useful_diameter_subset} stated later. Moreover, we have
\begin{align*}
    M^{(ii)}_s = \wt{O}\left( D_s S^2 A \right); \quad \quad M^{(iii)}_s \leq b(s); \quad \quad M^{(iv)} \leq 2 S A \log(T_M).
\end{align*}
While the first and third bounds above are similar to those considered in \citep{cohen2020near}, the key difference lies in the second bound, which leverages that the number of intervals that end in the goal state $s$ is, by definition of our problem, upper bounded by the number of samples required at state $s$, i.e., $b(s)$. All in all, this implies that
\begin{align*}
    \sum_{m: H_m < D_{\Gm}} D_{\Gm}^2 \leq \wt{O}\left( \sum_{s \in \cS} D_s^3 S^2 A \right) + \sum_{s \in \cS} b(s) D_s^2 + \wt{O}\left( D^2 S A \right).
\end{align*}
Moreover, in a similar manner as above, we bound the first term of Eq.\,\ref{eq_bound_time_intervals} as follows
\begin{align*}
      \sum_{m \in \mathcal{M}^{(iii)}} D_{\Gm} = \sum_{s \in \cS} M_s^{(iii)} D_s \leq \sum_{s \in \cS} D_s b(s).
\end{align*}
Putting everything together back into Eq.\,\ref{eq_bound_time_intervals} and simplifying using the subadditivity of the square root, we get
\begin{align*}
    T_M = \wt{O}\left( \sum_{s \in \cS} D_s b(s) + D S^2 A + S \sqrt{A D T_M} + S \sqrt{A} \sqrt{\sum_{s \in \cS} b(s) D_s^2} + S^2 A \sum_{s \in \cS} D_s^{3/2} \right).
\end{align*}
Using that $x \leq c_1 \sqrt{x} + c_2$ implies $x \leq (c_1 + \sqrt{c_2})^2$ for $c_1 \geq 0$ and $c_2 \geq 0$, we obtain
\begin{align}
    T_M = \wt{O} \left( \left[ S \sqrt{DA} + \sqrt{\sum_{s \in \cS} D_s b(s)} + \sqrt{ S \sqrt{A}\sqrt{\sum_{s \in \cS} b(s) D_s^2}  } + \sqrt{S^2 A \sum_{s \in \cS} D_s^{3/2} }  \right]^2 \right).
\label{eq_eq_T}
\end{align}
We now apply the Cauchy-Schwarz inequality to simplify the third summand
\begin{align*}
    S \sqrt{A} \sqrt{\sum_{s \in \cS} b(s) D_s^2}   \leq \sum_{s \in \cS} \sqrt{S^2 A D_s} \sqrt{D_s b(s)}  \leq \sqrt{ \sum_{s \in \cS} D_s b(s)} \sqrt{ S^2 A \sum_{s \in \cS} D_s}.
\end{align*}
Let us introduce $x := \sqrt{ \sum_{s \in \cS} D_s b(s)}$ and $y := \sqrt{ S^2 A \sum_{s \in \cS} D_s^{3/2}}$. Plugging the simplifications into Eq.\,\ref{eq_eq_T} finally yields with probability at least $1-\delta$ that $T_M = \wt{O}\left(\left( x + \sqrt{xy} + y \right)^2\right) = \wt{O}\left(\left( x + y \right)^2\right) = \wt{O}\left( x^2 + y^2 \right)$. Since $T_M$ amounts to the sample complexity, we get the desired bound of Eq.\,\ref{sample_complexity_2_specialcase}, which reads
\begin{align*}
    \mathcal{C}(\OSPmath, b, \delta) = \wt{O}\left( \sum_{s \in \cS} D_s b(s) + S^2 A \sum_{s \in \cS} D_s^{3/2} \right).
\end{align*}

\begin{lemma}\label{lemma_useful_diameter_subset}
    For any set of goals $\mathcal{G} \subsetneq \cS$, we introduce the meta SSP-diameter $D_{\mathcal{G}} := \max_{s \in \cS \setminus \mathcal{G}} \min_{\pi} \mathbb{E}\left[ \tau_{\pi}(s \rightarrow \mathcal{G}) \right]$, where we define $\tau_{\pi}(s \rightarrow \mathcal{G}) := \min \{ t \geq 0: s_{t+1} \in \mathcal{G} \,\vert\, s_1 = s, \pi \}$. Then we have
    \begin{align*}
        D_{\mathcal{G}} \leq \min_{s \in \mathcal{G}} D_s.
    \end{align*}
\end{lemma}
\begin{proof}
    For any $g \in \mathcal{G}$, $s \in \cS \setminus \mathcal{G}$ and policy $\pi$, we have $\mathbb{E}\left[ \tau_{\pi}(s \rightarrow \mathcal{G}) \right] \leq \mathbb{E}\left[ \tau_{\pi}(s \rightarrow g) \right]$. In particular, this implies that for any $g \in \mathcal{G}$, $D_{\mathcal{G}} \leq D_g$, which immediately gives the result.
\end{proof}

\subsection{From Corollary~\ref{theorem_upper_bound} to Theorem~\ref{theorem_upper_bound_general}}
\label{subsection_extension_general_requirements}

We now consider the general case of possibly action-dependent and time-dependent sampling requirements.

\textbf{State-action requirements.} First, \OSP can be easily extended from state requirements $b(s)$ to state-action requirements $b(s,a)$. Indeed, the only difference between these two settings occurs w.r.t.\,which action the algorithm takes at the end of a given episode (i.e., when a sought-after goal state is reached): for state-action requirements, any under-sampled action is taken (see footnote \ref{footnote_extension_action_requirements} for details). Bound-wise, the number of times where this scenario occurs is at most $B$ (since there are at most $B$ episodes), hence the guarantee from Cor.\,\ref{theorem_upper_bound} is unaffected whatever the action executed once a goal state is reached.

\textbf{Adaptive requirements.} \OSP can be also easily extended to requirements $(b_t(s,a))_{t \geq 1}$ that vary over time, where $b_t$ may be chosen adaptively depending on the samples observed so far (i.e., $b_t$ is measurable w.r.t.\,the filtration up to time $t$). Indeed, the important property required in the derivations of App.\,\ref{app_C1} that both the starting state and the goal states should be measurable (i.e., known and fixed) at the beginning of each episode still holds. As such, the sample complexity result of Cor.\,\ref{theorem_upper_bound} can be naturally extended by defining $B_\tau := \sum_{s,a} b_\tau(s,a)$, where $\tau$ is the first (random) time step when all the sampling requirements are met. In order for the sample complexity to remain bounded, a sufficient condition is Asm.\,\ref{asm_bounded_requirements}. In particular, considering the sequence $b_t(s,a)$ to be upper bounded by a fixed threshold $\overline{b}(s,a)$ for each $(s,a)$, the bound from Cor.\,\ref{theorem_upper_bound} trivially holds with $\wb B := \sum_{s,a}\wb b(s,a)$.

\subsection{Remark}
\label{remark_greedy}
Notice that the \say{comparator} we are using in the definition of the regret in Eq.\,\ref{eq_def_ssp_regret} may not be the \say{global} optimum in terms of sample complexity. Indeed, the optimal sequence of strategies would result in a non-stationary policy $\pi^{\star}_{\mathcal{C}} \in \argmin_{\pi} \mathcal{C}(\pi, b, \delta)$. Yet in our analysis, we compare the algorithmic performance with the larger quantity $\sum_{j=1}^J  \min_{\pi} V_{\Gj}^\pi(\underline{s}_j)$, which corresponds to \say{greedily} minimizing each time to reach an under-sampled state in a sequential fashion. This highlights that \OSP does not \textit{track} any optimal sampling allocation or distribution (i.e., it does not seek to \say{imitate} $\pi^{\star}_{\mathcal{C}}$), insofar as it discards the effect of traversing other states while reaching an undersampled goal state. While this means that some areas of the state space may be oversampled, \OSP is able to devote its full attention to the objective of minimizing the total sample complexity, instead of being mindful to avoid certain areas of the state space which it has already visited. We argue that this is what results in the appealing sample complexity of \OSP, whereas other techniques specifically designed to track distributions (via e.g., the Frank-Wolfe algorithmic scheme) struggle to minimize the sample complexity, as explained in Sect.\,\ref{subsection_treasure} and App.\,\ref{app_alternative_approaches}.

\section{Lower Bound}
\label{app:lower_bound}

In this section, we provide three complementary results that lower bound the sample complexity of the problem of Def.\,\ref{def:sample.comeplxity}.

\textbf{\ding{172} First}, as stated in Lem.\,\ref{lemma_LB_informal}, we construct a simple MDP such that for any arbitrary sampling requirements $b(s)$, the (possibly non-stationary) policy minimizing the time to collect all samples has sample complexity of order $\Omega\big(\sum_{s \in \mathcal{S}} D_s b(s)\big)$. 
We begin with a useful result.
\begin{lemma}
    Let $q \in (0, 1)$ and consider the Markov chain $M_q$ with two states $x$, $y$ whose dynamics $p_q$ are as follows: $p_q(y \vert x) = q$, $p_q(x \vert x) = 1-q$ and $p_q(x \vert y) = 1$. Then $M_q$ is communicating with diameter $D_q := \frac{1}{q}$. Moreover, denote by $T_B$ the (random) time of the $B$-th visit to state $y$ starting from any state, and assume that $B \geq 5$. Then with probability at least $\frac{1}{2}$, we have $T_B \geq \frac{B}{2q} + B = \frac{B D_q}{2} + B$.
    \label{lem_toy_ex_lower_bound}
\end{lemma}

\begin{proof}
Introduce $X := \sum_{i=1}^n X_i$ where $X_i \sim Ber(q)$ (i.e., it follows a Bernoulli with parameter $q$) and we set $n := \frac{B}{2q}$. 
We have $\mathbb{E}\left[X\right] = n q = \frac{B}{2}$. Moreover, the Chernoff inequality entails that
\begin{align*}
    \mathbb{P}\left( X \geq B \right) = \mathbb{P}\left( X \geq 2 \mathbb{E}[X] \right) \leq \exp\Big( - \frac{\mathbb{E}[X]}{3} \Big) = \exp\Big( - \frac{B}{6} \Big) \leq \frac{1}{2},
\end{align*}
where the last inequality holds whenever $B \geq 6 \log(2)$. Note that the random variable $T_B$ follows a negative binomial distribution for which each success accounts for two time steps instead of one. This means that with probability at least $\frac{1}{2}$, 
\begin{align*}
    T_B \geq n + B = \frac{B}{2q} + B = \frac{B D_q}{2} + B.   
\end{align*}
\end{proof}

Let us now consider a state space $\mathcal{S} := \{ s_1, \ldots, s_S\}$ and arbitrary sampling requirements $b: \mathcal{S} \rightarrow \mathbb{N}$. We construct a wheel MDP with state space $\mathcal{S} \cup \{ s_0 \}$, where $s_0$ is the starting center state. There are~$A = S$ actions available and the dynamics~$p$ are defined w.r.t.\,a set $(\epsilon_i) \in (0,1)^S$ such that $\forall i \in [S]$, $p(s_i \vert s_0, a_i) = \epsilon_i$, $p(s_0 \vert s_0, a_i) = 1 - \epsilon_i$, and for every action $a$, $p(s_0 \vert s_i, a) = 1$. Note that by having such $A=S$ actions, the attempts to collect relevant samples are independent, in the sense that at any $s \in \cS$, the learner cannot rely on the attempts performed for the other states $s' \neq s$. Let us assume that $b(s) \geq 6 \log(2S)$. From Lem.\,\ref{lem_toy_ex_lower_bound}, for any state $s \in \mathcal{S}$, with probability $1-\frac{1}{2S}$, the time needed to collect $b(s)$ samples from state $s$ is lower bounded by $\frac{b(s)}{2 \epsilon_i} + b(s)$, and furthermore we have $D_s = \frac{1}{\epsilon_i} + 1$. Taking a union bound over the $S$ states in $\mathcal{S}$ means that with probability at least $\frac{1}{2}$, the time to collect the required samples is lower bounded by $\sum_{s \in \cS} \frac{b(s)(D_s - 1)}{2} + b(s)$.

\vspace{0.1in}

\textbf{\ding{173} Second}, we show that the family of worst-case MDPs is relatively large. In fact, for any MDP with diameter $D$, we can perform a minor change to its dynamics without affecting the overall diameter and show that when the sampling requirements are concentrated in a single state, any policy would take at least $\Omega(BD)$ steps to collect all the $B$ samples. More specifically, there exists a class $\mathbb{C}$ of MDPs such that, for each MDP in $\mathbb{C}$, there exists a requirement function $b$ and a finite threshold (that depends on the considered MDP) such that the $\Omega(BD)$ lower bound holds whenever $B$ exceeds this threshold. The class $\mathbb{C}$ effectively encompasses a large number of environments: indeed, take \textit{any} MDP $M$, then we can find an MDP $M'$ in $\mathbb{C}$ such that $M$ and $M'$ differ in their transitions \textit{only} at one state and have the same diameter. Formally, we have the following statement (proof in App.\,\ref{subsection_proof_LB2}).

\begin{lemma}\label{theorem_lower_bound_2}
    Fix any positive natural numbers $S$, $A$ and $D$, and any MDP $M$ with $S = \abs{\mathcal{S}}$ states, $A = \abs{\mathcal{A}}$ actions and diameter $D$. There exists a modification of the transitions of $M$ at only one state which yields an MDP $M'$ with the same diameter $D$, and there exists a finite integer $W_{\mathfrak{A},M',\delta}$ (depending on $\mathfrak{A}$, $M'$) such that for any total requirement $B \geq W_{\mathfrak{A},M',\delta}$, there exists a function $b^{\dagger}: \mathcal{S} \rightarrow \mathbb{N}$ with $\sum_{s \in \mathcal{S}} b(s) = B$, such that, for any arbitrary starting state, the optimal non-stationary policy $\mathfrak{A}^{\star}$ needs $\mathcal{C}(\mathfrak{A}^{\star}, b^{\dagger})$ time steps to collect the desired samples in the modified MDP~$M'$, where
    \begin{align*}
    \mathbb{P}\left(  \mathcal{C}(\mathfrak{A}^{\star}, b^{\dagger}) > \frac{(B-1) D}{2}\right) \geq \frac{1}{2}.
    \end{align*}
\end{lemma}

\textbf{\ding{174} Third}, we note that both results above do not take into account the added difficulty for the agent to have to deal with a learning process. To do so, we can draw inspiration from the lower bound on the expected regret for learning in an SSP problem derived in \citep{cohen2020near}. Indeed, let us consider a environment $M$ with one state $\overline{s}$ in which all the required samples are concentrated, i.e., $b := B \mathds{1}_{\overline{s}}$ with $B \geq SA$. The $S-1$ other states $s$ each contain a special action $a_s^{\star}$. The transition dynamics $p$ are defined as follows: $p(\overline{s} \vert s, a_s^{\star}) = \frac{1}{D_{\overline{s}}}$, $p(s \vert s, a_s^{\star}) = 1-  \frac{1}{D_{\overline{s}}}$, $p(\overline{s} \vert s, a) = \frac{1-\nu}{D_{\overline{s}}}$, $p(s \vert s, a) = 1-  \frac{1-\nu}{D_{\overline{s}}}$ for any other action $a \in \cA \setminus \{ a_s^{\star} \}$, and finally $p(s \vert \overline{s}, a) = \frac{1}{S-1}$ for any action $a \in \cA$, with $\nu:= \sqrt{(S-1) A B}/64$. Recall that $D_{\overline{s}}$ is the SSP-diameter of state $\overline{s}$. The communicating, non-episodic structure of $M$ naturally mimics the interaction of an agent with an SSP problem with goal state $\overline{s}$. Denoting by $\mathcal{C} (\mathfrak{A}, b)$ the (random) time required by any algorithm $\mathfrak{A}$ to collect the $b$ sought-after samples, we obtain from \citep[][Thm.\,2.7]{cohen2020near} that
\begin{align*}
    \mathbb{E}\left[ \mathcal{C} (\mathfrak{A}, b=B \mathds{1}_{\overline{s}}) \right] \geq  \phi(B) &:= \underbrace{(D_{\overline{s}}+1) B}_{:= \phi_1(B)} + \underbrace{\frac{1}{1024} D_{\overline{s}} \sqrt{(S-1) A B}}_{:= \phi_2(B)} \\ &= \sum_{s \in \cS} \left( (D_{s}+1) b(s) + \frac{1}{1024} D_{s} \sqrt{(S-1) A b(s)} \right).
\end{align*}
This lower bound on the expected time to collect the samples implies in particular that no algorithm can meet the sampling requirements in less than $\wt{O}(\phi(B))$ time steps with high probability. Importantly, note that this result is not contradictory with Thm.\,\ref{theorem_upper_bound}. Indeed, as fleshed out in the proof in App.\,\ref{app_proof_upper_bound}, the upper bound of Thm.\,\ref{theorem_upper_bound} actually contains such a square root term $\phi_2(B)$, yet it is subsumed in the final bound by either the main-order term in $ \sum_{s} b(s) D_s$ or the lower-order term constant w.r.t.\,$B$ (see Eq.\,\ref{__eq__}). We can decompose $\phi(B)$ in two factors: the second term $\phi_2(B)$ comes from the learning process of trying to match the behavior of the optimal policy, while the first term $\phi_1(B)$ stems from the need to navigate through the environment as opposed to the generative model assumption (as such, it is incurred even if the optimal policy is deployed from the start). Part \ding{173} of this section actually shows that such a term $\phi_1(B)$ is unavoidable in multiple MDPs.

\subsection{Proof of Lemma~\ref{theorem_lower_bound_2}}
\label{subsection_proof_LB2}

Here we give the proof of Lem.\,\ref{theorem_lower_bound_2}. For any positive natural numbers $S$, $A$, $D$, we consider any MDP $M$ with $S$~states, $A$~actions and diameter~$D$. We consider
\begin{align*}
        (\underline{s}, \overline{s}) \in \argmax_{s \neq s' \in \mathcal{S}} \left\{ \min_{\pi \in \SD} \mathbb{E}\left[ \tau_{\pi}(s \rightarrow s')\right] \right\}.
\end{align*}

We modify the transition structure of $M$, so that $p(\underline{s} \vert \overline{s}, a) = 1$ for all actions $a \in \mathcal{A}$. Note that the diameter is not affected by this operation. Throughout, whatever the value of $B$, we will consider the following sampling requirements: $b(s) := B \mathds{1}_{ \{ s = \overline{s} \} }$. We denote by $s_0 \in \mathcal{S}$ the arbitrary starting state of the learning process.

Consider any learning algorithm $\mathfrak{A}$. We denote by $\pi$ the (possibly non-stationary) policy that is executed by $\mathfrak{A}$. In virtue of Asm.\,\ref{asm_communicating}, we can naturally (and without loss of generality) restrict our attention to a policy $\pi$ whose expected hitting time to $\overline{s}$ is finite starting from any state in $\mathcal{S}$ --- we denote by $\overline{\mu}_{\pi}$ such an upper bound.
We denote by $T_{\pi}^{(i)}$ the random time required by policy $\pi$ to collect the $i$-th sample at state $\overline{s}$, starting from $s_0$ if $i = 1$ or from $\underline{s}$ if $2 \leq i \leq B$.

\begin{lemma}
    The $(T_{\pi}^{(i)})_{2 \leq i \leq B}$ are i.i.d.\,sub-exponential random variables whose expectation satisfies $\mu_{\pi} := \mathbb{E}\left[ T_{\pi}^{(i)} \right] \geq D$ for all $2 \leq i \leq B$.
    \label{lemma_subexp_rv_time_collect_sample}
\end{lemma}

\begin{proof}
    Consider the SSP problem with unitary costs, starting state $\underline{s}$ and zero-cost, absorbing terminal state $\overline{s}$. According to \citep{bertsekas1991analysis}, Asm.\,\ref{asm_communicating} and the fact that the costs are all positive guarantee that the optimal value function of this SSP problem is achieved by a stationary deterministic policy. This implies that $\min_{\pi' \in \SD} \mathbb{E}\left[ \tau_{\pi'}(\underline{s} \rightarrow \overline{s})\right] \leq \mathbb{E}\left[ \tau_{\pi}(\underline{s} \rightarrow \overline{s})\right]$, and thus by definition of $D$ and $\mu_{\pi}$, we get the inequality $D \leq \mu_{\pi}$. There remains to prove the sub-exponential nature of the random variable $T_{\pi}$. For any $\lambda \in \mathbb{R}$, we have 
    \begin{align*}
        \mathbb{E}\left[ e^{\lambda (T_{\pi} - \mu_{\pi})} \right] = e^{-\lambda \mu_{\pi}} \mathbb{E}\left[ \sum_{n=0}^{+ \infty} \frac{1}{n!} \lambda^n T_{\pi}^n \right] = e^{-\lambda \mu_{\pi}} \sum_{n=0}^{+ \infty} \frac{1}{n!} \lambda^n \mathbb{E}\left[ T_{\pi}^n \right] \leq 2 e^{-\lambda \mu_{\pi}} \sum_{n=0}^{+ \infty} \frac{1}{n!} n^n (\lambda \overline{\mu}_{\pi})^n,
    \end{align*}
where the last inequality comes from \citep[][Lem.\,15]{tarbouriech2019no}, which can be applied to bound the moments $\mathbb{E}\left[ T_{\pi}^n \right] \leq 2 (n \overline{\mu}_{\pi})^n$, since the random variable $T_{\pi}$ satisfies $\mathbb{E}\left[T_{\pi}(s \rightarrow \overline{s})\right] \leq \overline{\mu}_{\pi}$ for all $s \in \mathcal{S}$ by definition of $\overline{\mu}_{\pi}$. From Lem.\,\ref{technical_lemma_convergence_series}, the series above converges whenever $\abs{\lambda} < \frac{1}{e \overline{\mu}_{\pi}}$. This proves that $T_{\pi}$ is sub-exponential according to the second condition of Def.\,\ref{def_sub_exponential_rv}.
\end{proof}

\begin{lemma}\label{technical_lemma_convergence_series}
    The series $\displaystyle\sum_{n=0}^{+ \infty} \frac{n^n}{n!} x^n$ converges absolutely for all $\abs{x} < \frac{1}{e}$.
\end{lemma}

\vspace{-0.1in}
\begin{proof}
Introduce the summand of the series $a_n(x) := \frac{n^n}{n!} x^n$. We then have
\begin{align*}
    \frac{a_{n+1}(x)}{a_n(x)} = \frac{n!}{(n+1)!} \frac{(n+1)^{n+1}}{n^n} x = \left( 1 + \frac{1}{n}\right)^n x \xrightarrow[n \rightarrow + \infty]{} e x.
\end{align*}
Hence, for any $\abs{x} < \frac{1}{e}$, we have $\abs{\frac{a_{n+1}(x)}{a_n(x)}} < 1$, which means from d'Alembert's ratio test that the series converges absolutely.
\end{proof}

Since $T_{\pi}$ is sub-exponential, from Def.\,\ref{def_sub_exponential_rv}, there exists a pair $(\sigma_{\pi}, \theta_{\pi})$ of finite positive parameters that verifies
\begin{align*}
    \mathbb{E}\left[ e^{\lambda (T_{\pi} - \mu_{\pi})} \right] \leq e^{\frac{\sigma_{\pi}^2 \lambda^2}{2}} \quad \textrm{for all~} \abs{\lambda} < \frac{1}{\theta_{\pi}}.
\end{align*}
We now apply the concentration inequality for sub-exponential random variables stated in Lem.\,\ref{prop_concentration_ineq_sub_exponential_rv}.
\begin{align*}
    \forall y > \frac{\sigma_{\pi}^2}{\theta_{\pi}}, \quad \mathbb{P}\left( \sum_{i=2}^B T_{\pi}^{(i)} \leq \mu_{\pi} (B-1) - y \right) \leq \exp\left( - \frac{y}{2 \theta_{\pi}} \right).
\end{align*}
We now fix the integer $$W_{\pi} := 1 + 2 \max \left\{ \Bigg\lceil \frac{\theta_{\pi}}{\mu_{\pi}} \Bigg\rceil, \Bigg\lceil \frac{ \sigma_{\pi}^2}{\theta_{\pi} \mu_{\pi}} \Bigg\rceil \right\}.$$ Consider any total sampling requirement $B \geq W_{\pi}$. Then setting $y := \frac{\mu_{\pi} (B-1)}{2} > \frac{\sigma_{\pi}^2}{\theta_{\pi}}$ yields
\begin{align*}
     \mathbb{P}\left( \sum_{i=2}^B T_{\pi}^{(i)} \leq \frac{\mu_{\pi} (B-1)}{2} \right) \leq \exp\left( - \frac{\mu_{\pi} (B-1)}{4 \theta_{\pi}} \right) \leq \frac{1}{2},
\end{align*}
since we have $B \geq \frac{4\theta_{\pi}}{\mu_{\pi}} \log(2) + 1$. This implies that with probability at least $\frac{1}{2}$,
\begin{align*}
    \sum_{i=1}^B T_{\pi}^{(i)} \geq \sum_{i=2}^B T_{\pi}^{(i)} > \frac{\mu_{\pi} (B-1)}{2} \geq \frac{(B-1) D}{2},
\end{align*}
where the last inequality stems from Lem.\,\ref{lemma_subexp_rv_time_collect_sample}. As a result, there exists a finite integer $W_{\pi,\delta}$ (depending on $\pi$ and the environment at hand) such that, for any total sampling requirement $B \geq W_{\pi}$, the algorithm $\mathfrak{A}$ that executes policy $\pi$ verifies
\begin{align*}
    \mathbb{P}\left(  \mathcal{C}(\mathfrak{A}, B \mathds{1}_{ \{ \overline{s} \} }) > \frac{(B-1) D}{2}\right) \geq \frac{1}{2},
\end{align*}
which gives the proof of Lem.\,\ref{theorem_lower_bound_2}.

We recall here the definition of sub-exponential random variables.

\begin{definition}[\citealp{wainwright}]\label{def_sub_exponential_rv}
    A random variable $X$ with mean $\mu < + \infty$ is said to be sub-exponential if one of the following equivalent conditions is satisfied:
    \begin{enumerate}[leftmargin=.2in,topsep=-4pt,itemsep=0pt,partopsep=0pt, parsep=0pt]
            \item (Laplace transform condition) There exists $(\sigma, \theta) \in \mathbb{R}^+ \times \mathbb{R}^{+ \star}$ such that, for all $\abs{\lambda} < \frac{1}{\theta}$, \label{cond:subexponential.c1}
        \begin{align*}
            \mathbb{E}\left[ e^{\lambda (X - \mu)} \right] \leq e^{\frac{\sigma^2 \lambda^2}{2}}.
        \end{align*}
        \item There exists $c_0 > 0$ such that $\mathbb{E}\left[ e^{\lambda (X - \mu)} \right] < + \infty$ for all $\abs{\lambda} \leq c_0$.
    \end{enumerate}
    For any pair $(\sigma, \theta)$ satisfying condition~\ref{cond:subexponential.c1}, we write $X \sim \subExp(\sigma, \theta)$.
\end{definition}

We finally recall a concentration inequality satisfied by sub-exponential random variables.
\begin{lemma}[\citealp{wainwright}]\label{prop_concentration_ineq_sub_exponential_rv}
    Let $(X_i)_{1 \leq i \leq n}$ be a collection of independent sub-exponential random variables such that for all $i \in [n]$, $X_i \sim \subExp(\sigma_i, \theta_i)$ and $\mu_i := \mathbb{E}\left[X_i\right]$. Set $\sigma := \sqrt{\frac{\sum_{i=1}^n \sigma^2_i}{n}}$ and $\theta := \max_{i \in [n]} \{ \theta_i \}$. The following concentration inequalities hold for any $t \geq 0$,
    \begin{align*}
        \mathbb{P}\left( \sum_{i=1}^n X_i - \sum_{i=1}^n \mu_i \geq t \right) &\leq \begin{cases}
         e^{-\frac{t^2}{2n\sigma^2}} \quad \text{if~} 0 \leq t \leq \frac{\sigma^2}{\theta} \\ e^{-\frac{t}{2\theta}} \quad \text{if~} t > \frac{\sigma^2}{\theta}  \end{cases}, \\
        \mathbb{P}\left( \sum_{i=1}^n X_i - \sum_{i=1}^n \mu_i \leq -t \right) &\leq \begin{cases}
         e^{-\frac{t^2}{2n\sigma^2}} \quad \text{if~} 0 \leq t \leq \frac{\sigma^2}{\theta} \\ e^{-\frac{t}{2\theta}} \quad \text{if~} t > \frac{\sigma^2}{\theta}  \end{cases}.
    \end{align*}
\end{lemma}


\vspace{0.2in}

\section{\OSPtitle Beyond the Communicating Setting}
\label{app_beyond_communicating}

Sect.\,\ref{section_osp_algorithm} and App.\,\ref{app:lower_bound} demonstrate that the diameter $D$ and/or the SSP-diameters dictate the performance of a sampling procedure in a communicating environment. Indeed, both the \OSP upper bound and the worst-case lower bound contain $D$ and/or $D_s$ as a multiplicative factor w.r.t.\,the total sampling requirement $B$. However, in many environments, there may exist some states that are hard to reach, or plainly impossible to reach. In that case, the diameter is prohibitively large and even possibly infinite, thus rendering the sample complexity guarantee of Thm.\,\ref{theorem_upper_bound} vacuous. To circumvent this issue, a desirable property of the algorithm would be the ability to assess online the \say{feasibility} of the sampling requirements, by discarding states that are indeed too difficult to reach. For ease of exposition, we consider throughout App.\,\ref{app_beyond_communicating} the special case of time- and action-independent sampling requirements $b : \cS \rightarrow \mathbb{N}$ (as explained in App.\,\ref{subsection_extension_general_requirements} the extension to the general case of adaptive action-dependent sampling requirements follows straightforwardly).

Formally, we consider any environment that need not be communicating (i.e., it may not satisfy Asm.\,\ref{asm_communicating}). The learning agent receives as input an integer parameter $L \geq 1$, which acts as a reachability threshold that partitions the state space between the states from which we expect sample collection and those that we categorize as too difficult to reach. Specifically, given a sampling requirement $b : \cS \rightarrow \mathbb{N}$, the desiderata of the agent is to minimize the time it requires, for each state $s \in \cS$, to i) either collect the $b(s)$ samples, ii) or discard the sample collection at state $s$ only if there exists a state (accessible from the starting state) that cannot reach $s$ within $L$ steps in expectation. In other words, we do not allow for samples to be discarded if the state is actually below the reachability threshold $L$. We introduce the following new definition of the sample complexity.

\begin{definition}\label{def_sample_complexity_L}
    Given a reachability threshold $L \geq 1$, sampling requirements $b: \mathcal{S} \rightarrow \mathbb{N}$, starting state $s_0 \in \cS$ and a confidence level $\delta \in (0,1)$, the \textit{sample complexity} of a learning algorithm $\mathfrak{A}$ is defined as
    \begin{align*}
        \mathcal{C}(\mathfrak{A}, b, \delta, L, s_0) := \min \Big\{t > 0: \mathbb{P}\Big( \forall s \in \mathcal{S}_L, ~N_t(s) \geq b(s) \land I_{\mathfrak{A}}(t) = 1 \Big) \geq 1 - \delta \Big\},
    \end{align*}
where $\mathcal{S}_L := \{ s \in \mathcal{S}: \max_{ \{ y \in \cS : D_{s_0 y} < + \infty \}} D_{y s} \leq L \}$ and where $I_{\mathfrak{A}}(t)$ corresponds to a Boolean equal to $1$ if the algorithm $\mathfrak{A}$ considers at time $t$ that none of the states that remain to be sampled (if there remains any) belong to $\mathcal{S}_L$.\footnote{Why isn't $\mathcal{S}_L$ defined as $\mathcal{S}_L := \{ s \in \mathcal{S}: D_s \leq L \}$? Under such a definition, in the case of a weakly communicating environment, the optimal strategy $\mathfrak{A}$ would be to set $I_{\mathfrak{A}}(t=1)=1$, which would yield a sample complexity of 1, since there would exist at least one \say{isolated} state and hence $S_L = \emptyset$. Of course, this is not the behavior we would want, as we expect the optimal strategy to perform the sample collection at states in the communicating class (starting from $s_0$), and discard the sample collection at states that are not accessible from $s_0$. This is explained in more detail in the last paragraph of this section.}
\end{definition}

\paragraph{Algorithm $\OSPmathL$.} We now propose a simple adaptation of \OSP to handle this setting, and call the corresponding algorithm $\OSPmathL$ since it receives as input a reachability threshold $L$. We split time in \textit{episodes} indexed by $j$, where the first episode begins at the first time step and the $j$-th episode ends when the $j$-th desired sample is collected. From Thm.\,\ref{theorem_upper_bound} we know that in a communicating environment with diameter $D$, there exists an absolute constant $\alpha > 0$ (here we exclude logarithmic terms for ease of exposition) such that with probability at least $1-\delta$, after any $j$ episodes (i.e., after the $j$-th desired sample is collected), $T_{j}$ the (total) time step at the end of the $j$ episodes is upper bounded as follows
\begin{align*}
    T_{j} \leq \alpha j D + \alpha j D^{3/2} S^2 A.
\end{align*}
The key idea is to run \OSP and stop its execution if its total duration at some point exceeds a certain threshold depending on $L$ and the current episode. Specifically, in the $j$-th episode, this threshold is set to $\Phi(j) := \alpha j L + \alpha j L^{3/2} S^2 A$. If the accumulated duration never exceeds the threshold, the algorithm is naturally run until all the sampling requirements are met.

\begin{lemma}\label{theorem_upper_bound_L}
Consider any reachability threshold $L \geq 1$, starting state $s_0 \in \cS$, confidence level $\delta \in (0,1)$ and sampling requirements $b: \cS \rightarrow \mathbb{N}$, with $B = \sum_{s \in \cS} b(s)$. Then running the algorithm $\OSPmathL$ in any environment yields a sample complexity that can be upper bounded as
    \begin{align*}
    \mathcal{C}(\OSPmathL, b, \delta, L, s_0) = \wt{O}\Big( B L + L^{3/2} S^2 A \Big).
    \end{align*}
\end{lemma}

\begin{proof} The result is obtained by performing a \textit{reductio ad absurdum} reasoning. We initially make the assumption $\mathcal{H}$ that for all episodes $j \geq 1$, we have $D_{\Gj} \leq L$, where we recall that $D_{\Gj}$ is the SSP-diameter of the goal states $\Gj$ considered during episode $j$. The condition that is checked at any time step is whether it is smaller or larger than the threshold $\Phi(j) := \alpha j L + \alpha j L^{3/2} S^2 A$, where $j$ is the current episode. \textbf{\textit{i)}} In the first case, the total duration is always smaller (or equal) than its threshold and the algorithm performs $J$ episodes until the sampling requirements are met. Since $J \leq B$ and $\Phi$ is an increasing function, the sample complexity is bounded by $\Phi(J) \leq \Phi(B) = \wt{O}\left(B L + L^{3/2} S^2 A \right)$. \textbf{\textit{ii)}} In the second case, there exists an episode $j' \geq 1$ and a time step (during that episode) which is larger than the threshold $\Phi(j')$. This implies that with probability at least $1-\delta$, assumption $\mathcal{H}$ is wrong. Thus there exists an episode $1 \leq j \leq j'$ such that $D_{\Gj} > L$. Since $\mathcal{G}_{j'} \subset \Gj$, we have $D_{\Gj} \leq D_{\mathcal{G}_{j'}}$, thus $D_{\mathcal{G}_{j'}} > L$, which implies from Lem.\,\ref{lemma_useful_diameter_subset} that for all $s \in \mathcal{G}_{j'}$, $D_s > L$. Hence the algorithm can terminate and confidently guarantee that none of the states that remain to be sampled belong to $S_L$. Given that $j' \leq B$, the sample complexity (in the sense of Def.\,\ref{def_sample_complexity_L}) is bounded by $\Phi(j') \leq \Phi(B) = \wt{O}\left(B L + L^{3/2} S^2 A \right)$.
\end{proof}

The algorithm $\OSPmathL$ requires no computational overhead w.r.t.\,\OSP, as it simply tracks the total duration of \OSP and terminates if it exceeds a threshold depending on $L$. Under the new appropriate definition of sample complexity of Def.\,\ref{def_sample_complexity_L}, the dependency in Thm.\,\ref{theorem_upper_bound} on the possibly very large or infinite diameter $D$ is effectively replaced by the reachability threshold $L$. A large value of $L$ signifies that the sample collection is required at quite difficult-to-reach states, while a small value of $L$ keeps in check the duration of the sampling procedure.

Narrowing the sample collection to states in $\cS_L$ may seem at first glance restrictive. Indeed, the presence of states in which the agent may get stuck could disrupt the learning process. However, assume for instance that we consider the canonical assumption made in episodic RL of a \textit{resetting} environment, i.e., an environment that contains a reset action that brings the agent with probability 1 to a reference starting state $s_0$ (where here we consider that the reset action can be executed at any time step for simplicity). Then we have that $\left\{ s \in \cS : \min_{\pi} \mathbb{E}\left[ \tau_{\pi}(s_0 \rightarrow s) \right] \leq L - 1 \right\} \subseteq \cS_L$, which shows that numerous states can effectively belong to the set $\cS_L$.

Finally, let us delve into the particular case of a weakly communicating MDP, whose state space $\mathcal{S}$ can be partitioned into two subspaces \citep[][Sect.\,8.3.1]{puterman2014markov}: a communicating set of states (denoted $\cSc$) with each state in $\cSc$ accessible --- with non-zero probability --- from any other state in $\cSc$ under some stationary deterministic policy, and a (possibly empty) set of states that are transient under all policies (denoted $\cSt$). The sets $\cSc$ and $\cSt$ form a partition of $\cS$, i.e., $\cSc \cap \cSt = \emptyset$ and $\cSc \cup \cSt = \cS$. Finally, we denote by $D^{\textsc{C}} < + \infty$ the diameter of the communicating part of $M$ (i.e., restricted to the set $\cSc$), i.e., $D^{\textsc{C}} := \max_{s \neq s' \in \cSc} \min_{\pi \in \SD} \mathbb{E}\left[ \tau_{\pi}(s \rightarrow s')\right] < + \infty$. Assume that the starting state $s_0$ belongs to $\cSc$. We expect the optimal strategy to perform the sample collection at states in $\cSc$ and discard the sample collection at states in $\cSt$. This is what $\OSPmathL$ does if we have $\mathcal{S}_L = \mathcal{S}^{\textsc{C}}$, i.e., whenever $D^{\textsc{C}} \leq L$. Hence, in that setting, the optimal (yet critically unknown) value of the threshold $L$ would be $D^{\textsc{C}}$.

\vspace{0.2in}


\section{Application: Model Estimation (\MODESTtitle)}
\label{app_modest}

In this section we demonstrate that \OSP can be readily applied to tackle the \MODEST problem, as well as a \say{robust} variant called \RMODEST, both of which are defined as follows. The agent $\mathfrak{A}$ interacts with the environment and, after $t$ time steps, it must return an estimate $\wh{p}_{\mathfrak{A},t}$ of the transition dynamics, which naturally corresponds to the empirical average of the transition probabilities. The accuracy of the estimate and the corresponding sample complexity are evaluated as follows.
\begin{definition}\label{def_robust_modest}
    Given an accuracy level $\eta > 0$ and a confidence level $\delta \in (0,1)$, the \MODEST and \RMODEST sample complexity of an online learning algorithm $\mathfrak{A}$ are defined as 
    \begin{mysizebis}
    \begin{align*}
        \mathcal{C}_{\MODEST}(\mathfrak{A}, \eta, \delta) &:= \min \big\{t > 0: \mathbb{P}\big( \forall (s,a) \in \mathcal{S} \times \mathcal{A}, ~ \norm{\widehat{p}_{\mathfrak{A},t}(\cdot \vert s,a) - p(\cdot \vert s,a)}_{1} \leq \eta \big) \geq 1 - \delta \big\}, \\
        \mathcal{C}_{\RMODEST}(\mathfrak{A}, \eta, \delta) &:= \min \big\{t > 0: \mathbb{P}\big( \forall (s',s,a) \in \mathcal{S}^2 \times \mathcal{A},~ \abs{\widehat{p}_{\mathfrak{A},t}(s' \vert s,a) - p(s' \vert s,a)} \leq \eta \big) \geq 1 - \delta \big\},
    \end{align*}
    \end{mysizebis}
where $\widehat{p}_{\mathfrak{A},t}$ is the estimate (i.e., empirical average) of the transition dynamics $p$ after $t$ time steps.
\end{definition}%
We have the following sample complexity guarantees.
\begin{lemma}\label{lemma_modest}
Instantiating \OSP with two different sequences of sampling requirements yields respectively
\begin{align*}
    &\mathcal{C}_{\RMODEST}(\OSPmath, \eta, \delta) = \widetilde{O}\Big( \displaystyle\frac{D S A}{\eta^2} + D^{3/2} S^2 A \Big), \\
    &\mathcal{C}_{\MODEST}(\OSPmath, \eta, \delta) = \widetilde{O}\Big( \displaystyle\frac{D \Gamma S A}{\eta^2} + \frac{D S^2 A}{\eta} +  D^{3/2} S^2 A \Big).
    \end{align*}%
\end{lemma}%

\begin{proof}
We first focus on the \RMODEST objective with desired accuracy level $\eta$. From Def.\,\ref{def_robust_modest}, we would like that, for any state-action pair $(s,a)$ and next state~$s'$, the following condition holds:
\begin{align}
    \abs{\widehat{p}_{t}(s' \vert s,a) - p(s' \vert s,a)} \leq \eta.
\label{condition_accuracy}
\end{align}
From the empirical Bernstein inequality (see e.g., \citealp{audibert2009exploration, improved_analysis_UCRL2B}), we have with probability at least $1 - \delta$, for any time step $t \geq 1$ and for any state-action pair $(s,a)$ and next state $s'$,
\begin{align}
    \abs{\widehat{p}_{t}(s' \vert s,a) - p(s' \vert s,a)} \leq 2 \sqrt{ \frac{\widehat{\sigma}^2_{t}(s'\vert s,a)}{N^{+}_{t}(s,a)} \log\left(\frac{2 S A N^{+}_{t}(s,a)}{\delta} \right)} + \frac{6 \log\left(\frac{2 S A N^{+}_{t}(s,a)}{\delta} \right)}{N^{+}_{t}(s,a)},
\label{eq_empirical_bernstein_ineq}
\end{align}
where $N^{+}_{t}(s,a) := \max \{ 1, N_t(s,a) \}$ and where the $\wh{\sigma}_t^2$ are the population variance of transitions, i.e., $\wh{\sigma}_t^2(s' \vert s,a) := \wh{p}_t(s' \vert s,a)(1-\wh{p}_t(s' \vert s,a))$. Let us now define, for any $X, Y \geq 0$, the quantity
\begin{align*}
    \Phi(X, Y) := \left\lceil \frac{57 X^2 }{\eta^2} \left[ \log\left(\frac{8 e X \sqrt{2 S A} }{\sqrt{\delta}\eta}   \right) \right]^2 + \frac{24 Y}{\eta} \log\left( \frac{24 Y S A}{\delta \eta} \right) \right\rceil.
\end{align*}
Using a technical lemma (Lem.\,\ref{lem:bound_exp}), we can prove that condition \eqref{condition_accuracy} holds whenever the number of samples at the pair $(s,a)$ becomes at least equal to
\begin{align*}
    \phi_t^{\RMODEST}(s,a) := \Phi\left( X, Y \right), \quad \quad X := \max_{s' \in \cS} \sqrt{\widehat{\sigma}^2_{t}(s'\vert s,a)}, \quad \quad Y := 1.
\end{align*}
We thus execute \OSP until there exists a time step $t \geq 1$ such that $b_t(s,a) := \phi_t^{\RMODEST}(s,a)$ samples have been collected at each state-action pair $(s,a) \in \SA$. Although the sampling requirement $b_t$ depends on the time step $t$, this is not an issue from Sect.\,\ref{subsection_sample_complexity_sampling_procedure} since for any $s \in \mathcal{S}$ and $t \geq 1$, $b_t(s,a)$ is bounded from above due to the fact that $\widehat{\sigma}^2_{t}(s'\vert s,a) \leq \frac{1}{4}$. This means that the total requirement for \RMODEST is $B_{\RMODEST} = \wt{O}\left(S A / \eta^2\right)$, which yields the first bound of Lem.\,\ref{lemma_modest}.

We now turn to the \MODEST objective. \OSP collects samples until there exists a time step $t$ such that the number of samples at each pair $(s,a)$ is at least equal to
\begin{align*}
    \phi_t^{\MODEST}(s,a) := \Phi\left( X, Y \right), \quad \quad X := \sum_{s' \in \cS} \sqrt{\widehat{\sigma}^2_{t}(s'\vert s,a)}, \quad \quad Y := S.
\end{align*}
Introducing $\Gamma(s,a) := \norm{p(\cdot \vert s,a)}_0$ the maximal support of $p(\cdot \vert s,a)$, we use the following inequality (valid at any time step $t \geq 1$): $\sum_{s' \in \cS} \widehat{\sigma}_{t}(s' \vert s,a) \leq \sqrt{\Gamma(s,a) - 1}$ (see e.g., \citealp[][Lem.\,4]{improved_analysis_UCRL2B}).
This means that the total requirement for \MODEST is $B_{\MODEST} = \wt{O}\left(\frac{\sum_{s,a} \Gamma(s,a)}{\eta^2} + \frac{S^2 A}{\eta} \right)$. Plugging in the result of Thm.\,\ref{theorem_upper_bound} finally yields the second bound of Lem.\,\ref{lemma_modest} (which corresponds to the statement of Lem.\,\ref{prop_arme} in Sect.\,\ref{ex}).
\end{proof}


\begin{lemma}\label{lem:bound_exp}
For any $x\geq 2$ and $a_{1},a_{2},a_{3},a_{4}>0$ such that $a_{3}x \leq a_{1}\sqrt{x}\log(a_{2}x) + a_4 \log(a_2 x)$, the following holds
$$x\leq \frac{4a_4}{a_3} \log\left( \frac{2a_4 a_2}{a_3}\right) + \frac{128 a_1^2}{9 a_3^2} \left[ \log\left(\frac{4 a_1 \sqrt{a_2} e}{a_3} \right)\right]^2.$$
\end{lemma}

\begin{proof}
Assume that $a_{3}x \leq a_{1}\sqrt{x}\log(a_{2}x) + a_4 \log(a_2 x)$. Then we have $\frac{a_{3}}{2}x \leq - \frac{a_{3}}{2}x + a_{1} \sqrt{x}\log(a_{2}x) + a_4 \log(a_2 x)$. From Lem.\,\ref{lemma_kazerouni_1} we have
\begin{align*}
    - \frac{a_{3}}{2}x + a_{1}\sqrt{x}\log(a_{2}x) \leq \underbrace{\frac{32 a_1^2}{9 a_3} \left[ \log\left(\frac{4 a_1 \sqrt{a_2} e}{a_3} \right)\right]^2}_{:= a_0}.
\end{align*}
Thus we have $x \leq \frac{2a_4}{a_3} \log(a_2 x) + \frac{2a_0}{a_3}$ and we conclude the proof using Lem.\,\ref{lem:bound_log}.
\end{proof}

\begin{lemma}[\citealp{kazerouni2017conservative}, Lem.\,8]\label{lemma_kazerouni_1}
For any $x \geq 2$ and $a_1, a_2, a_3 > 0$, the following holds
\begin{align*}
    -a_3 x + a_1 \sqrt{x} \log(a_2 x) \leq \frac{16 a_1^2}{9 a_3} \left[ \log\left(\frac{2 a_1 \sqrt{a_2} e}{a_3} \right)\right]^2.
\end{align*}
\end{lemma}

\begin{lemma}\label{lem:bound_log}
Let $b_1$, $b_2$ and $b_3$ be three positive constants such that $\log(b_1 b_2) \geq 1$. Then any $x > 0$ satisfying $x \leq b_1 \log(b_2 x) + b_3$ also satisfies $x \leq 2 b_1 \log(2 b_1 b_2) + 2b_3$.
\end{lemma}
\begin{proof}
Assume that $x \leq b_1 \log(b_2 x) + b_3 $ and set $y = x - b_3$. If $y \leq b_3$, then we have $x \leq 2 b_3$. Otherwise, we can write $y \leq b_1 \log(b_2 y + b_2 b_3) \leq b_1 \log(2 b_2 y)$. From Lem.\,\ref{lemma_kazerouni_2} we have $y \leq 2 b_1 \log(2 b_1 b_2)$, which concludes the proof.
\end{proof}

\begin{lemma}[\citealp{kazerouni2017conservative}, Lem.\,9]\label{lemma_kazerouni_2} Let $b_1$ and $b_2$ be two positive constants such that $\log(b_1 b_2) \geq 1$. Then any $x > 0$ satisfying $x \leq b_1 \log(b_2 x)$ also satisfies $x \leq 2 b_1 \log(b_1 b_2)$.
\end{lemma}


\section{Application: Sparse Reward Discovery (\TREASUREtitle Problem)}
\label{app_alternative_approaches}

In this section, we focus on the canonical sampling requirement of the \TREASURE problem of Sect.\,\ref{subsection_treasure}, where each state-action pair must be visited at least once. We illustrate how direct adaptations of existing algorithms are not able to match the guarantees of \OSP in Lem.\,\ref{prop_ohrd}.

\paragraph{Discussion on finite-horizon or discounted PAC-MDP algorithms.}

At first glance, an approach to tackle the \TREASURE problem could be to consider a well-known PAC exploration algorithm such as \RMAX \citep{brafman2002r} (the same discussion holds for \Ethree of \citealp{kearns2002near}). In particular, we can examine the \ZERORMAX variant proposed in \citep{jin2020reward}. Indeed the demarcation between known states and unknown states is an algorithmic principle related to the problem at hand: a state is considered known when the number of times each action has been executed at that state is at least $m$ for a suitably chosen $m$ and its reward is set to 0, while an unknown state receives a reward of 1. The set of known states captures what has been sufficiently sampled (and the empirical estimate of the transitions is used), while the set of unknown states drives exploration to collect additional samples. The central concept for analyzing the sample complexity of the algorithm is the escape probability (i.e., the probability of visiting the unknown states), which, in the case of $m=1$, would amount exactly to the probability of collecting a required sample in the \TREASURE problem. However, despite the similarities, \ZERORMAX (as well as \RFRLExplore of~\citealp{jin2020reward}) are designed in the infinite-horizon discounted setting or the finite-horizon setting. As such, only a finite number of steps is relevant, and the episode lengths (and resulting sample complexity) directly depend on the discount factor $\gamma$ or on the horizon $H$, respectively. Such approach cannot be employed in the setting of communicating MDPs, where there is no known imposed horizon of the problem, and where the agent must interweave the policy planning and policy execution processes by defining algorithmic episodes. As such, despite bearing high-level similarity with \OSP at an algorithmic level, such finite-horizon (or discounted) guarantees cannot be translated to sample complexity for the \TREASURE problem.

\paragraph{Leveraging \UCRL.}
We now analyze \UCRLtwo \citep{jaksch2010near}, an efficient algorithm for reward-dependant exploration in the infinite-horizon undiscounted setting. In order to tackle the \TREASURE problem, a first approach could be to consider true rewards of zero everywhere while the uncertainty around the rewards remains, i.e., the algorithm observes as reward $r(s,a) \sim \sqrt{\frac{1}{N^+(s,a)}}$, which corresponds to the usual uncertainty on the rewards \citep{jaksch2010near}, with $N(s,a)$ denoting the number of visits of $(s,a)$ so far. The underlying idea is that as the algorithm visits a state-action pair, its observed reward will decrease, thus favoring the visitation of non-sampled state-action pairs. Yet while this algorithm is fairly intuitive, it appears tricky to directly leverage the analysis of \UCRLtwo to obtain a guarantee on the time the algorithm requires to solve the \TREASURE problem. Indeed, the inspection of the tools used in the regret derivation of \UCRLtwo does not point out to a step in the analysis which explicitly lower bounds state-action visitations.

Another possibility is to design a non-stationary reward signal to feed to \UCRLtwo. Namely, assigning a reward of 1 if the state is under-sampled and 0 otherwise, corresponds to a sensible strategy (note that this reward signal changes according to the behavior of the algorithm). Yet as explained in \citep{tarbouriech2019no}, for any SSP problem with unit costs, the SSP-regret bound that is obtained from the analysis of average-reward techniques (by assigning a reward of 1 at the goal state, and 0 everywhere else) is worse than that obtained from the analysis of SSP goal-oriented techniques. This difference directly translates into a worse performance of \UCRLtwo-based approaches for the \TREASURE problem. Indeed, retracing the analysis of \citep[][App.\,B]{tarbouriech2019no}, we obtain that $\wt{O}(D_s^3 S^2 A)$ time steps are required to collect a sought-after sample when running the algorithm \UCRLtwoB \cite{improved_analysis_UCRL2B} (which is a variant of \UCRLtwo that constructs confidence intervals based on the empirical Bernstein inequality rather than Hoeffding's inequality and thus yields tighter regret guarantees). Since the analysis renders the re-use of samples difficult, performing this reasoning for each sought-after state to sample yields a total \TREASURE sample complexity of $\wt{O}\left(\sum_{s \in \cS} D_s^3 S^2 A \right)$, which is always worse than the bound in Lem.\,\ref{prop_ohrd} since $\max_s D_s = D$.

\paragraph{Leveraging \MaxEnt.}
At first glance, an alternative and natural approach to visit each state-action pair at least once may be to optimize the \MaxEnt objective over the state-action space, i.e., maximize the entropy function $H$ over the stationary state-action distributions $\lambda \in \Lambda$,
\begin{align*}
    H(\lambda) := \sum_{(s,a) \in \SA} - \lambda(s,a) \log(\lambda(s,a)).
\end{align*}
This objective --- over the state space, yet the extension to the state-action space is straightforward --- was studied in \citep{hazan2019provably} in the infinite-horizon discounted setting and in \citep{cheung2019exploration} in the infinite-horizon undiscounted setting. Following the latter, there exists a learning algorithm such that, with overwhelming probability,
\begin{align}
    H(\lambda^{\star}) - H(\wt{\lambda}_t) = \wt{O}\left( \frac{D S^{1/3}}{t^{1/3}} + \frac{D S \sqrt{A} }{\sqrt{t}} \right),
\label{eq_max_entropy}
\end{align}
where $\lambda^{\star} \in \argmax_{\lambda \in \Lambda} H(\lambda)$ and $\wt{\lambda}_t$ is the empirical state-action frequency at time $t$, i.e., $\wt{\lambda}_t(s,a) = \frac{N_t(s,a)}{t}$. The \TREASURE sample complexity translates into the first time step $t \geq 1$ such that $\wt{\lambda}_t(s,a) \geq \frac{1}{t}$ for all $(s,a) \in \SA$. However, the state-action entropy $H$ corresponds to the sum of a function related to each state-action frequency, and maximizing it provides no guarantee on each summand, i.e., on each state-action frequency. Indeed, assume that there exists a time $t$ such that $\wt{\lambda}_t(s,a) \geq \frac{1}{t}$ for all $(s,a) \in \SA$. This implies that $H(\wt{\lambda}_t) \geq \frac{SA}{t} \log(t)$. However, the regret bound of Eq.\,\ref{eq_max_entropy} cannot be leveraged to show that $t$ must necessarily be small enough. Overall, it seems that directly optimizing \MaxEnt is unfruitful in guaranteeing the visitation of each state-action pair at least once, and thus in provably enforcing the \TREASURE objective.

Instead of maximizing \MaxEnt, the discussion above encourages us to optimize the \say{worst-case} summand of the entropy function, by maximizing over $\Lambda$ the following function
\begin{align*}
    F(\lambda) := \min_{(s,a) \in \SA} \lambda(s,a).
\end{align*}
It is straightforward to show that $F$ is concave in $\lambda$ (as the minimum of $S \times A$ concave functions), as well as $1$-Lipschitz-continuous w.r.t.\,the Euclidean norm $\norm{\cdot}_2$, i.e.,
\begin{align*}
    \forall (\lambda, \lambda') \in \Lambda^2, \abs{ F(\lambda) - F(\lambda')} \leq \norm{\lambda - \lambda'}_{\infty} \leq \norm{\lambda - \lambda'}_{2}.
\end{align*}
However, $F$ is a non-smooth function, therefore the Frank-Wolfe algorithmic design of \citep{hazan2019provably, cheung2019exploration} cannot be leveraged. Instead, we propose to use the mirror descent algorithmic design of \citep[][Sect.\,5]{cheung2019exploration} that can handle general concave functions. It guarantees that there exists a constant $\beta > 0$ such that, with overwhelming probability (here we exclude logarithmic terms for ease of exposition)
\begin{align*}
    F(\lambda^{\star}) - F(\wt{\lambda}_t) \leq \frac{\beta D }{t^{1/3}} + \frac{\beta D S \sqrt{A} }{\sqrt{t}}.
\end{align*}
Introduce $\omega^{\star} := F(\lambda^{\star}) = \min_{s,a} \lambda^{\star}(s,a) \in (0, \frac{1}{SA}]$. We then have
\begin{align}
    F(\wt{\lambda}_t) \geq \omega^{\star} - \frac{\beta D }{t^{1/3}} - \frac{\beta D S \sqrt{A} }{\sqrt{t}}.
\label{eq_worst_case_lambda}
\end{align}
Equipped with Eq.\,\ref{eq_worst_case_lambda}, we can easily prove that if
\begin{align}
    t = \Omega\left( \min\left\{ \frac{D^2 S^2 A}{(\omega^{\star})^2}, \frac{D^3}{(\omega^{\star})^3}  \right\} \right),
\label{sample_complexity_ohrd_worst_case_lambda}
\end{align}
then $F(\wt{\lambda}_t) \geq \frac{1}{t}$, which immediately implies that the \TREASURE is discovered. This sample complexity result is quite poor compared to Lem.\,\ref{prop_ohrd}. In particular, it depends polynomially on $(\omega^{\star})^{-1}$, which cannot be smaller than $SA$.

\vspace{0.2in}
\section{Application: Goal-Free Cost-Free Exploration in Communicating MDPs}
\label{app_extension_goalfree_costfree}


\subsection{Reward-Free Exploration in Finite-Horizon MDPs vs.\,Cost-Free Exploration in Goal-Conditioned RL}

Jin et al.\,\citep{jin2020reward} introduced the reward-free framework in the finite-horizon case, which we recall is a special case of a goal-oriented (i.e., SSP) problem where each episode terminates after exactly $H$ steps. The agent receives as input an accuracy level $\epsilon > 0$, a confidence level $\delta \in (0,1)$, the state and action spaces, and the horizon $H$, while no knowledge is provided about the transition model $p$. The learning process is decomposed into two phases. \ding{172}~\textit{Exploration phase:} The agent first collects trajectories from the MDP without a pre-specified reward function and returns an estimate of the transition model $\wh p$. \ding{173}~\textit{Planning phase:} The agent receives an arbitrary reward function and is tasked with computing an $\epsilon$-optimal policy with probability at least $1-\delta$, without any additional interaction with the environment. The objective is to minimize the duration of the exploration phase needed to simultaneously enforce any requested planning guarantee. 

In \citep{jin2020reward} the reward-free exploration problem is studied for any arbitrary MDP, where there may exist states that are difficult or impossible to reach. The core mechanism in their analysis is to partition the states depending on their ease of being reached within $H$ steps. Specifically, they distinguish between \textit{significant} states, that can be sufficiently visited and whose transition probability can thus be accurately estimated, and \textit{insignificant} states that are too difficult to reach within $H$ steps, but therefore have negligible contribution to any reward optimization.

\begin{minipage}{0.6\linewidth}
    Interestingly, in the goal-conditioned setting this distinction may no longer be meaningful. By way of illustration, consider any fixed horizon $H$ and the toy environment in Fig.\,\ref{fig:toy_ex_H}. Suppose that the objective is to quickly reach state $z$ (i.e., the goal state is $z$, the starting state is $x$ and all costs are equal to~1). Even though state $y$ is \textit{insignificant} within $H$ steps (in the finite-horizon sense of \citealp{jin2020reward}, for any positive \say{significance level}), it is actually crucial in solving the objective, as $z$ can be reached deterministically in 1 step from $y$. Extrapolating this scenario, in the goal-conditioned setting, we may have an effective horizon of $H = + \infty$ for some goals, which implies that the transition model $p$ must be accurately estimated across the \textit{entire} state-action space to ensure that a near-optimal goal-conditioned policy can be computed.
\end{minipage}%
\hfill%
\begin{minipage}{0.36\linewidth}
\definecolor{light-gray}{gray}{0.95}
    \flushright
	\begin{tikzpicture}[thick,scale=0.9,rotate=90,transform shape]
	\node[circle,fill=light-gray,rotate=-90,transform shape] at (1,2.5) (1) {\scriptsize{$x$}}; 
    \node[circle,fill=black] at (0,1.5) (2) {}; 
    \node[circle,fill=black] at (2,1.5) (4) {}; 
    \node[circle,fill=black] at (0,-1) (5) {}; 
    \node[circle,fill=black] at (0,0.5) (2bis) {}; 
    \node[circle,fill=black] at (2,0.5) (4bis) {}; 
    \node[circle,fill=black] at (2,-1) (7) {}; 
    \node[circle,fill=black] at (2,-2) (8) {}; 
    \node[circle,fill=light-gray,rotate=-90,transform shape] at (0,-2) (9) {\scriptsize{$z$}}; 
    \node[circle,fill=light-gray,rotate=-90,transform shape] at (2,-3) (10) {\scriptsize{$y$}}; 
	\begin{scope}[>={Stealth[black]},
	every node/.style={fill=white,circle},
	every edge/.style={draw=gray, thick},
	every loop/.style={draw=gray, thick, min distance=5mm,looseness=5}]
	\path[]
	(1) [->,thick] edge[] node[scale=0.001, text width = 0mm] {} (2)
	(1) [->,thick] edge[] node[scale=0.001, text width = 0mm] {} (4)
	(2) [->,thick] edge[] node[scale=0.001, text width = 0mm] {} (2bis)
	(2bis) [->,thick] edge[dashed] node[scale=0.001, text width = 0mm] {} (5)
	(4) [->,thick] edge[] node[scale=0.001, text width = 0mm] {} (4bis)
	(4bis) [->,thick] edge[dashed] node[scale=0.001, text width = 0mm] {} (7)
    (5) [->,thick] edge[] node[scale=0.001, text width = 0mm] {} (9)
    (7) [->,thick] edge[] node[scale=0.001, text width = 0mm] {} (8)
    (8)[->,thick] edge[] node[scale=0.001, text width = 0mm] {} (10)
    (10)[->,thick] edge[bend left=30] node[scale=0.001, text width = 0mm] {} (9);
	\end{scope}
	\end{tikzpicture}
    
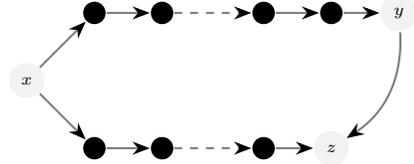
\captionof{figure}{\small The agent starts at state $x$ and reaches $z$ in $H$ steps with probability $1/2$, and $y$ in $H+1$ steps with probability $1/2$. From state $y$ the agent deterministically transitions to state $z$ in 1 step.}
    \label{fig:toy_ex_H}
\end{minipage}%


As a result, the challenges that emerge in the cost-free exploration problem in goal-conditioned RL are orthogonal to the ones in finite-horizon \citep{jin2020reward}: a \textit{constraint on the environment is added} (all states must now be reachable, Asm.\,\ref{asm_communicating}), allowing the \textit{removal of the constraint on performance} (which is not limited to $H$ steps anymore) and thus enabling to tackle the more general class of goal-oriented problems. 

For a designated goal state $g \in \cS$, recall that the SSP objective is to compute a policy $\pi: \mathcal S\rightarrow \mathcal A$ minimizing the cumulative cost before reaching $g$. Formally, the (possibly unbounded) value function is defined as
\begin{align*}
    V_{\pi}(s \rightarrow g) := \mathbb{E}\bigg[ \sum_{t=1}^{\tau_{\pi}(s \rightarrow g)} c(s_t, \pi(s_t)) ~\big\vert~ s_1 = s \bigg],
\end{align*}
where $\tau_{\pi}(s \rightarrow g) := \inf \{ t \geq 0: s_{t+1} = g \,\vert\, s_1 = s, \pi \}$ is the (random) number of steps needed to reach $g$ from $s$ when executing policy $\pi$. An optimal policy (if it exists) is denoted by $\pi^\star \in \arg\min_\pi V_{\pi}(s \rightarrow g)$. 

Without loss of generality, we consider throughout that the maximum $c_{\max}$ of the cost functions that we intend to consider in the planning phase is equal to 1. On the other hand, the minimum value $c_{\min}$ has a more subtle impact on the type of performance guarantees we can obtain. Following \citep{tarbou}, for any cost function $c$ and any pair of initial and goal states $s$ and $g$, we introduce a slack parameter $\theta \in [1, + \infty]$ and we say that a policy $\wh \pi$ is $(\epsilon,\theta)$-optimal if~\footnote{This reduces to standard $\epsilon$-optimality for $\theta\rightarrow \infty$.}
\begin{align}
    V^{\wh{\pi}}(s \rightarrow g) \leq \min_{\pi : \mathbb{E}\left[ \tau_{\pi}(s \rightarrow g) \right] \leq \theta D_{s, g}} V^{\pi}(s \rightarrow g) + \epsilon.
\label{eq_eps_theta_optimal}
\end{align}
We consider this restricted optimality only in the general cost case of $c_{\min} = 0$, where the $(\epsilon, + \infty)$-optimal policy may not be proper \citep{tarbouriech2019no, cohen2020near}. In that case, we are interested in finding the best proper policy, which is what the restricted optimality in Eq.\,\ref{eq_eps_theta_optimal} enables as it constrains the targeted policy to be proper. This consideration is required in the analysis of \citep{tarbou} when translating the performance from the cost-perturbed MDP to the original MDP, which needs constraining the expected goal-reaching time of the targeted policy. 

We are now ready to formally define the \textit{goal-free cost-free exploration} problem. It is characterized by an accuracy level $0 < \epsilon \leq 1$, a confidence level $\delta \in (0,1)$, a minimum cost $c_{\min} \in [0, 1]$ and a slack parameter $\theta \in [1, + \infty]$ (and we allow either $c_{\min} = 0$ or $\theta = + \infty$, but not both simultaneously). After its exploration phase (whose number of time steps defines the sample complexity of the problem), the agent is expected to be able to compute, with probability at least $1-\delta$, an $(\epsilon, \theta)$-optimal goal-conditioned policy $\wh \pi$ for \textit{any} goal state $g \in \cS$ and \textit{any} cost function $c \in [c_{\min}, 1]$, i.e., satisfying Eq.\,\ref{eq_eps_theta_optimal} for all $s \in \cS$.

\subsection{Proof of Lem.\,\ref{theorem_goal_free_cost_free_main}}

We show that instantiating \OSP for carefully selected sampling requirements $b_t(s,a)$ enables to obtain the guarantee of Lem.\,\ref{theorem_goal_free_cost_free_main}. To do so, we build on the sample complexity analysis of solving a fixed-goal SSP problem with a generative model of \citep{tarbou}. Specifically, we introduce the following sampling requirement function
\begin{align}
     \phi(X,y) := \alpha \cdot \Bigg( \frac{X^3 \wh{\Gamma}}{y \epsilon^2} \log\left( \frac{X S A }{y \epsilon \delta} \right) + \frac{X^2 S}{y \epsilon} \log\left( \frac{X S A }{y \epsilon \delta} \right) + \frac{X^2 \wh{\Gamma}}{y^2} \log^2\left( \frac{X S A }{y \delta} \right) \Bigg),
\label{alloc_function_1}
\end{align}
where $\alpha > 0$ is a numerical constant and $\wh{\Gamma} := \max_{s,a} \norm{ \widehat{p}(\cdot\vert s,a)}_0 \leq \Gamma$ is the largest support of $\wh p$.

The sampling requirement function of Eq.\,\ref{alloc_function_1} instantiated for specific values of $X$ and $y$ is used to guide the \OSP algorithm. Specifically, the analysis distinguishes between two cases: \textit{either} $c_{\min} > 0$ and the cost function considered in the planning phase can be the same as the original one, \textit{or} $c_{\min} = 0$ and all costs incur an additive perturbation of $ \epsilon / (\theta D) > 0$ (as considered in the analysis of \citep{tarbou}). As stated in Sect.\,\ref{subsection_costfree_goalfree}, we set $\omega := \max\big\{ c_{\min}, \epsilon / (\theta D) \big\}$, which is guaranteed to be positive since we enforce either $c_{\min} = 0$ or $\theta = + \infty$, but not both simultaneously. As such, in Eq.\,\ref{alloc_function_1} we define $y := \omega$ to be equal to the minimum cost of either the true or the perturbed cost function. As for the value of $X$, we perform the following distinction of cases.

\ding{172} First let us assume that the learning agent has prior knowledge of the diameter $D$ of the MDP. Then we set $X:=D$. Since the analysis of \citep{tarbou} accurately estimates the transition kernel and thus holds for arbitrary cost function in $[\omega, 1]$, we can ensure that collecting at least $\phi(D,\omega)$ samples from each state-action pair provides the $\epsilon$-optimality cost-free planning guarantee of Lem.\,\ref{theorem_goal_free_cost_free_main}. The total time required to collect such samples is upper bounded by $D S A \phi(D,\omega^{-1})$, which directly yields the sample complexity guarantee stated in Lem.\,\ref{theorem_goal_free_cost_free_main}.

\ding{173} Second we show that we can relax the assumption of knowing the diameter $D$ without altering the sample complexity guarantee. To do so, we begin the algorithm by a procedure which computes a quantity $\wh D$ such that $D \leq \wh{D} \leq D(1+\epsilon)$ with high probability. From App.\,\ref{app_diameter_estimation}, this can be done in $\wt{O}(D^3 S^2 A / \epsilon^2)$ time steps by leveraging \OSP. We thus begin the algorithm by running such diameter-estimation subroutine. Crucially, we note that its sample complexity is subsumed in the total sample complexity of Lem.\,\ref{theorem_goal_free_cost_free_main}. Then we simply apply the reasoning in case \ding{172} by considering $X := \wh{D}$ in the allocation of Eq.\,\ref{alloc_function_1} instead of $X = D$. Since $\wh D$ is a sufficiently tight upper bound on $D$ (i.e., $\wh D = O(D)$), we ultimately obtain the same sample complexity guarantee as in case \ding{172}.

\vspace{0.15in}
\section{Other Applications}\label{ssec:other.applications}
\label{app_other_applications}

In this section, we provide additional applications where \OSP can be leveraged to readily obtain an online learning algorithm. We first summarize them here.

\textbf{Diameter estimation (see App.\,\ref{app_diameter_estimation}).} \OSP can be leveraged to estimate the MDP diameter $D$. In App.\,\ref{app_diameter_estimation} we develop a \OSP-based procedure that computes an estimate $\wh{D}$ such that $D \leq \wh{D} \leq (1+\epsilon)D$ in $\wt{O}( D^3 S^2 A / \epsilon^2)$ time steps. This improves on the diameter estimation procedure recently devised in \citep{zhang2019regret} by a multiplicative factor of $D S^2$. As $\wh{D}$ provides an upper bound on the optimal bias span $sp(h^{\star})$, our procedure may be of independent interest for initializing average-reward regret-minimization algorithms that leverage prior knowledge of $sp(h^{\star})$ (as done in e.g., \citep{zhang2019regret}).

\textbf{PAC-policy learning (see App.\,\ref{app_pac}).}
One of the most common \SO-based settings is the computation of an $\epsilon$-optimal policy via sample-based value iteration. Since \OSP is agnostic to how the sampling requirements are generated, we can easily integrate it with any state-of-the-art \SO-based algorithm and directly inherit its properties. For instance, in App.\,\ref{app_pac} we show that \OSP can be easily combined with \textsc{Bespoke}~\cite{zanette2019almost} to obtain a competitive online learning algorithm for the policy learning problem. In fact, the sample complexity of the resulting algorithm is only a factor $D$ worse than existing online learning algorithms in the worst case and, leveraging the refined problem-dependent bounds of \textsc{Bespoke}, it is likely to be superior in many MDPs.

\textbf{Bridging bandits and MDPs with \OSP (see App.\,\ref{app_extension_bandit_MDP}).} In multi-armed bandit (MAB) an agent directly collects samples by pulling arms. If we map each arm to a state-action pair, we can see any MAB algorithm as having access to an \SO. As such, we can readily turn any bandit algorithm into an RL online linear algorithm by calling \OSP to generate the samples needed by the MAB algorithm. Exploiting this procedure, in App.\,\ref{app_extension_bandit_MDP} we show how we can tackle problems such as \textit{best-state identification} and \textit{active exploration} (i.e., state-signal estimation) in the communicating MDP setting, for which no specific online learning algorithm exists yet.


\subsection{Application: Diameter Estimation}
\label{app_diameter_estimation}

\OSP can be leveraged to estimate the diameter $D$ which is a quantity of interest in the average-reward setting. Indeed, $D$ dictates the performance of reward-based no-regret algorithms \citep{jaksch2010near}, and some works assume that an upper bound on the optimal bias span $sp(h^{\star})$ is known (e.g., \citep{jian2019exploration}). Since we have $sp(h^{\star}) \leq r_{\max} D$ (e.g., \citep{bartlett2009regal}), upper bounding $D$ enables to relax this assumption. Recently, for such purpose of upper bounding $sp(h^{\star})$, \citep{zhang2019regret} developed an initial procedure based on successive applications of \UCRLtwo that can compute an estimate $\wh{D}$ such that $D \leq \wh{D} \leq (1+\epsilon)D$ in $\wt{O}( D^4 S^4 A / \epsilon^2)$ time steps (see \citep[][App.\,D \& Alg.\,3 \say{LD: Learn the Diameter}]{zhang2019regret}). In Alg.\,\ref{algo:subroutine} we derive an iterative estimation procedure based on \OSP which can compute such upper bound of $D$ faster, namely in $\wt{O}(D^3 S^2 A / \epsilon^2)$ time steps, while simultaneously providing an accurate estimation of the transition dynamics. As such it may be an initial procedure of independent interest for regret-minimization algorithms in the average-reward setting. Note that the procedure is similar to the one considered in \citep{tarbou} to estimate an upper bound of the SSP-diameter of a given SSP problem in the generative model case, while here we focus on the estimation of the diameter (i.e., worst-case SSP diameter) in the \textit{online} case by leveraging \OSP.

We define a notation used throughout the section, $\norm{U}_{\infty}^{\infty} := \max_{s,s'} U(s \rightarrow s')$, which holds for any quantity $U$ that can be naturally mapped to a $\cS \times \cS$ matrix.

\begin{lemma}\label{lemma_delta_subroutine}
    With probability at least $1-\delta$, Alg.\,\ref{algo:subroutine}:
    \begin{itemize}[leftmargin=.2in,topsep=-4pt,itemsep=0pt,partopsep=0pt, parsep=0pt]
\item has a sample complexity bounded by $\wt{O}\left( D^3 S^2 A / \epsilon^2 \right)$,
\item requires at most $\log_2\left(D(1+\epsilon)\right) + 1$ inner iterations,
\item solves the \MODEST problem for an accuracy level $\eta > 0$ and outputs an optimistic $\cS \times \cS$ matrix $\wt{v}$ s.t.\,$\frac{\epsilon}{2 D} \leq \eta \leq \frac{\epsilon}{\norm{\wt{v}}_{\infty}^{\infty}}$,
\item outputs a quantity $\wh D := \left( 1 + 2 \eta \norm{\wt{v}}_{\infty}^{\infty} \right) \norm{\wt{v}}_{\infty}^{\infty}$ that verifies $D \leq \wh D \leq \left( 1 + 2 \epsilon (1+\epsilon) \right) (1+\epsilon) D$.
\end{itemize}
\end{lemma}

\begin{algorithm}[tb]
\begin{small}
  \caption{\OSP-based procedure to estimate the diameter}
    \label{algo:subroutine}
\begin{algorithmic}[1]
  \STATE {\bfseries Input:} accuracy $\epsilon > 0$, confidence level $\delta \in (0,1)$.
  \STATE Set $W := \tfrac{1}{2}$ and $\norm{\wt{v}}_{\infty}^{\infty} := 1$.
  \WHILE{$\norm{\wt{v}}_{\infty}^{\infty} > W$}
  \STATE Set $W \leftarrow 2 W$.
  \STATE Set the accuracy $\eta := \frac{\epsilon}{W}$.
  \STATE Collect additional samples by running \OSP for the \MODEST problem with accuracy $\frac{\eta}{2}$ and confidence level $\delta$.
  \FOR{each state $s \in \cS$}
  \STATE Compute a vector $\wt{v}(\cdot \rightarrow s)$ using EVI for SSP, with goal state $s$, unit costs and \VI precision $\gammaVI := \frac{\min\{1,\epsilon\}}{2}$ (see App.\,\ref{app_value_iteration_SSP}).
  \ENDFOR
  \ENDWHILE
  \STATE {\bfseries Output:} the quantity $\wh D := \left( 1 + 2 \eta \norm{\wt{v}}_{\infty}^{\infty} \right) \norm{\wt{v}}_{\infty}^{\infty}$.
\end{algorithmic}
\end{small}
\end{algorithm}

\begin{proof}
  
We will assume throughout that the event $\mathcal{E}$ (defined in App.\,\ref{app_value_iteration_SSP}) holds. We now give a useful statement stemming from optimism:

``At any stage of Alg.\,\ref{algo:subroutine}, for any given goal state, denote by $\wt{v}$ the vector computed using EVI for SSP. Then under the event $\mathcal{E}$, we have component-wise (i.e., starting from any non-goal state): $\wt{v} \leq \min_{\pi} V^{\pi}_p \leq D$.''

To prove this useful statement, we observe that the first inequality stems from Lem.\,\ref{lemma_app_value_iteration_SSP} of App.\,\ref{app_value_iteration_SSP} while the second inequality uses the definition of the diameter $D$ and the fact that the considered costs are equal to $1$.

Now, denote by $n$ the iteration index of the Alg.\,\ref{algo:subroutine} (starting at $n=1$), so that $W_n = 2^n$. Introduce $N := \min \{n \geq 1: \norm{\wt{v}_n}_{\infty}^{\infty} \leq W_n \}$. We have $\norm{\wt{v}_n}_{\infty}^{\infty} \leq D$ at any iteration $n \geq 1$ from the useful statement on optimism above. Since $(W_n)_{n \geq 1}$ is a strictly increasing sequence, Alg.\,\ref{algo:subroutine} is bound to end in a finite number of iterations (i.e., $N < + \infty$), and given that \mbox{$W_{N-1} \leq \norm{\wt{v}_{N-1}}_{\infty}^{\infty} \leq D$}, we get $N \leq \log_2\left(D\right) + 1$. Moreover, we have $\norm{\wt{v}_{N}}_{\infty}^{\infty} \leq W_{N}$ and $\eta_N = \frac{\epsilon}{W_N}$, which implies that $\eta_N \leq \frac{\epsilon}{\norm{\wt{v}_{N}}_{\infty}^{\infty}}$. Moreover, combining $W_{N-1} \leq D$ and $W_{N-1} = \frac{W_N}{2} = \frac{\epsilon}{2 \eta_N}$ yields that $\frac{\epsilon}{2D} \leq \eta_N$.

Denote by $\eta := \eta_N$ the achieved \MODEST accuracy at the end of Alg.\,\ref{algo:subroutine}. Plugging in the guarantee of Prop.\,\ref{prop_arme} yields a sample complexity of
\begin{align*}
\wt{O}\Big( \frac{D S^2 A}{\eta^2} \Big) = \wt{O}\Big( \frac{D^3 S^2 A}{\epsilon^2} \Big).
\end{align*}
Denote by $\wt{v} := \wt{v}_N$ the optimistic matrix output by Alg.\,\ref{algo:subroutine}. Consider $(s_1, s_2) \in \arg\max_{(s,s')} \min_{\pi} \mathbb{E}\left[\tau_{\pi}(s \rightarrow s') \right]$. Denote by $\wt{\pi}$ the greedy policy w.r.t.\,the vector $\wt{v}(\cdot \rightarrow s_2)$ in the optimistic model with goal state $s_2$. Then we have
\begin{align*}
    D &= \min_{\pi} \mathbb{E}\left[\tau_{\pi}(s_1 \rightarrow s_2) \right]  \mathbb{E}\left[\tau_{\wt{\pi}}(s_1 \rightarrow s_2) \right] \myineeqa \left( 1 + 2 \eta \norm{\mathbb{E}\left[\wt{\tau}_{\wt{\pi}}\right]}_{\infty}^{\infty} \right) \mathbb{E}\left[\wt{\tau}_{\wt{\pi}}(s_1 \rightarrow s_2) \right] \\
    &\myineeqb \left( 1 + 2 \eta (1+\epsilon) \norm{\wt{v}}_{\infty}^{\infty} \right) (1+\epsilon) \wt{v}(s_1 \rightarrow s_2) \leq \left( 1 + 2 \eta (1+\epsilon) \norm{\wt{v}}_{\infty}^{\infty} \right) (1+\epsilon) \norm{\wt{v}}_{\infty}^{\infty} := \wh D \\
    &\myineeqc \left( 1 + 2 \eta (1+\epsilon) \norm{\wt{v}}_{\infty}^{\infty} \right) (1+\epsilon) D \myineeqd \left( 1 + 2 \epsilon (1+\epsilon) \right) (1+\epsilon) D,
\end{align*}
where (a) corresponds to an SSP simulation lemma argument (see \citep[][Lem.\,B.4]{cohen2020near}; \citep[][Lem.\,3]{tarbouriech2020improved}) given that a \MODEST accuracy of $\eta$ is fulfilled, (b) comes from the value iteration precision $\gammaVI := \frac{\min\{1,\epsilon\}}{2}$ which implies that $\mathbb{E}\left[\wt{\tau}_{\wt{\pi}}\right] \leq (1+2\gammaVI)\wt{v} \leq (1+\epsilon)\wt{v}$ component-wise according to Lem.\,\ref{lemma_app_value_iteration_SSP} of App.\,\ref{app_value_iteration_SSP}, (c) is implied by the useful statement on optimism given at the beginning of the proof, and finally (d) leverages that $\eta \norm{\wt{v}}_{\infty}^{\infty} \leq \epsilon$.
\end{proof}


\subsection{Application: PAC-Policy Learning}
\label{app_pac}

One of the most common \SO-based settings is the computation of an $\epsilon$-optimal policy via sample-based value iteration. Since \OSP is agnostic to how the sampling requirements are generated, we can easily integrate it with any state-of-the-art \SO-based algorithm and directly inherit its properties. For instance, consider the \textsc{Bespoke} algorithm introduced by \citep{zanette2019almost}. \textsc{Bespoke} proceeds through phases and at the beginning of each phase $k$, it determines the additional number of samples $n^{k+1}_{sa}$ that need to be generated at each state-action pair $(s,a)$ based on the estimates of the model and reward of the MDP computed so far. Then it simply queries the \SO as needed and it moves to the following phase. In order to turn \textsc{Bespoke} into an online learning algorithm, we can simply replace the query step by running \OSP until $n^{k+1}_{sa}$ samples are generated and then move to the next phase. Furthermore, let $b(s,a)$ be the total number of samples required by \textsc{Bespoke} in each state-action pair as stated by \citep[][Thm.\,2]{zanette2019almost}, then we can directly apply Thm.\,\ref{theorem_upper_bound} and obtain the sample complexity of the online version of \textsc{Bespoke} (\textsc{Online-Bespoke}). As discussed in Sect.\,\ref{so} the resulting complexity is \textit{at most} a factor $D$ larger than the one of (offline) \textsc{Bespoke} plus an additional term of order $\wt O(D^{3/2} S^2 A)$ independent from the desired accuracy $\epsilon$. It is interesting to contrast this result with existing online algorithms for this problem. While to the best of our knowledge, there is no algorithm specifically designed for optimal policy learning, we can rely on regret-to-PAC conversion (see e.g., \citep[][Sect.\,3.1]{jin2018is-q-learning}) to derive sample complexity guarantees for existing regret minimization algorithms and do a qualitative comparison.\footnote{Regret minimization guarantees are usually provided for the finite-horizon setting, while \textsc{Bespoke} is designed for the discounted setting. Furthermore, the $\epsilon$-optimality guarantees for \SO-based algorithms are typically defined in $\ell_\infty$ norm, while the regret-to-PAC conversion only provides guarantees on average w.r.t.\ the initial distribution.} For instance, we can use \textsc{Euler}~\cite{zanette2019tighter} to derive an $\epsilon$-optimal policy. If we consider a worst-case analysis, \textsc{Euler} achieves the same sample complexity of \textsc{Bespoke}, which in turn matches the lower bound of~\citep{azar2013minimax}. As a result, \textsc{Online-Bespoke} would be a factor $D$ suboptimal w.r.t.\,to \textsc{Euler}. Nonetheless, our \SO-to-online learning conversion approach enables \textsc{Online-Bespoke} to directly benefit from the problem-dependent performance of \textsc{Bespoke}, which in many MDPs may outperform the guarantees obtained by using \textsc{Euler} as a online learning algorithm for policy optimization.


\subsection{Application: Bandit Problems with MDP Dynamics}
\label{app_extension_bandit_MDP}

\newcommand*\rectangled[1]{%
   \tikz[baseline=(R.base)]\node[draw,rectangle,inner sep=0.5pt](R) {#1};\!
}

\subsubsection{Algorithmic protocol} 

The sampling procedure \OSP provides an effective way to collect samples for states of the agent's choosing, and can thus be related to the multi-armed bandit setting by mapping arms (in bandits) to states (in MDPs). From Thm.\,\ref{theorem_upper_bound}, each state can now be \say{pulled} within $\wt{O}(D)$ time steps (instead of a single time step in the bandit case). This allows to naturally extend some \textit{pure exploration} problems from the bandit setting to the communicating MDP setting. The algorithmic protocol alternates between the two following strategies:
\begin{enumerate}
    \item[\textcolor{pearThree}{\rectangled{1}}] the \say{bandit algorithm} identifies the arm(s), i.e., state(s), from which a sample is desired,
    \item[\textcolor{pearThree}{\rectangled{2}}] \OSP is executed to collect a sought-after sample as fast as possible.
\end{enumerate}

To illustrate our decoupled approach we consider the two following problems: best-state identification (App.\,\ref{app_beststateid}) and reward-estimation, a.k.a.\,active exploration (App.\,\ref{app_rewardestimation}).

\subsubsection{Best-state identification}
\label{app_beststateid}

This is the MDP extension of the best-arm identification problem in bandits \citep{audibert2010best}. Each state $s \in \cS := \{ 1, \ldots, S \}$ is characterized by a reward function $r_s$. For the sake of simplicity, we assume that the rewards are in $[0,1]$ and that there is a unique highest-rewarding state $s^{\star} := \arg\max_s r_s$. Let $r^{\star} := r_{s^{\star}}$. Consider a budget of $n$ steps. The objective is to bound the probability of error $e_n := \mathbb{P}(J_n \neq s^{\star})$, where $J_n$ is the state from which we desire a sample at step $n$. For $s \neq s^{\star}$, we introduce the following suboptimality measure of state $s$: $\Delta_s := r^{\star} - r_s$. We introduce the notation $(i) \in \{1, \ldots, S \}$ to denote the $i$–th best arm (with ties break arbitrarily). The hardness of the task will be characterized by the following quantities $H_1 := \sum_{s \in \cS} \frac{1}{\Delta_s^2}$ and $H_2 := \max_{s \in \cS} s \Delta_{(s)}^{-2}$. These quantities are equivalent up to a logarithmic factor since we have $H_2 \leq H_1 \leq \log(2S) H_2$. A fully connected MDP with known and deterministic transitions amounts to a multi-armed bandit problem of $K := S$ arms for our problem, thus the {\small\textsc{Successive Rejects}} algorithm \citep{audibert2010best} directly yields the following bound after $j$ time steps
\begin{align*}
    e_j \leq \frac{S(S-1)}{2} \exp\left( - \frac{j - SA}{\overline{\log}(S) H_2} \right), \quad \textrm{where~} \overline{\log}(S) := \frac{1}{2} + \sum_{i=1}^{S} \frac{1}{i}.
\end{align*}
In a general MDP, we combine \OSP (for the sample collection) with the {\small\textsc{Successive Rejects}} algorithm (for deciding which sample to collect). Consider any large enough budget of $n = \Omega(D^{3/2} S^2 A)$ time steps. Denote by $j_n$ the number of time steps during which \OSP effectively collects the desired sample stipulated by the {\small\textsc{Successive Rejects}} algorithm. Thm.\,\ref{theorem_upper_bound} yields that $n = \wt{O}\left( D j_n + D^{3/2} S^2 A \right)$, which means that $j_n = \wt{\Omega}\left( \frac{n - D^{3/2} S^2 A}{D} \right)$. Therefore we obtain the following guarantee.

\begin{lemma} \label{lemma_beststateid} In any unknown communicating MDP with unique highest-rewarding state $s^{\star}$, combining \OSP with the {\small\textsc{Successive Rejects}} algorithm~\citep{audibert2010best} yields the existence of a polynomial function $p$ such that the probability $e_n$ of wrongly identifying the \say{best state} $s^{\star}$ at time step $n$ is upper bounded by
\begin{align*}
    e_n \leq p(S, A, D, n) \exp\left( -\frac{n - D^{3/2} S^2 A}{D \log(S) H_2} \right),
\end{align*}
which corresponds to an exponential decrease w.r.t.\,$n$ whenever $n$ is large enough (i.e., after the $D^{3/2} S^2 A$ burn-in phase).
\end{lemma}

\subsubsection{Reward estimation (a.k.a.\,active exploration)}
\label{app_rewardestimation}

The objective of this problem in bandits (resp.\,MDPs) is to accurately estimate the mean pay-off (resp.\,the average reward signal) at each arm (resp.\,state). Note that this problem was originally studied in the bandit setting (see e.g., \citep{carpentier2011upper}) and recently extended in ergodic MDPs by \citep{tarbouriech2019active} using a Frank-Wolfe approach. The extension to communicating MDPs remained an open question, and it becomes immediately addressed with \OSP.
We recall the problem formulation: for a desired accuracy $\epsilon > 0$, for each state-action pair $(s,a) \in \cS \times \cA$ with mean reward $r_{s,a}$ in $[0,1]$, we seek to output an estimate $\widehat{r}_{s,a}$ such that $\abs{\widehat{r}_{s,a} - r_{s,a}} \leq \epsilon$. Under the \OSP framework, it is sufficient to visit each state-action pair at least $\Omega\left(\epsilon^{-2} \right)$ times, which directly induces the following sample complexity guarantee.

\begin{lemma}\label{lemma_activeexp}
  In any unknown communicating MDP, \OSP can reach any reward-estimation accuracy $\epsilon > 0$ with high probability under a sample complexity scaling as $$\wt{O}\left( \frac{DSA}{\epsilon^2} + D^{3/2} S^2 A \right).$$
\end{lemma}

\subsubsection{Comments}

\paragraph{Distinction between regret and sample complexity.} Note that the results above (Lem.\,\ref{lemma_beststateid} and \ref{lemma_activeexp}) do not provide any guarantee on the \textit{regret} of the corresponding algorithms (which is often the metric of interest in sequential learning). Indeed, our algorithmic approach does not track nor adapt to a notion of optimal performance. Likewise, there remains to derive lower bounds on these problems extended to MDPs, in order to quantity the optimality of our procedure. Nonetheless, our decoupled approach is, to the best of our knowledge, the first method with provably bounded sample complexity that can successfully extend classical bandit problems (such as the two aforementioned ones) to communicating MDPs.

\paragraph{On the link between MDPs and bandits with a special form of transportation costs.} Under the mapping between bandit arms and MDP states, our sampling paradigm has the effect of casting any MDP as a bandit problem with \textit{transportation costs} between arms. In our setting, the transportation cost from a state to another is unknown, initially unbounded and has to be refined over the learning process (the asymptotically optimal cost amounts to the shortest path distance between the two states). We believe that such a setting of unknown and learnable transportation costs is an interesting formalism to study in the bandit setting, as it may then be applied to the MDP extension and allow for smart algorithms that take into account each transportation cost when proposing the arm/state from which a sample is desired (i.e., in part \textcolor{pearThree}{\rectangled{1}} of the algorithmic protocol given at the beginning of App.\,\ref{app_extension_bandit_MDP}). For completeness, it is worth mentioning that some papers study various settings of movement/switching costs between arms (see e.g., \citealp{dekel2014bandits, koren2017bandits}), yet none of these settings can be leveraged for our problem.

\vspace{0.2in}

\section{On Ergodicity} 
\label{app_ergodicity}

In this section we explain why the ergodic setting (Asm.\,\ref{asm_ergodic}) and the more general communicating setting of Asm.\,\ref{asm_communicating} effectively set the boundary on the difficulty of the problem, in the sense that in an ergodic MDP any sampling requirement is eventually fulfilled, whatever the policy executed.

\begin{assumption}[$M$ is ergodic]
    For any stationary policy $\pi$, the corresponding Markov chain $P_{\pi}$ is ergodic, i.e., all states are aperiodic and recurrent.
\label{asm_ergodic}
\end{assumption}

Consider any sampling requirements $b: \mathcal{S} \rightarrow \mathbb{N}$ and fix any stationary policy $\pi$. It induces an ergodic chain $P_{\pi}$ with stationary distribution denoted by $\mu_{\pi} \in \Delta(S)$. Let $\mu_{\pi,\min} := \min_{s \in \mathcal{S}} \mu_{\pi}(s) > 0$. We assume without loss of generality that $P_{\pi}$ is reversible with spectral gap $\gamma_{\pi} > 0$. (Otherwise, in the non-reversible case, the dependency on $\gamma_{\pi}$ in Eq.\,\ref{eq_ergodic} and thus in Eq.\,\ref{eq_n_ergodic} is simply replaced by the pseudo-spectral gap introduced in~\citep{paulin2015concentration}.) It is well-known (see e.g.,\,\citep{hsu2015mixing, paulin2015concentration}) that with probability at least $1-\delta$, for any $s \in \mathcal{S}$ and $t \geq 1$,
\begin{align}
    \left\vert \frac{N_{\pi,t}(s)}{t} - \mu_{\pi}(s)\right\vert \leq \sqrt{ \frac{2 \log\left( \frac{S}{\delta} \sqrt{\frac{2}{\mu_{\pi,\min}}} \right)}{\gamma_{\pi} t} } + \frac{20 \log\left( \frac{S}{\delta} \sqrt{\frac{2}{\mu_{\pi,\min}}}\right)}{\gamma_{\pi} t},
\label{eq_ergodic}
\end{align}
which implies that
\begin{align*}
    N_{\pi,t}(s) \geq t \mu_{\pi,\min} - \sqrt{t} \sqrt{ \frac{2 \log\left( \frac{S}{\delta} \sqrt{\frac{2}{\mu_{\pi,\min}}} \right)}{\gamma_{\pi}} } - \frac{20 \log\left( \frac{S}{\delta} \sqrt{\frac{2}{\mu_{\pi,\min}}}\right)}{\gamma_{\pi}}.
\end{align*}
In particular, we can guarantee that $N_{\pi,t}(s) \geq b(s)$ for any $s \in \mathcal{S}$ whenever
\begin{align}
    t = \Omega\left( \frac{\max_{s \in \mathcal{S}} b(s)}{\gamma_{\pi}\mu_{\pi,\min}} + \frac{1}{\gamma_{\pi}\mu_{\pi,\min}^2} \right).
\label{eq_n_ergodic}
\end{align}
This shows that any policy inevitably meets the sampling requirements in the ergodic setting. Moreover we see that in the case of sampling requirements $b$ that are evened out across the state space, better performance should be achieved by policies with more uniform stationary distributions (i.e., $\mu_{\pi, \min} \gg 0$) and with good mixing properties (i.e., $\gamma_{\pi} \gg 0$).



\section{Experiments}
\label{app_exp}

\begin{figure}[t!]
    \centering
    \begin{subfigure}{0.38\linewidth}
        \centering
        \begin{tikzpicture}
\begin{scope}[local bounding box=scope1,scale=0.4]
\tikzset{VertexStyle/.style = {draw,
		shape          = circle,
		text           = black,
		inner sep      = 2pt,
		outer sep      = 0pt,
		minimum size   = 10 pt}} 
\tikzset{VertexStyle2/.style = {shape          = circle,
		text           = black,
		inner sep      = 2pt,
		outer sep      = 0pt,
		minimum size   = 10 pt}}  
\tikzset{Action/.style = {draw,
		shape          = circle,
		text           = black,
		fill           = black,
		inner sep      = 2pt,
		outer sep      = 0pt}}

\newcommand\xdist{3}
\newcommand\axshift{1.5}
\newcommand\ayshift{1.5}

\node[VertexStyle,gray](s1) at (0,0) {{\tiny $ s_1 $}};
\node[VertexStyle,gray](s2) at (\xdist,0) {{\tiny $ s_2 $}};
\node[gray](s3) at (2*\xdist,0) {$\textbf{....}$};
\node[VertexStyle,gray](s5) at (3*\xdist,0) {{\tiny $ s_{S-1} $}};
\node[VertexStyle,gray](s6) at (4*\xdist,0) {{\tiny $ s_S $}};

\begin{tiny}
\draw[->, >=latex,gray](s1) to [out=70,in=110,looseness=6.5] node[above]{$0.4$} (s1);
\draw[->, >=latex,gray](s1) to [out=45,in=130,looseness=1.2] node[above]{$0.6$} (s2);
\draw[->, >=latex,gray](s2) to [out=45,in=130,looseness=1.2] node[above]{$0.35$} (s3);
\draw[->, >=latex,gray](s2) to [out=70,in=110,looseness=6.5] node[above]{$0.6$} (s2);
\draw[->, >=latex,gray](s2) to [out=170,in=15,looseness=1.2] node[above]{$0.05$} (s1);
\draw[->, >=latex,gray](s3) to [out=45,in=130,looseness=1.2] node[above]{$0.35$} (s5);
\draw[->, >=latex,gray](s3) to [out=170,in=15,looseness=1.2] node[above]{$0.05$} (s2);
\draw[->, >=latex,gray](s5) to [out=170,in=15,looseness=1.2] node[above]{$0.05$} (s3);
\draw[->, >=latex,gray](s5) to [out=45,in=130,looseness=1.2] node[above]{$0.35$} (s6);
\draw[->, >=latex,gray](s6) to [out=70,in=110,looseness=6.5] node[above]{$0.6$} (s6);
\draw[->, >=latex,gray](s6) to [out=170,in=15,looseness=1.2] node[above]{$0.4$} (s5);

\draw[dashed, ->, >=latex,gray](s1) to [out=-70,in=-110,looseness=5.5] node[right]{~~$1$} (s1);
\draw[dashed, ->, >=latex,gray](s2) to [out=225,in=-30,looseness=1.2] node[below]{$1$} (s1);
\draw[dashed, ->, >=latex,gray](s3) to [out=225,in=-30,looseness=1.2] node[below]{$1$} (s2);
\draw[dashed, ->, >=latex,gray](s5) to [out=225,in=-30,looseness=1.2] node[below]{$1$} (s3);
\draw[dashed, ->, >=latex,gray](s6) to [out=225,in=-30,looseness=1.2] node[below]{$1$} (s5);
\end{tiny}


\end{scope}
\end{tikzpicture}
        \caption{Reward-free RiverSwim ($S=6$ states)}
        \label{fig:RS}
        \end{subfigure}
      \begin{subfigure}{0.28\linewidth}
      \centering
        \includegraphics[width=0.66\linewidth]{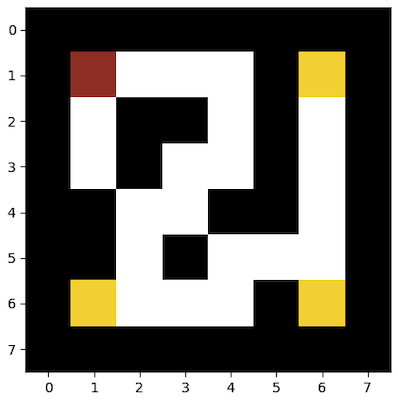}
        \caption{Corridor gridworld ($S=24$ states)}
        \label{fig:GWcorridor}
    \end{subfigure}
    \hspace*{0.2cm}
    \begin{subfigure}{0.3\linewidth}
    \centering
    \includegraphics[width=0.8\linewidth]{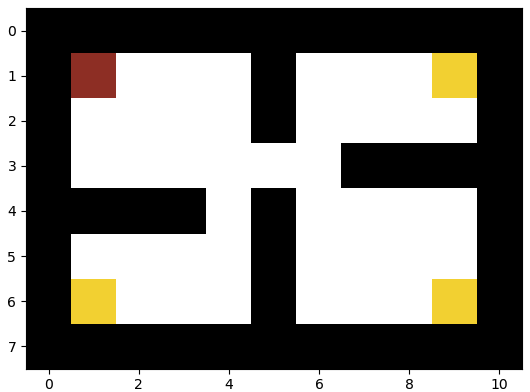}
    \caption{$4$-room gridworld ($S=43$ states)}
    \label{fig:GW4room}
    \end{subfigure}
\caption{The three domains considered in Fig.\,\ref{fig:proportion_treasure}. For the gridworlds (b) and (c), the red tile is the starting state, yellow tiles are terminal states that reset to the starting state, and black tiles are reflecting walls (see §``Details on environments'').}
\label{fig:environments}
\end{figure}
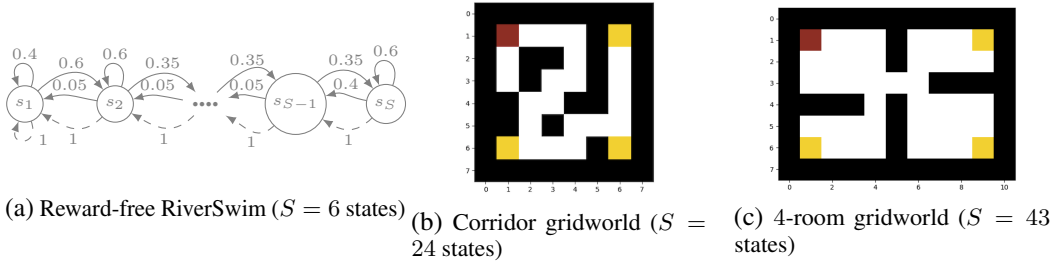


This section complements the experiments reported in Sect.\,\ref{section_experiments}. We provide details about the algorithmic configurations and the environments as well as additional experiments.

\textbf{Details on Fig.\,\ref{fig:proportion_treasure} and Fig.\,\ref{fig:proportion_treasure_more}.} Fig.\,\ref{fig:proportion_treasure} reports, as a function of time $t$, the proportion $\mathcal{P}_t$ of states that at time $t$ satisfy the sampling requirements of the \TREASURETEN problem (i.e., $b(s,a)=10$). Formally,
$\mathcal{P}_t :=  \abs{ \{ s \in \cS : \forall a \in \cA, N_t(s,a) \geq b(s,a) \} } \cdot S^{-1}$. As such, all sampling requirements are met as soon as $\mathcal{P}_t = 1$, meaning that the black line $y=1$ on the y-axis characterizes our objective. Furthermore, we report in Fig.\,\ref{fig:proportion_treasure_more} results on additional domains (see below).

\textbf{Details on environments.} The three domains considered in Fig.\,\ref{fig:proportion_treasure} are given in Fig.\,\ref{fig:environments}. The first one corresponds to a reward-free version of the RiverSwim domain introduced in \citep{strehl2008analysis}, which is a stochastic chain with $6$ states and $2$ actions classically used for testing exploration algorithms. The other two domains are gridworlds. 
In Fig.\,\ref{fig:proportion_treasure_more} we test on a larger RiverSwim domain with $36$ states and three additional gridworlds that are given in Fig.\,\ref{fig:gridworlds_more}. Throughout our experiments, the gridworld domains are defined as follows. The agent can move using the cardinal actions (Right, Down, Left, Up). An action fails with probability $p_{f} = 0.1$, in which case the agent follows (uniformly) one of the other directions. The starting state is shown in red. Yellow tiles are terminal states that, when reached, deterministically reset to the starting state. The black walls act as reflectors, i.e., if the action leads against the wall, the agent stays in the current position with probability $1$. The gridworlds are all reward-free, except the one in Fig.\,\ref{fig:gridworlds_more_A} where the blue tile incurs large negative environmental reward: it is thus a \textit{trap state} which should be avoided as much as possible. Finally, in the experiments with the randomly generated Garnet environments and state-action requirements (Fig.\,\ref{fig:boxplots}), we guarantee the MDPs randomly generated to be communicating by setting $p(s_0|s, a) \geq 0.001$ for every $(s, a)$ and an arbitrary state $s_0$.

\textbf{Algorithmic details.} For all experiments and all considered algorithms, we choose a scaling factor $\alpha_p = 0.1$ of the confidence intervals of the transition probabilities (which enables to speed up the learning, see e.g., \citep{fruit2018efficient}), as well as a confidence level set to $\delta=0.1$. Recall that for \OSP, in the case of state-only requirements, a state $s$ is considered as under-sampled and is thus a goal state if $\sum_{a \in \mathcal{A}} N(s,a) < b(s)$, while in the case of state-action requirements, a state $s$ is considered as under-sampled if $\exists a \in \mathcal{A}, N(s,a) < b(s,a)$. We consider the following initial phase for \OSP (i.e., when all states are under-sampled): we select as goal states those minimizing the \say{remaining budget} $b(s) - N(s)$ for state-only requirements (or $\sum_{a \in \cA} \max\{ b(s,a) - N(s,a), 0 \}$ for state-action requirements), which has the effect of shortening the length of the initial phase. In the case of state-action requirements, once a sought-after goal state $s$ is reached, \OSP selects an under-sampled action $a$ whose gap $b(s,a) - N(s,a)$ is maximized. We note that this design choice can be observed in Fig.\,\ref{fig:proportion_treasure} and \ref{fig:proportion_treasure_more} where \OSP seeks to \say{even out} its sampling strategy, with a steady increase in $(\mathcal{P}_t)$, instead of exhausting the requirements state after state.

\begin{figure*}[t!]
	\centering
	\begin{minipage}{0.31\linewidth}
		\begin{center}
		\includegraphics[width=0.5\linewidth]{PLOTS/legend.pdf}
		\end{center}
		\captionof{figure}{Proportion $\mathcal{P}_t$ of states that satisfy the sampling requirements at time $t$, averaged over 30 runs, on the \TREASURETEN problem with \mbox{$b(s,a)=10$}. \textit{Top left:} RiverSwim(36) with $36$ states (see Fig.\,\ref{fig:RS}), \textit{Top right:} 10-state gridworld with high-cost state, \textit{Bottom left:} 20-state 4-room symmetric gridworld, \textit{Bottom right:} 48-state CliffWalk-type gridworld.} 
	\label{fig:proportion_treasure_more}
	\end{minipage}\hfill
	\begin{minipage}{0.29\linewidth}
		\centering
		\includegraphics[width=0.99\linewidth]{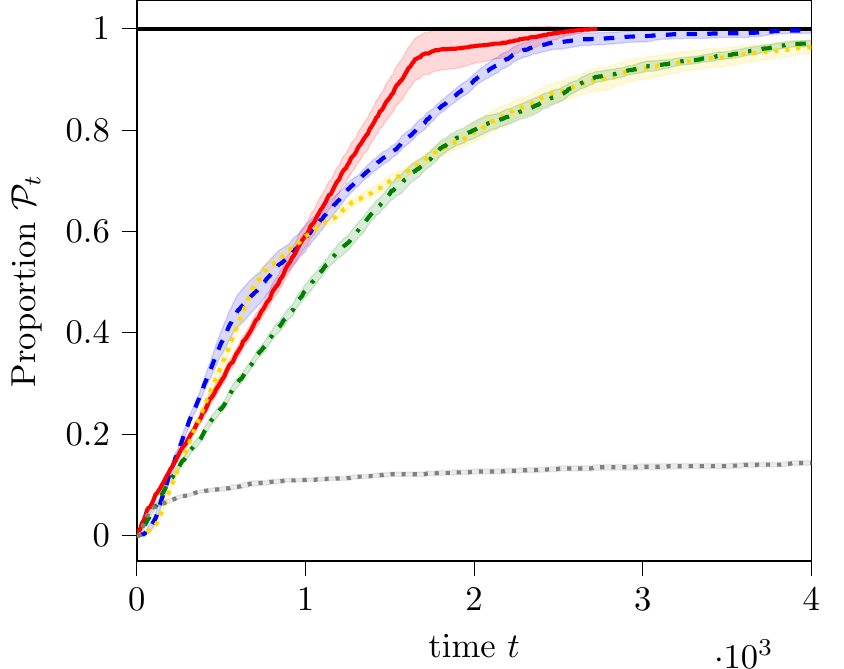} \\
		\includegraphics[width=0.99\linewidth]{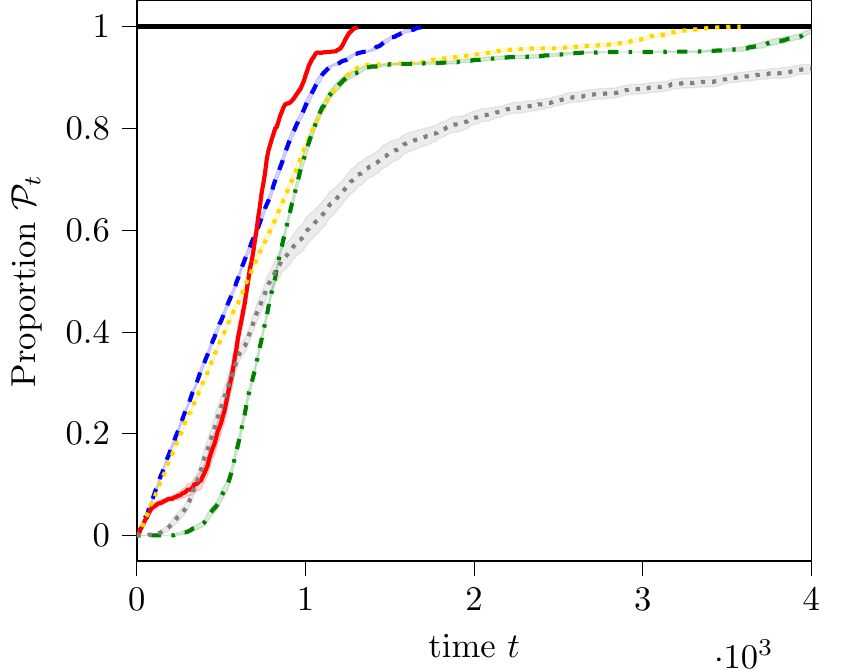}
	\end{minipage}\hfill
	\vspace{0.2in}
	\begin{minipage}{0.29\linewidth}
		\centering
		\includegraphics[width=0.99\linewidth]{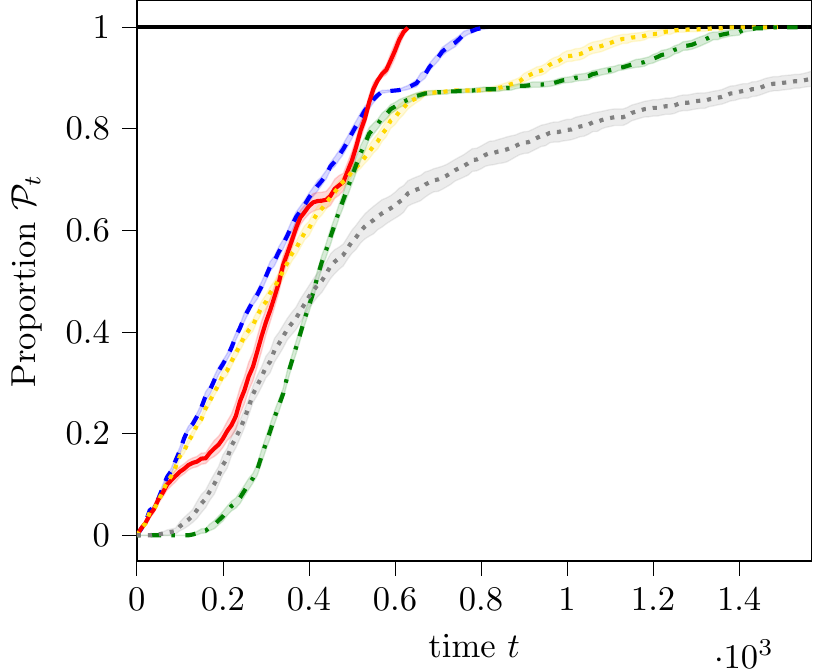}\\
		\includegraphics[width=0.99\linewidth]{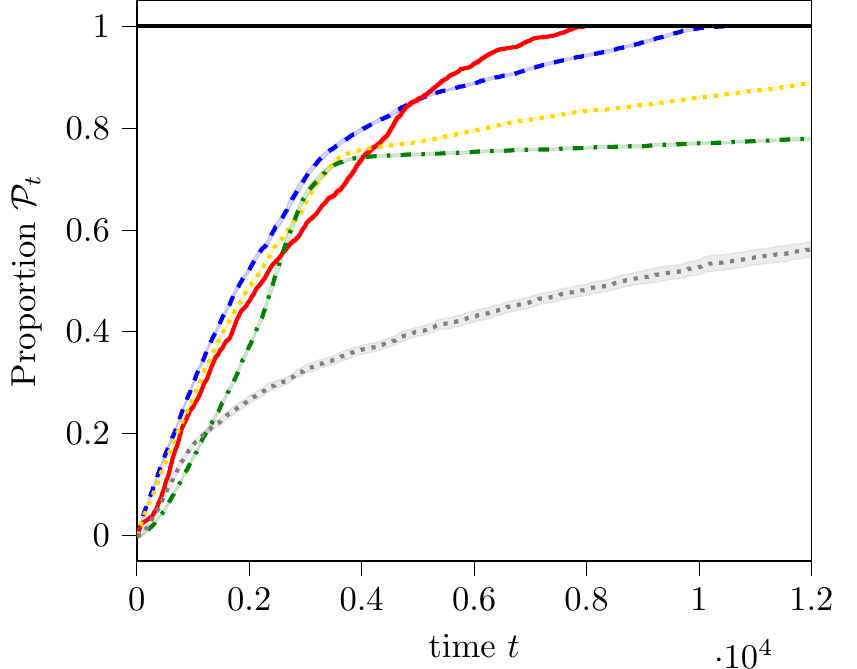}
	\end{minipage}
\end{figure*}

\begin{figure}[t!]
    \centering
    \begin{subfigure}{0.3\linewidth}
        \centering
        \includegraphics[width=0.5\linewidth]{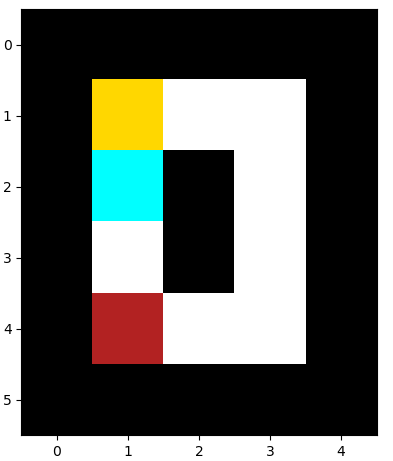}
        \caption{Gridworld with high-cost state \\ \textcolor{white}{.} \hspace{0.07in} ($S=10$ states)}
        \label{fig:gridworlds_more_A}
        \end{subfigure}
      \hspace*{0.2cm}
      \begin{subfigure}{0.3\linewidth}
      \centering
        \includegraphics[width=0.6\linewidth]{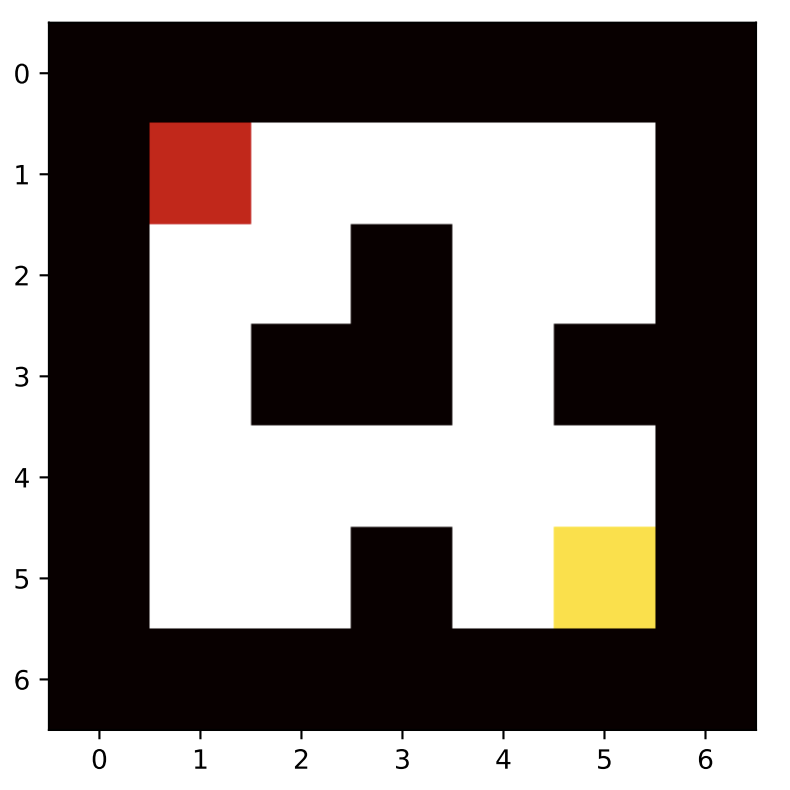}
        \caption{4-room symmetric gridworld \\ \textcolor{white}{.} \hspace{0.07in} ($S=20$ states)}
        \label{fig:gridworlds_more_B}
    \end{subfigure}
    \begin{subfigure}{0.35\linewidth}
    \flushright
    \includegraphics[width=0.84\linewidth]{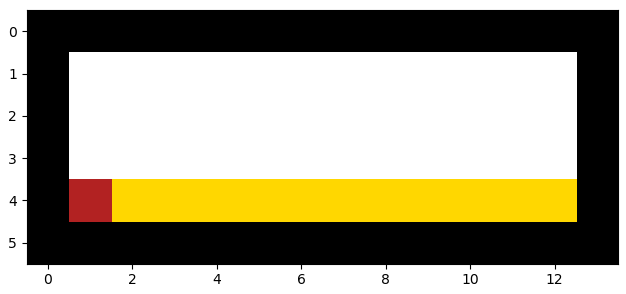}
    \caption{CliffWalk-type gridworld \\ \hspace{-0.4in} ($S=48$ states)}
    \label{fig:gridworlds_more_C}
    \end{subfigure}
\caption{The three gridworlds considered in Fig.\,\ref{fig:proportion_treasure_more}. The blue tile in (a) is a ``trap state'' that incurs large negative environmental reward and should thus be avoided as much as possible.}
\label{fig:gridworlds_more}
\end{figure}

\textbf{\OSP-for-\MODEST algorithm.} Here we detail the \OSP-for-\MODEST algorithm used in the \MODEST experiment of Fig.\,\ref{fig:modest}. The \OSP sampling requirements are computed using a decreasing \MODEST accuracy $\eta$, which enables the algorithm to be accuracy-agnostic like the \WeightedMaxEnt heuristic to which it is compared. \OSP-for-\MODEST starts at an initial accuracy of $\eta \leftarrow1$ and iteratively performs the two following steps until the algorithm ends: \textit{i)} it requires a sampling requirement of $b_t^{\MODEST}(s,a) = \alpha_b \Phi\big(\sum_{s' \in \cS} \sqrt{\widehat{\sigma}^2_{t}(s'\vert s,a)}~, ~ S\big)$, where $\Phi$ is defined after Eq.\,\ref{budget_modest} for accuracy~$\eta$ and where $\alpha_b = 0.01$ is a scaling factor to speed up the learning; and \textit{ii)} when the sampling requirements are fulfilled by \OSP, it sets $\eta \leftarrow \eta / 2$ and goes back to the first step.

\textbf{Dependencies.} For each environment of Fig.\,\ref{fig:environments} on the \TREASURETEN problem (i.e., $b(s,a)=10$, $B=10SA$), we compute in Tab.\,\ref{tab:dependencies} the sample complexity of \OSP run with known dynamics, to put aside the learning component so that its corresponding sample complexity can be bounded exactly by $BD$ or by $\sum_s b(s)D_s$ according to the analysis in Sect.\,\ref{section_osp_algorithm}. Both bounds are reported in Tab.\,\ref{tab:dependencies}: we observe that the second (more state-dependent) quantity is tighter and more preferable than the first. Despite both bounds being loose w.r.t.\,the actual algorithmic performance, they can effectively capture the difficulty of the problem (in a relative sense where the higher the bounds, the higher the sample complexity). We also recall from Sect.\,\ref{section_experiments} that there exist simple worst-case problems (see e.g., Fig.\,\ref{toy_diameter}) where these bounds are tight, i.e., where the sample complexity of \OSP (whether the dynamics are known or not) must directly scale with these diameter quantities. 
Notice that running \OSP with known dynamics corresponds to deploying an optimal \textit{greedy} strategy (i.e., by minimizing each time to reach under-sampled states in a sequential fashion), which is likely not the optimal non-stationary solution (which would involve solving a sort of highly difficult, online travelling salesman problem), see App.\,\ref{remark_greedy} for additional discussion. Finally, we study the sample complexity of \OSP across similar MDPs with increasing number of states to see how that dependence pans out. Fig.\,\ref{fig:dep_states} reports the sample complexity of \OSP in randomly generated Garnet MDPs for increasing values of $S$. We observe that as expected, the sample complexity scales linearly with $S$.

\begin{figure}[t]
\begin{minipage}{0.28\linewidth}
	\captionof{table}{For the \TREASURE-10 problem, we report the quantities $BD$, $\sum_s b(s)D_s$ and the sample complexity of \OSP run with known dynamics (averaged over 30 runs), on the $3$ domains of Fig.\,\ref{fig:environments}.}
    \label{tab:dependencies}
	\end{minipage}\hfill
	\begin{minipage}{0.7\linewidth}
\centering
\renewcommand*{\arraystretch}{1.2}
\begin{small}
\begin{tabular}{|c|c|c|c|}
  \hline
  \textit{Environment} & $BD$ & $\sum_s b(s)D_s$ & \thead{Sample comp.\,of \\ \OSP run with \\ known dynamics $p$}  \\
  \hline
  RiverSwim(6) & $1766.7$ & $958.7$ & $249.9$ \\ 
   \hline
  Corridor gridworld(24) & $24375.6$ &  $13695.2$ & $3156.5$ \\
  \hline
  $4$-room gridworld(43) & $27399.7$ &  $19048.3$ & $3342.5$ \\
  \hline
\end{tabular}
\end{small}
\end{minipage}
\end{figure}

\begin{figure}[t]
\begin{minipage}{0.50\linewidth}
\centering
\includegraphics[width=0.6\linewidth]{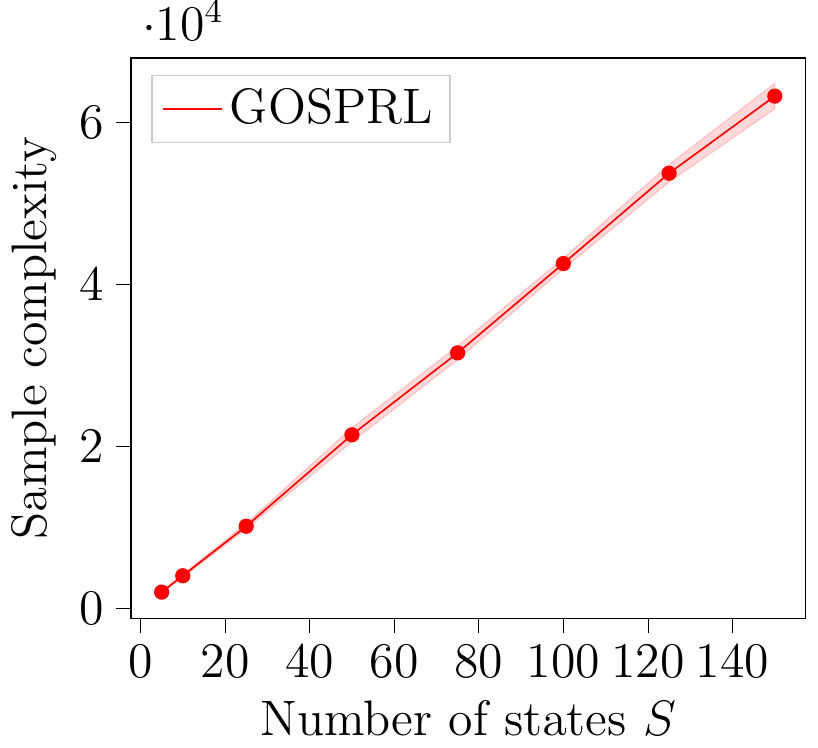}
\captionof{figure}{\small Sample complexity of \OSP in randomly generated Garnet MDPs \textbf{for increasing values of $S$}, with all other parameters fixed ($A$, $\beta$, $\overline{U}$) as in Fig.\,\ref{fig:boxplots}. Results are averaged over $5$ Garnets, each for $12$ runs.}
\label{fig:dep_states}
\end{minipage}%
    \hfill%
    \begin{minipage}{0.44\linewidth}
    \centering
\includegraphics[width=0.8\linewidth]{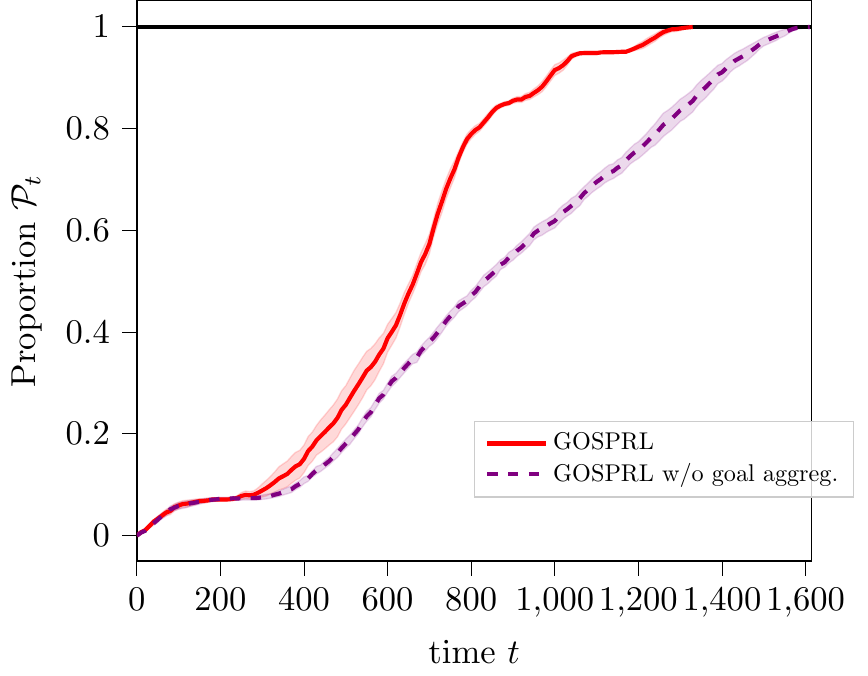}
\captionof{figure}{\small \textbf{Impact of goal aggregation on \OSP.} Proportion $\mathcal{P}_t$ averaged over 30 runs, on the \TREASURETEN problem with $b(s,a)=10$ on the environment of Fig.\,\ref{fig:gridworlds_more_B}.}
\label{fig:aggregation}
\end{minipage}%
\end{figure}

\renewcommand\theadalign{bc}
\renewcommand\theadfont{}
\renewcommand\theadgape{\Gape[1pt]}
\renewcommand\cellgape{\Gape[1pt]}

\vspace{0.2in}

\begin{figure}[t]
\begin{minipage}{0.45\linewidth}
	\captionof{table}{\small \textbf{Impact of cost shaping on \OSP.} Environment of Fig.\,\ref{fig:gridworlds_more_A}. Sampling requirement are concentrated at the yellow terminal state $y \in \cS$, i.e., $b(y,a)=10$ for all $a \in \cA$. Cost-weighted \OSP sets a cost of $10$ (instead of $1$) at the blue trap state during each SSP planning step. Values are averaged over 30 runs.}
    \label{tab:costs}
	\end{minipage}\hfill
	\begin{minipage}{0.5\linewidth}
\centering
\renewcommand*{\arraystretch}{1.2}
\begin{small}
        \begin{tabular}{ | c | c | c |}
      \hline
      \thead{} & \thead{\OSP \\ (Alg.\,\ref{algo:SO})} & \thead{Cost-weighted \\ \OSP} \\
      \hline
      \makecell{Sample \\ complexity} &  $253.1$ & $520.0$  \\
      \hline
      \makecell{Visits to \\ trap state} & $44.6$ & $4.7$ \\
      \hline
    \end{tabular}
\end{small}
\end{minipage}
\end{figure}

\textbf{Impact of goal aggregation on \OSP.} \OSP iteratively aggregates the undersampled states into a \textit{meta-goal} for which it computes an optimistic goal-oriented policy. While it is possible to focus on specific goal states as mentioned in App.\,\ref{sub_sect_subroutines} without affecting the sample complexity guarantee, performing the goal aggregation leads to shorter and more successful sample collection attempts. We observe in Fig.\,\ref{fig:aggregation} that this indeed translates into better empirical performance. Indeed, \OSP collects the prescribed samples faster than a version of \OSP that selects uniformly at random a single goal state among all undersampled states (i.e., that does not perform goal state aggregation).

\textbf{Impact of cost shaping on \OSP.} While \OSP in Alg.\,\ref{algo:SO} considers unit costs for each SSP problem it constructs, any non-unit costs can be designed as long as they are positive and bounded. In particular, detering costs may be assigned to trap states with large negative environmental reward that the agent seeks to avoid. To study this, we consider the gridworld of Fig.\,\ref{fig:gridworlds_more_A} where the blue tile is a trap state that the agent must avoid as much as possible. For ease of exposition we consider here sampling requirements concentrated at the terminal state in yellow denoted by $y \in \cS$, i.e., $b(y,a)=10$ for any $a \in \cA$. We compare \OSP with a cost-weighted \OSP where a cost of~$10$ is set at the blue trap state during each SSP planning step. Tab.\,\ref{tab:costs} shows that while the sample complexity of cost-weighted \OSP is worsened, the number of visits to the undesirable trap state is considerably decreased w.r.t.\,\OSP. This makes sense since the shortest path from the red starting state to the sought-after yellow terminal state goes through the blue trap state, so a trade-off appears between minimizing the sample complexity and visiting undesirable states. This numerical simulation shows that \OSP can naturally adjust this trade-off by cost-weighting the successive SSP problems it tackles.

\end{document}